\documentclass[10pt]{article}
\usepackage{preamble}
\usepackage{todonotes}
\usepackage[numbers]{natbib}
\usepackage{geometry}
\newtheorem{example}{Example}

\renewcommand{\citet}{\cite}
\renewcommand{\citep}{\cite}

\title{Gradient dynamics for low-rank fine-tuning beyond kernels}
\author{
Arif Kerem Dayi\thanks{Email: \href{mailto:keremdayi@college.harvard.edu}{keremdayi@college.harvard.edu}}\\
\small Harvard College \\ 
\and
Sitan Chen\thanks{Email: \href{mailto:sitan@seas.harvard.edu}{sitan@seas.harvard.edu}, this work was supported in part by NSF Award 2331831} \\
\small Harvard SEAS
}
\date{}

\renewcommand{\epsilon}{\varepsilon}
\renewcommand{\top}{\intercal}
\newcommand{\Id}{I}

\begin{document}

\maketitle

\begin{abstract}
    LoRA has emerged as one of the \emph{de facto} methods for fine-tuning foundation models with low computational cost and memory footprint. The idea is to only train a low-rank perturbation to the weights of a pre-trained model, given supervised data for a downstream task. Despite its empirical sucess, from a mathematical perspective it remains poorly understood what learning mechanisms ensure that gradient descent converges to useful low-rank perturbations.

    In this work we study low-rank fine-tuning in a student-teacher setting. We are given the weights of a two-layer base model $f$, as well as i.i.d. samples $(x,f^*(x))$ where $x$ is Gaussian and $f^*$ is the teacher model given by perturbing the weights of $f$ by a rank-1 matrix. This generalizes the setting of generalized linear model (GLM) regression where the weights of $f$ are zero.

    When the rank-1 perturbation is comparable in norm to the weight matrix of $f$, the training dynamics are nonlinear. Nevertheless, in this regime we prove under mild assumptions that a student model which is initialized at the base model and trained with online gradient descent will converge to the teacher in $dk^{O(1)}$ iterations, where $k$ is the number of neurons in $f$. Importantly, unlike in the GLM setting, the complexity does not depend on fine-grained properties of the activation's Hermite expansion. We also prove that in our setting, learning the teacher model ``from scratch'' can require significantly more iterations.
\end{abstract}

\thispagestyle{empty}

\clearpage

\thispagestyle{empty}

\tableofcontents

\thispagestyle{empty}

\clearpage

\setcounter{page}{1}

\section{Introduction}

Modern deep learning at scale involves two phases: pre-training a foundation model with self-supervised learning, and fine-tuning the model towards various downstream tasks. Given the significant computational cost of the former, effective fine-tuning has been essential to the deployment of these models under hardware constraints and the development of powerful open-source models.

In this space, Low-Rank Adaptation (LoRA) has emerged as one of the most successful and widely adopted methods~\citep{hulora}. The idea is to freeze the weights of the pre-trained model and only train \emph{low-rank perturbations} to the weight matrices. Remarkably, this works well even with rank 1 perturbations, reducing the number of trainable parameters by as many as four orders of magnitude.

Despite the surprising effectiveness of LoRA in practice, it is poorly understood from a theoretical perspective why this method works so well. While it is known that for sufficiently deep and wide pre-trained networks, any sufficiently simple target model can be approximated by a low-rank perturbation of the larger model~\citep{zengexpressive}, it is largely unknown what mechanisms ensure that stochastic gradient descent (SGD) converges to these perturbations. Recent works have made initial progress towards understanding this question from the perspective of kernel approximations of neural networks in the lazy training regime~\citep{jang2024lora,malladi2023kernel}. These works consider a setting where the perturbation is small enough relative to the weights of the pre-trained model that the fine-tuned model is well-approximated by its linearization around the pre-trained model.

While the kernel picture provides useful first-order intuition for the dynamics of fine-tuning, it only partially explains its success. For one, the kernel approximation is mainly relevant in the few-shot setting where the network is only fine-tuned on a small number of examples (e.g. a few dozen), but the gap between what is possible with few- vs. many-shot fine-tuning is significant. Even within the few-shot setting, \cite{malladi2023kernel} found that fine-tuning for certain language tasks is not well-explained by kernel behavior, and neither is prompt-based fine-tuning if the prompt is insufficiently aligned with the pre-training task. The gap is even more stark for fine-tuning without prompts. 

In this work we ask:
\begin{center}
    {\em Why does SGD for low-rank fine-tuning converge to a good solution even when the kernel approximation breaks down?}
\end{center}

\noindent To answer this question, we study fine-tuning in a natural student-teacher setting where the training dynamics are inherently non-linear.

\subsection{Problem formulation} Let $\mathcal{F} = \{f_\theta\}_{\theta\in\Theta}$ be some family of neural networks, each parametrized by a set of weights $\theta$. Suppose we are given $\theta_0\in\Theta$, corresponding to a pre-trained \emph{base model}, and then get access to training data $\{(x_i, y_i)\}^N_{i=1}$ for fine-tuning. In this work, we focus on the setting of \emph{realizable Gaussian data} in which the $x_i$'s are i.i.d. Gaussian and there exists a perturbation of the base model, $\theta = \theta_0 + \Delta$ where $\Delta$ is low-rank, for which $f_{\theta}$ perfectly fits the training data. That is, 
\begin{equation}
    x_i \sim\mathcal{N}(0,\Id_n), \ \ \ f_{\theta}(x_i) = y_i
\end{equation} for all $i = 1,\ldots,N$. We call $f_{\theta}$ the \emph{teacher model}.\footnote{In fact our analysis directly extends to the setting where there is unbiased, moment-bounded label noise, but we focus on the noiseless setting as it is slightly cleaner while exhibiting all the relevant phenomena.}
% \kerem{We could say the analysis directly extends to the case when there is unbiased label noise with bounded moments. In fact, i think this might not be that hard to add to the proofs if we have time.}

The goal is to find $\hat{\theta} = \theta_0 + \hat{\Delta}$, where $\hat{\Delta}$ is also low-rank, such that the objective $L(\hat{\theta})$ is small. Here the objective is given by
\begin{equation*}
    L(\hat{\theta}) \triangleq \mathbb{E}_x[\ell(f_{\hat{\theta}}(x), f_{\theta}(x))]\,,
\end{equation*}
where $\ell: \R^2\to\R_{\ge 0}$ is some loss function. In this work we specialize to squared loss.

Algorithms for fine-tuning in practice are based on training the student model, which is initialized to the base model, with gradient descent on $L$. That is, the parameter $\hat{\Delta}$ is repeatedly updated via stochastic gradient descent on the function $\hat{\Delta} \mapsto L(\theta_0 + \hat{\Delta})$. To ensure that $\hat{\Delta}$ is low-rank throughout the course of training, it is typically parametrized by a low-rank factorization, and the matrices in this factorization are the ones with respect to which one performs gradient descent.

Unfortunately, rigorously analyzing the gradient dynamics at this level of generality is well outside the reach of current theory. Instead, in this work we will focus on a specific instantiation of the above setting, namely \emph{two-layer networks} and \emph{rank-1 perturbations}. Despite the apparent simplicity of this setting, the dynamics here already exhibit rich behavior beyond the kernel regime, and as we will see, this model strictly generalizes the problem of \emph{generalized linear model (GLM)} regression,\footnote{This is sometimes referred to as \emph{single-index model} regression. While closely related, the latter technically refers to the setting where the activation $\sigma$ is unknown.} a widely studied toy model in the theoretical foundations of deep learning (see Section~\ref{sec:related}).

Concretely, given $k\in\mathbb{N}$, take $\mathcal{F}$ to be the set of all two-layer networks of width $k$. The base model then takes the form 
\begin{equation}
    f_{\theta_0}(x) \triangleq \lambda^\top\sigma(Wx)\,, \label{eq:twolayer}
\end{equation}
where $\theta_0 = (\lambda, W) \in \R^k \times \R^{k\times d}$ and $\sigma$ is a known scalar activation applied entrywise. 
% For convenience, we will assume that the rows of $W$ have unit norm. {\color{red} do we need to put this in a separate assumption env}

The low-rank perturbation defining the teacher model will be given by $\theta \triangleq (\lambda, W^*)$ where
\begin{equation}
    W^* = W + \Delta \ \ \ \text{for} \ \ \ \Delta = \xi cu^{\top} \label{eq:Wstar}
\end{equation}
for $\xi > 0$ a known \emph{scale} parameter and for unit vectors $c\in\mathbb{S}^{k-1}$, $u\in\mathbb{S}^{d-1}$. Given a target level of error $\epsilon$, our goal is to find unit vectors $\hat{c}, \hat{u}$ for which $L(\hat{\theta}) \le \epsilon$ for $\hat{\theta} \triangleq (\lambda, W + \xi\hat{c}\hat{u}^\top)$ with high probability over the training data $\{(x_i, y_i)\}^N_{i=1}$.

\paragraph{Connection to GLMs, feature learning, and lazy training.}

Note that the special case where the base model is trivial, i.e. when $W = 0_{k\times d}$, recovers the well-studied question of GLM regression.
Indeed, consider the case of $c = (1/\sqrt{k},\ldots,1/\sqrt{k})$, $\lambda = \frac{1}{k}(1,\ldots,1)$, and $\xi = \sqrt{k}$. In this case, if the teacher models' parameters are given by $\theta = (\lambda, W^*)$ where $W^*$ is defined in Eq.~\eqref{eq:Wstar}, then the teacher model is given by $f_{\theta} = \sigma(\langle u, x\rangle)$. Learning a direction $\hat{u}$ for which $\mathbb{E}_x[\ell(\sigma(\langle \hat{u}, x\rangle), \sigma(\langle u, x\rangle))]$ is small, given samples $\{(x_i, \sigma(\langle u, x_i\rangle)\}^N_{i=1}$, is precisely GLM regression. The behavior of gradient descent for this question is by now very well-understood, shedding light on the training dynamics of neural networks in the \emph{feature learning} regime (also called the ``rich'' or ``$\mu$P'' regime) in a stylized but rich model~\cite{bietti2022learning}.

Equivalently, instead of keeping the scale $\xi$ fixed and sending $W$ to zero, we can consider keeping $W$ fixed but nonzero, sending $\xi\to\infty$, and considering $\epsilon$ scaling with $\xi$. This equivalent view is the one we will take in this work as it is more natural for us to regard $W$ as fixed and $\xi$ as a parameter to be varied. 

Under this view, note that at the other extreme where $\xi\to 0$, the teacher model becomes well-approximated by its linearization around the base model, in which case the training dynamics degenerate to the \emph{lazy training} regime (also called the ``NTK regime''). For this reason, the scale parameter $\xi$ gives a natural way to interpolate between feature learning and lazy training dynamics.

\subsection{Our contributions}

\subsubsection{Assumptions} 

Our guarantees will apply to a very wide family of activations $\sigma$ including all standard ones, e.g. ReLU, sigmoid, polynomial, etc. As the conditions on $\sigma$, while quite general, are rather technical to state, we defer them to \Cref{assumption:activation} in the supplement and henceforth refer to such activations as \emph{nice}.

More importantly, we make the following assumptions on the base model and teacher model. Denote the rows of $W$, i.e. the pre-trained features, by $w_1,\ldots,w_k\in\mathbb{R}^d$. Then we have:
% \kerem{Maybe we should combine all the assumptions related to the activation into one with part A,B,C since there are 3 assumptions}
% \begin{assumption}[Niceness of activation]
%     The activation function $\sigma$ has expansion $\sigma(z) = \sum^\infty_{p=0} \mu_p(\sigma) h_p(z)$, where $h_p(\sigma)$ is the normalized probabilist's Hermite polynomial of degree $p$, and there exist constants $C_\sigma, \rho > 0$ such that $|\mu_p(\sigma)| \le C_\sigma p^{-1-\rho}$ for all $p\ge 0$. Additionally, $\sigma$ is differentiable almost everywhere with respect to the standard Gaussian measure, and its derivative has almost linear polynomial growth.
% \end{assumption}

\begin{assumption}[Normalization]\label{assumption:normalize}
    $\|w_i\|_2 = 1$ for all $i = 1,\ldots,k$. 
\end{assumption}

\begin{assumption}[Orthogonality of perturbation]\label{assumption:orthogonal}
    The vector $u$ for the teacher model (see Eq.~\eqref{eq:Wstar}) is orthogonal to the span of $w_1,\ldots,w_k$.
\end{assumption}

\begin{assumption}[Random quantized $c$]\label{assumption:quantization}
    $c$ is sampled uniformly from $\{\pm 1/\sqrt{k}\}^k$.
\end{assumption}

\noindent 
% The first is an extremely mild assumption on the decay of the Hermite coefficients of $\sigma$, which is satisfied by any reasonable activation function, e.g. ReLU, sigmoid, absolute value, polynomial, etc. 
Assumption~\ref{assumption:normalize} is without loss of generality when $\sigma$ is positive homogeneous like in the case of ReLU activation. For general activations, note that one can also handle the case of $\|w_i\|_2 = R$ for all $i$ for arbitrary constant $R > 0$ by redefining $\sigma$. This assumption is not essential to our analysis and we assume the scales of the pre-trained features are the same to keep the analysis transparent. 
% {\color{red} we may want to reword this}

Assumption~\ref{assumption:orthogonal} is crucial to our analysis. To motivate this, in Appendix~\ref{app:global}, we give a simple example where it fails to hold and the low-rank fine-tuning problem ends up having \emph{multiple global optima}, suggesting that the dynamics in the absence of Assumption~\ref{assumption:orthogonal} may be significantly more challenging to characterize. We leave this regime as an interesting area for future study.

Assumption~\ref{assumption:quantization} consists of two parts: 1) the entries of $c$ are constrained to lie within $\{\pm 1/\sqrt{k}\}$, and 2) they are random. The former is for technical reasons. First note that the connection to GLMs still holds under this assumption. Our main reason to make this is that our proof uses Hermite analysis, and while it is in principle possible to handle neurons with different norms, assuming the $c_i$'s are quantized renders our analysis more transparent without sacrificing descriptive power. As our simulations suggest, the phenomena we elucidate persist without this assumption (see Figure~\ref{fig:train-plots}). 

As for the randomness of $c$, while we conjecture that fine-tuning should be tractable even in the worst case over $c$ (see Remark~\ref{remark:momenttensors}) albeit with more complicated dynamics, in this work we only show guarantees that hold with \emph{high probability} over $c$. We primarily use the randomness to ensure that certain quantities that are generically non-vanishing indeed do not vanish.
% , in the spirit of smoothed analysis~\cite{spielman2004smoothed}.
One could equivalently formulate our guarantees as holding under a certain set of deterministic nondegeneracy conditions on the rank-1 perturbation. 

% in example, we need to arrange such that $\langle c_{\lambda}, \hat c_{\lambda}\rangle$ and $\sum_i \lambda_i c_i \sum_j \lambda_j c_j$ have opposite signs. Furthermore, the lambdas have to be imbalanced

\subsubsection{Training algorithm} In this work, we will focus on learning the factor $u$ in the rank-1 perturbation $\Delta = \xi cu^\top$ from Eq.~\eqref{eq:Wstar} using gradient descent. As the weight vectors in the teacher model are given by $w_i + \xi c_i u$, the vector $u$ corresponds to the \emph{direction} in which each of the pre-trained features gets perturbed. Learning this direction turns out to be the most challenging part of fine-tuning: once one has converged to a sufficiently good estimate of $u$, it is straightforward to learn $c$ even using a linear method \--- see Appendix~\ref{app:learnc} for details. As such, in the student model, we will keep $\hat{c}$ frozen at random initialization and only train $\hat{u}$.  Remarkably, as we will see, \emph{the misspecification between $\hat{c}$ and the true $c$ does not significantly affect the learning dynamics}. This robustness to misspecification suggests it may be possible to prove convergence even if $c$ and $u$ were jointly trained, as is done in practice, and we leave this as another important future direction.

We now specify the instantiation of online SGD that we will analyze. Let $f^*$ denote the teacher model and $(u_t)$ the iterates of online SGD with learning rate $\eta > 0$. Let $\hat{c}\in\{\pm 1/\sqrt{k}\}^k$ be sampled uniformly at random at initialization. The algorithm is initialized with 
\begin{equation*}
    u_0 \sim \mathbb{S}_{\Pi^\perp_{\mathrm{span}(W)}}\,,
\end{equation*}
i.e. uniformly over the set of unit vectors which are orthogonal to the span of the pre-trained features $w_1,\ldots,w_k$.
Given training example $(x,f^*(x))$, define the loss attained by $\hat{u}$ on this example by 
% \kerem{We need to specify that we are restricting $\hat u$ to be orthogonal to $w_i$}
\begin{equation*}
    L(\hat{u};x) \triangleq (f^*(x) - \lambda^\top \sigma((W_0 + \xi \hat{c}\hat{u}^\top)x))^2\,.
\end{equation*}
Denote its \emph{spherical gradient} by $\hat{\nabla}L(\hat{u};x) \triangleq (I - \hat{u}\hat{u}^\top)\tilde \nabla L(\hat{u};x)$. Note we are working with the gradients restricted to the subspace of training, i.e. $\tilde \nabla L(\hat u;x) \triangleq \Pi^\perp_{\mathrm{span}(W)} \nabla L(\hat u;x)$ to keep $\hat u$ in this subspace. The update rule is then given by the following: at each step $t$, defining $\mathrm{proj}(v) \triangleq v / \norm{v}$,
\begin{equation}
    u_{t+1} = \mathrm{proj}(u_t - \eta\hat{\nabla} L(u_t;x_t))\,, \qquad x_t \sim \mathcal{N}(0,I)\,.\label{eq:main_update}
\end{equation}
Understanding the gradient dynamics of low-rank fine-tuning in our setting therefore amounts to quantifying the convergence of $u_t$ to the ground truth vector $u$.

\subsubsection{Statement of results}

In this work, we consider two regimes: (1) when $\{w_i\}$ are orthogonal, and (2) when $\{w_i\}$ have very mild angular separation but are otherwise arbitrary.

\paragraph{Orthonormal features.} For this case, we will consider the regime where the scale $\xi$ of the rank-1 perturbation defining the teacher model is large, namely $\xi = \Theta(\sqrt{k})$. Because the norm of the perturbation is comparable to the Frobenius norm of the weight matrix of the base model, the teacher model is not well-approximated by its linearization around the base model.
% In fact, as we show in Lemma~???, it is not well-approximated by any linear function in features which are independent of the projection of the input in the direction of $u^*$. 
This is therefore a minimal, exactly solvable setting for low-rank fine-tuning where kernel approximation fails and the dynamics fall squarely outside of the lazy training regime.

Our first result is to show that online SGD efficiently converges to the correct rank-1 perturbation.

\begin{theorem}[Informal, see \Cref{thm:orth_frob_main}]\label{thm:ortho_informal}
    Let $0 < \epsilon < 1$, and let $\xi\asymp \sqrt{k}$ for sufficiently small absolute constant factor. Suppose the rows of $W$ are orthogonal. Then under Assumptions~\ref{assumption:normalize}-\ref{assumption:quantization} and for any nice activation $\sigma$ (see Assumption~\ref{assumption:activation}), the following holds with high probability over the randomness of $c, \hat{c}$ and the examples encountered over the course of training, and with constant probability over the random initialization $u_0$: online SGD (see Eq.~\eqref{eq:main_update}) run with step size $\eta = \tilde{\Theta}(\epsilon^3 / dk^{7/2})$ and $T = \tilde{\Theta}(dk^4  / \epsilon^4)$ iterations results in $u_T$ for which $\langle u_T, u\rangle^2 \ge 1 - \epsilon$.
\end{theorem}

\noindent Interestingly, the iteration complexity does not depend on fine-grained properties of the activation $\sigma$. In contrast, as we discuss in Section~\ref{sec:overview_reduce}, the iteration complexity of noisy gradient descent for learning GLMs depends heavily on the decomposition of $\sigma$ in the Hermite basis. Given that the GLM setting can be recovered from the fine-tuning setting in the $\xi \to\infty$ limit, \Cref{thm:ortho_informal} implies that the gradient dynamics for fine-tuning exhibit a transition in behavior at some scale $\xi = \Omega(\sqrt{k})$.

\paragraph{Separated features.}

While the orthonormal features setting illustrates an important difference between low-rank fine-tuning and GLM regression, the assumption that the features are orthonormal is constraining. We next turn to a more general setting where we only assume that no two pre-trained features are too correlated. Specifically, we make the following mild assumption:

\begin{assumption}[Angular separation]\label{assume:angle}
    For all $i\neq j$, we have $|\langle w_i, w_j\rangle| \le 1 - \log k/\sqrt{k}$.
\end{assumption}

\begin{theorem}[Informal, see Theorem~\ref{thm:separated_main}]\label{thm:separated_informal}
    Under the same assumptions as Theorem~\ref{thm:ortho_informal}, except with $\xi = 1$ and assuming the rows of $W$ satisfy Assumption~\ref{assume:angle} instead, the following holds with high probability over $c,\hat{c}$ and the examples, and with constant probability over $u_0$: online SGD run with step size $\eta = \tilde{\Theta}(\epsilon^3/ dk^{5/2})$ and $T = \tilde{\Theta}(dk^3/\epsilon^4)$ iterations results in $u_T$ for which $\langle u_T, u\rangle^2 \ge 1 - \epsilon$.
\end{theorem}

\noindent Given the generality of Assumption~\ref{assume:angle}, we are unable to show a guarantee for learning a rank-1 perturbation at the same scale $\xi$ as Theorem~\ref{thm:ortho_informal}. Nevertheless, note that in the regime of $\xi = \Theta(1)$, if one simply considers the linearization of the teacher model around the base model, one is bottlenecked at some fixed level of error. In particular, this means that the kernel approximation to fine-tuning is insufficient to explain why gradient descent converges to the ground truth. One can thus interpret our Theorem~\ref{thm:separated_informal} as shedding light on the later stages of many-shot fine-tuning whereby the result of the linearized dynamics gets refined to arbitrarily high accuracy.

\begin{remark}[Other algorithms for fine-tuning]
Furthermore, the focus of our paper is to understand why \emph{gradient-based} fine-tuning works, motivated by the practical success of LoRA, we would like to note that there are potentially algorithms beyond gradient descent that can solve our proposed learning problem. For example, suppose $\sigma = \mathrm{ReLU}$ and let the base model be given by $f_{\theta_0}(x) = \sum_i \lambda_i \sigma(\langle w_i, x\rangle)$ and the teacher model be given by $f_\theta(x) =\sum_i \lambda_i \sigma(\langle w_i + c_i u, x\rangle)$. Then, note $\E_x[x(f_{\theta_0}(x) - f_\theta(x))]= \mu_1(\sigma) \left(\sum_{i=1}^k\lambda_i c_i\right) u$, where $\mu_1(\sigma)$ denotes the first normalized Hermite coefficient of the ReLU activation $\sigma$. This means the ground truth vector $u$ could be recovered when $\sum_i \lambda_i c_i$ and $\mu_1(\sigma)$ do not vanish. One can then recover $c$ using our algorithm in Appendix~\ref{app:learnc}.

On the other hand, if $\sum_i \lambda_i c_i$ and $\mu_1(\sigma)$ vanish, it is unclear how to proceed. It might be possible to use higher order moment tensors to solve the fine tuning problem, but it is open to understand at what level of generality this would work and whether such an algorithm can achieve sample complexity scaling linearly in $d$ as we show online SGD does. We leave finding general algorithms for the low-rank fine-tuning problem and worst-case lower bounds as an interesting orthogonal direction to be explored.
\end{remark}

% maybe move to a separate paragraph heading
\paragraph{Separating fine-tuning and learning from scratch.} Finally, we show a rigorous lower bound suggesting that fine-tuning is strictly more tractable than learning from scratch in our two-layers setting (see Section~\ref{sec:lowerbound} for details):
% \kerem{I am not sure what we are going to say about GLMs?}
\begin{theorem}[Informal, see \Cref{thm:hard-from-scratch}] 
    For any $p > 2$, there exists a base network and a perturbation for which learning the teacher model from scratch using any correlational statistical query algorithm requires either $n=d^{p/2}$ queries or $\tau= d^{-p/4}$ tolerance. However, fine-tuning the base network using Gaussian examples labeled by the teacher only requires $\tilde O(d)$ online SGD iterations.
\end{theorem}

\noindent The proof involves a base model with Hermite activation of degree $p$ whose perturbation has orthonormal weight vectors (see \Cref{claim:hard-from-scratch}) with a carefully chosen $c, u$. Even though $c$ is not random (i.e. Assumption~\ref{assumption:quantization} does not apply to the lower bound instance), we nevertheless prove online SGD still converges to the ground truth perturbation in $\tilde O(d)$ iterations.

% remark that learning is trivial in our setting when first Hermite coefficient is nonzero (e.g. in relu setting), but outside of this case it's not clear how to use a moment-based estimator. our analysis isn't so sensitive to these kinds of things

\subsection{Related work}
\label{sec:related}

\paragraph{Parameter-efficient fine-tuning.} 

Following the popularization of LoRA~\citep{hulora}, there have been a large number of proposed refinements thereof~\citep{fu2023effectiveness,dettmers2024qlora,lialin2023relora}; a thorough review of the empirical literature is beyond the scope of this work.

Within the mathematical literature on fine-tuning, the works directly related to ours are the aforementioned results of \citet{malladi2023kernel, jang2024lora}.  \citet{malladi2023kernel} primarily presented empirical evidence of kernel behavior for prompt-based fine-tuning methods, including LoRA, in the few-shot regime. Their main theoretical result regarding LoRA roughly states that if standard (full-rank) fine-tuning exhibits kernel behavior, then low-rank fine-tuning exhibits kernel behavior, provided the rank of the perturbation is at least $\Omega(1/\epsilon^2)$. \citet{jang2024lora} build upon this as follows. In the kernel regime where the student model is well-approximated by its linearization around the base model throughout training, they consider the resulting linearized empirical loss for an arbitrary dataset. This is still non-convex if one tries jointly training the factors of the low-rank perturbation, but they nevertheless show that this loss has a rank-$O(\sqrt{N})$ global minimizer, where $N$ is the number of training examples. They then show that all local minimizers of this loss are global minimizers, using tools from prior work on low-rank matrix factorization.

These works are incomparable to ours in several regards. Firstly, they operate in the few-shot regime so that the number of training examples $N$ is relatively small, and the perturbation is small enough that one can work with a linear approximation. In contrast, we consider ``full'' low-rank fine-tuning, for which $N$ must scale at least with the ambient dimension, and we are trying to learn much larger perturbations; as we show in Figure~\ref{fig:linearized-network}, this puts us well outside the regime where the kernel approximation does well. In addition, the aforementioned works do not handle the regime where the rank is extremely small, even though LoRA still works quite well in this case. That said, there is no free lunch: our work derives insights in the challenging rank-one, non-linear setting at the cost of working with a specific set of assumptions on the data-generating process.

% On the mathematical front, rigorous results on fine-tuning in the kernel regime require the perturbations to have rank scaling polynomially in either the number of training examples~\citep{jang2024lora} or the inverse approximation error~\citep{malladi2023kernel} and say nothing when the rank is very small, for instance $1$.

\paragraph{GLMs and single/multi-index model regression.}

Generalized linear models have received significant attention in learning theory as a stylized model for feature learning, see~\citet{dudeja2018learning} for an overview of older works on this. Most relevant to our work is \citet{arous2021online} which studied the gradient dynamics of learning GLMs models $\sigma(\langle w,\cdot\rangle)$ over Gaussian examples with online SGD. Their main finding was that online SGD achieves high correlation with the ground truth direction in $\tilde{\Theta}(d^{1\vee l^*-1})$ iterations/samples, where $l^*$ is the \emph{information exponent}, defined to be the lowest degree at which $\sigma$ has a nonzero Hermite coefficient. We draw upon tools from~\citet{arous2021online} to analyze online SGD in our setting, one important distinction being that the population gradient dynamics in our setting are very different and furthermore our finite-sample analysis makes quantitative various bounds that were only proved asymptotically in~\citet{arous2021online}.

By a result of~\citet{szorenyi2009characterizing}, the information exponent also dictates the worst-case complexity of learning generalized linear models: for noisy gradient descent (and more generally, correlational statistical query algorithms), $d^{1\vee l^*/2}$ samples are necessary. Various works have focused on deriving algorithms that saturate this lower bound and related lower bounds for learning \emph{multi-index models}, i.e. functions that depend on a \emph{bounded-dimension} projection of the input, over Gaussian examples~\citep{bietti2022learning,damian2022neural,damian2024smoothing,abbe2023sgd}. A key finding of our work is that quantities like information exponent do not dictate the complexity of fine-tuning.

Finally, we note that several works have explored whether gradient descent for feature learning can circumvent the information exponent barrier, e.g. if one performs multiple passes over the data~\cite{lee2024neural,arnaboldi2024repetita,dandi2024benefits} or via label transformations~\cite{chen2020learning,chen2022fpt,damian2024computational}. For instance, the recent work~\cite{damian2024computational} identified an alternative property of the activation called the \emph{generative exponent}, which can be smaller than the information exponent, that dictates the (non-correlational) statistical query complexity of learning GLMs. We emphasize however that the generative exponent can still be quite large in general, yet even for such activations, the runtime bounds we prove in the fine-tuning setting still apply and achieve \emph{linear} dependence in $d$.

\paragraph{PAC learning neural networks.}

Within the theoretical computer science literature on learning neural networks, there has been numerous works giving algorithms, many of them based on spectral or tensor methods, for learning two-layer networks from scratch over Gaussian examples. The literature is vast, and we refer to~\citet{chen2024faster,chen2023learning} for an overview.

On the hardness side, \citet{diakonikolas2020algorithms} (see also~\citet{goel2020superpolynomial}) proved that for correlational statistical query algorithms, the computational cost of learning such networks from scratch in the worst case must scale with $d^{\Omega(k)}$, 
% \kerem{Is this just $d^{\Omega(k)}$ (what do we mean by exponential?)} % thanks for the catch!
which~\citet{diakonikolas2024efficiently} recently showed is tight for this class of algorithms. 
% Notably, one can modify the family of hard instances in their lower bound to satisfy our Assumption~\ref{assume:angle}, at the cost of a slightly weaker lower bound of $d^{k^{\Theta(1)}}$. Our upper bound shows that almost all rank-1 perturbations of such networks are not hard to learn in the fine-tuning setting.
Additionally, central to these lower bounds for learning two-layer networks is the existence of networks $\sum_i \lambda_i \sigma(\langle w_i, x\rangle)$ for which the tensor $\sum_i \lambda_i w_i^{\otimes s}$ vanishes for all small $s$. As we discuss at the end of Appendix~\ref{app:gradient_derivation}, even if the base model or teacher model satisfies this in the setting that we consider, it does not appear to pose a barrier for low-rank fine-tuning in the same way that it does for learning from scratch.

\subsection{Technical preliminaries}

\paragraph{Notation.}
Let $\S^{d-1}=\{v\in \R^d: \norm{v}=1\}$. For $w\in \R^d$, let $w^{\otimes s}$ denote the order-$s$ tensor power of $w$, and for two tensors $T_1, T_2$ we use $\langle T_1, T_2\rangle$ to denote their element-wise dot product and $\norm{T_1}_F\triangleq \sqrt{\langle T_1, T_1\rangle}$ for the corresponding Frobenius norm. Note the identity $\sum_{i,j=1}^k \lambda_i\mu_j \langle w_i, v_j\rangle^{s}=\langle\sum_{i=1}^k \lambda_i w_i^{\otimes s}, \sum_{i=1}^k \mu_i v_i^{\otimes s}\rangle$ which arises in our analysis as the interactions between different neurons in the population loss.

\paragraph{Bounds.} Our results hold uniformly over the choice of $w_i, u, \lambda$ under their constraints. We make dependencies on $\lambda_{\min}\triangleq \min_i |\lambda_i|$ and $\lambda_{\max}\triangleq\max_i |\lambda_i|$ explicit, but in our $O(\cdot)$ notation, we ignore constants that only depend on the activation $\sigma$. We write $\tilde O(\cdot)$ to omit logarithmic factors.

\paragraph{Hermite analysis.} We will use Hermite analysis to analytically evaluate expectations of products of functions under the Gaussian measure. We let $h_p$ denote the $p$-th normalized probabilist's Hermite polynomial, and $\mu_p(\sigma)$ the $p$-th Hermite coefficient of $\sigma$. In particular, Hermite coefficients form an orthonormal basis for functions that are square integrable w.r.t the Gaussian measure. That is, functions $\sigma$ for which $\norm{\sigma}_2^2 \triangleq \E_{g\sim \mathcal{N}(0,1)}[\sigma(g)^2] < \infty$ and we denote $\sigma \in L_2(\mathcal{N}(0,1))$. These functions admit a Hermite expansion $\sigma(a)=\sum_{p=0}^\infty \mu_p(\sigma) h_p(a)$, and for two functions $f,g \in L_2(\mathcal{N}(0,1))$, we have $\langle f,g\rangle \triangleq \E_{a\sim \mathcal{N}(0,1)}[f(a)g(a)]= \sum_{p}\mu_p(f) \mu_p(g)$. Furthermore, for $u, v\in\S^{d-1}$, Hermite polynomials satisfy 
\begin{equation*}
    \E_{x\sim \mathcal{N}(0, I_d)} [h_p(\langle u,x\rangle)h_q(\langle v,x\rangle)]= \ind{p=q}\langle u,v\rangle^p\,.
\end{equation*}

\section{Main results}

Here we provide a formal statement of our results.

\subsection{Assumptions on the activation function}
\label{sec:activation}

For completeness, we first state the assumptions we make on the activation function $\sigma$. These are exceedingly mild and hold for many standard classes of activations including Lipschitz activations (e.g. ReLU, absolute value, sigmoid, tanh) and polynomial activations. As such, this section can be safely skipped upon a first reading. These conditions are only relevant when we need to bound the variance of gradients in Appendix~\ref{app:variances}.

\begin{assumption}[Activation function]
The activation $\sigma$ satisfies all of the following:
\begin{enumerate}[leftmargin=*]
    \item $\sigma$ is almost surely differentiable (with respect to the standard Gaussian measure), with derivative $\sigma'$ having at most polynomial growth: There exists some $b,c, q>0$ such that $|\sigma'(a)| \leq b + c |a|^{q}$ for all $a \in \R$.
    \item The Hermite coefficients of $\sigma$ have faster than linear decay: There exists $C_{\sigma}, \rho > 0$ such that $|\mu_p(\sigma)|\leq C_{\sigma} p^{-1-\rho}$. 
    \item $\sigma$ satisfies the following moment condition: For $g_1, g_2 \sim \mN(0,1)$ Gaussians (potentially correlated), for some $C_{p, \sigma} > 0$ that only depends the activation and $p$, we have
    \begin{equation*}
    \left(\E|\sigma(g_1) - \sigma(g_2)|^p\right)^{1/p}\leq C_{p, \sigma}  \left(\E|g_1 - g_2|^{2p}\right)^{1/(2p)}
\end{equation*}
\end{enumerate}    
\label{assumption:activation}
\end{assumption}
\begin{remark}
These conditions are satisfied for any reasonable activation used in practice. For the last condition in \Cref{assumption:activation}, note that it holds for any Lipschitz activation (e.g. ReLU, absolute value, sigmoid). Furthermore it is satisfied for any polynomial activation (e.g. finite Hermite expansion). To see why, for a degree $s$ polynomial $p(x) = \sum_{n=0}^s a_n x^n$, note that
    \begin{align*}
        \Bigl|\sum_{n=1}^s a_n g_1^n - \sum_{n=1}^s a_n g_2^n\Bigr|\leq s \max\{|g_1|^{s-1}, |g_1|^{s-2}|g_2|, \dots, |g_2|^{s-1}\} \Bigl(\sum_{n=1}^s |a_n|\Bigr)|g_1 - g_2| 
    \end{align*}
    Then, applying Cauchy-Schwarz, we have
    \begin{align*}
        \sqrt[p]{\E|p(g_1) - p(g_2)|^p} \leq s \Bigl(\sum_{n=1}^s |a_n|\Bigr)\left(\E\max \{|g_1|^{s-1}, \dots, |g_2|^{s-1}\}^{2p}\right)^{1/(2p)} \left(\E |g_1- g_2|^{2p}\right)^{1/(2p)}
    \end{align*}
    Using the fact that the first expectation can be bounded by a constant that only depends on $s$ concludes the result.
\end{remark}

\subsection{Results for fine-tuning with online SGD in different regimes}

We consider two different settings: when the weights of the base model are orthogonal and when they have only mild angular separation (Assumption~\ref{assume:angle}). For the former (resp. latter), we prove convergence guarantees when the perturbation has norm as large as the Frobenius (resp. spectral) norm of the base weight matrix.

\subsubsection{Orthogonal weights}
In this section, we assume $\langle w_i, w_j\rangle = 0$ for all $i \neq j$. Then, we can show convergence guarantees for fine-tuning whenever $\xi = O(\sqrt{k})$. In particular, we state the results for $\xi=1$ and $\xi = \overline{\xi} \sqrt{k}$ in this section.
% Then, notice that $h$ reduces in form to the following:
% \begin{equation*}
%     h(m) = 2 k \left(\sum_i \lambda_i c_i \right)\left(\sum_i \lambda_i \hat c_i\right) \sum_{l\text{ even }}^{\infty} a_l m^l + \sum_{l \text{ odd }} \hat a_l m^l + k\left(\sum_i \lambda_i^2 c_i \hat c_i \right) \sum_{l \text{ even }}  b_l m^{l}
% \end{equation*}

% where the $a_l, \hat a_l, b_l$ are all positive coefficients. Then, we are interested in the magnitudes of the random quantities in the above sum to characterize the behavior of $h$. We do this in the next appendix. Essentially, if the first hermite coefficient is non-zero, the term $(\sum_i \lambda_i c_i) (\sum_i \lambda_i \hat c_i)$ governs the lower bound for $h$. In the other case, we show the term $\sum_i \lambda_i^2 c_i \hat c_i$ governs the lower bound.

\begin{theorem}[Orthogonal weights, $\xi=1$] 
Let \Cref{assumption:orthogonal} hold, and $0 < \epsilon< 1$ be given and $\gamma > 0$. For a sufficiently small $C_\delta = \Theta(1)$, set
\begin{equation*}
    \delta = \frac{C_\delta \gamma \lambda_{\min}^2\epsilon^3}{(\log \lambda_{\rm max}^4 d k^2)^2} \cdot \left.\begin{cases}
        1& \mu_1(\sigma) \neq 0\\
        \frac{1}{\sqrt{k}} &\text{ o.w. }
    \end{cases}\right\} \qquad \text{ and } \qquad \alpha = \frac{\log ( \lambda_{\rm max}^4 dk^2)}{\lambda_{\min}^2\gamma\epsilon\delta} \cdot \left.\begin{cases}
        1 & \mu_1(\sigma) \neq 0\\
        \sqrt{k} & \text{ o.w. }
    \end{cases}\right\}
\end{equation*}
Then, with probability $1- o(\lambda_{\max}^2/\lambda_{\min}^2)- O(\gamma^{1/2})$ over the randomness of $c,\hat c$ and for intialization satisfying $\langle u_0, u\rangle \mathrm{sign}(h(0))\geq \beta/\sqrt{d}$, online SGD run with step size $\eta$ and time $T$ for
\begin{equation*}
\eta = \frac{\delta}{\lambda_{\max}^4 dk^2} \qquad \text{ and } \qquad T = \lceil \alpha \lambda_{\max}^4 dk^2\rceil = 
    \widetilde{O}\Bigl(\frac{\lambda_{\max}^4}{\lambda_{\min}^4 \gamma^2}\cdot \frac{dk^3}{\epsilon^4}\Bigr) \cdot \begin{cases}
        1& \mu_1(\sigma) \neq 0\\
        k&\text{ o.w. }
    \end{cases}
\end{equation*}
satisfies $\langle u_T, u\rangle^2 \geq 1-\epsilon$ with high probability over the randomness of the data. 
\label{thm:orth-spectral-main}
\end{theorem}

% \noindent In particular, when $\mu_1(\sigma) \neq 0$, we need $T = O(\frac{\lambda_{\max}^4}{\lambda_{\min}^4 \gamma^2}\cdot \frac{dk^3}{\epsilon^4})$ iterations. Similarly, when $\mu_1(\sigma) = 0$, we need $T = O(\frac{\lambda_{\max}^4}{\lambda_{\min}^4 \gamma^2}\cdot \frac{dk^3}{\epsilon^4})$ iterations. \sitan{shouldn't the second sentence have an extra factor of $k$, and for both bounds (as well as the one below), it should be $\tilde{O}()$ instead of $O()$}

\begin{theorem}[Orthogonal weights, $\xi=\overline{\xi} \sqrt{k}$]\label{thm:orth_frob_main}
Let \Cref{assumption:orthogonal} hold, $0 < \epsilon< 1$ and $\xi = \overline{\xi} \sqrt{k}$ for some small $\overline{\xi}>0$. For a sufficiently small $C_\delta = \Theta(1)$, set
\begin{equation*}
    \delta = \min\left\{\frac{C_\delta \overline{\xi}^2 \sqrt{k} \gamma \lambda_{\min}^2\epsilon^3}{(\log \lambda_{\rm max}^4 d k^2)^2} \cdot \left.\begin{cases}
        \sqrt{k} & \mu_1(\sigma) \neq 0\\
        1 & \text{ o.w. }
    \end{cases}\right\}, 1\right\} \qquad \text{ and } \qquad \alpha = \frac{\log (\lambda_{\rm max}^4 dk^2)}{\overline{\xi}^2\lambda_{\min}^2\sqrt{k}\gamma\epsilon\delta} \cdot \left.\begin{cases}
        1/\sqrt{k} & \mu_1(\sigma) \neq 0\\
        1 &\text{ o.w. }
    \end{cases}\right\}
\end{equation*}
Then, with probability $1-o(\lambda_{\max}^2/\lambda_{\min}^2)-\exp(-2/(e\overline{\xi})) - O(\gamma^{1/2})$ randomness of $c, \hat c$, for initializations satisfying $\langle u_0, u\rangle\cdot \mathrm{sign}(h(0)) \geq \beta/\sqrt{d}$, online SGD run with step size $\eta$ and time $T$ for
\begin{equation*}
    \eta = \frac{\delta}{\overline{\xi}^2 \lambda_{\max}^4 dk^4} \qquad \text{ and }\qquad T=\lceil \alpha \lambda_{\max}^4\overline{\xi}^2 dk^4 \rceil = \tilde O\Bigl(\frac{\lambda_{\max}^4}{\lambda_{\min}^4 \gamma^2 \overline{\xi}^2} \cdot \frac{dk^4}{\epsilon^4}\Bigr) \cdot  \begin{cases}
        \frac{1}{k} & \mu_1(\sigma) \neq 0\\
        1 & \mu_1(\sigma) = 0
    \end{cases}
\end{equation*}
satisfies $\langle u_T, u\rangle^2\geq 1- \epsilon$ with high probability over the randomness of the data.
% \begin{enumerate}
%     \item For activations with $\mu_1(\sigma)\neq 0$, for a sufficiently small $C_\delta=\Theta(1)$, let $\delta = \min\left\{\frac{C_\delta \overline{\xi}^2 k \gamma \lambda_{\min}^2\epsilon^3}{(\log \lambda_{\rm max}^4 d k^2)^2}, 1\right\}$. Furthermore, let $\alpha = \frac{\log (\lambda_{\rm max}^4 dk^2)}{\overline{\xi}^2\lambda_{\min}^2k\gamma\epsilon\delta}$. 
% \item For activations with $\mu_1(\sigma)= 0$, for a sufficiently small $C_\delta=\Theta(1)$, let $\delta = \min\left\{\frac{C_\delta \overline{\xi}^2 \gamma \lambda_{\min}^2\epsilon^3\sqrt{k}}{(\log \lambda_{\rm max}^4 d k^2)^2}, 1\right\}$. Furthermore, let $\alpha = \frac{\log (\lambda_{\rm max}^4 dk^2)}{\overline{\xi}^2\lambda_{\min}^2\gamma\epsilon\delta\sqrt{k}}$. Then, with probability $1-o\left(\frac{\lambda_{\max}^2}{\lambda_{\min}^2}\right)-\exp\{-\frac{2}{e \overline{\xi}}\} - O(\gamma^{1/2})$ randomness of $c, \hat c$, for initializations satisfying $\langle u_0, u\rangle\cdot \mathrm{sign}(h(0)) \geq \frac{\beta}{\sqrt{d}}$, online SGD run with step size $\eta = \frac{\delta}{\lambda_{\max}^4\overline{\xi}^2 dk^4}$ and time $T=\lceil \alpha \lambda_{\max}^4\overline{\xi}^2 dk^4 \rceil$ satisfies $\langle u_T, u\rangle^2\geq 1- \epsilon$ with high probability over the randomness of the data..
% \end{enumerate}
\end{theorem}

% Note that we consider two kinds of randomness in our probabilistic bounds. There is the randomness due to the $c, \hat c$, and also due to the randomness of the training trajectory due to the data.
Note that in the theorems above, we make the very mild assumption that the initialization satisfies $m_0 \geq (\beta/\sqrt{d})\cdot \mathrm{sign}(h(0))$. Note that the magnitude condition $|m_0| \geq \beta/\sqrt{d}$ is standard in this kind of analysis and will be satisfied with probability $1-O(\beta)$ since random unit vectors in $d$ dimensions have correlation of order $1/\sqrt{d}$. Hence, we think of $\beta$ as a small constant, whose value only effects the initialization condition probability and some other tail probabilities. As long as $\beta = \Theta(1)$ (no dimension dependence), we can increase the probability of the initialization condition to be arbitrarily close to $1$, and all of our results hold with high probability as $d \to \infty$. For the sign condition $\mathrm{sign}(m_0) = \mathrm{sign}(h(0))$, this holds with probability $1/2$ and, if it doesn't, we can simply re-run the algorithm initialized at $-u_0$. 

% \sitan{We should quantify this precisely in the theorem statement, or else there might be confusion between two possible interpretations: 1) it works as long as $\beta$ is at least some universal constant, 2) it works for any $\beta$ but the bounds in the theorem will incur some $O_\beta(1)$ overhead} 
% Empirically, fine-tuning is not sensitive to this
% empirically, the results are not sensitive: However, handling both sign initializations requires knowing more about the structure of $h(m)$ and we defer it to future work. However, note that the sign condition holds with probability $1/2$, and if not, flipping the sign of $u_0$ will ensure that the sign condition holds.

\subsubsection{Angularly separated weights}

In this section, we do not assume the pre-trained weights are orthogonal and instead only assume they are not too correlated (Assumption~\ref{assume:angle}). Then, we have the following result for $\xi = 1$.

\begin{theorem}[Angularly separated weights, $\xi=1$] \label{thm:separated_main}
Let~\Cref{assume:angle} hold, and $0 < \epsilon< 1$. For a sufficiently small $C_\delta=\Theta(1)$, let 
\begin{equation*}
    \delta = \frac{C_\delta \gamma \lambda_{\min}^2\epsilon^3}{(\log \lambda_{\rm max}^4 d k^2)^2\sqrt{k}} \qquad \text{ and }\qquad \alpha = \frac{\log ( \lambda_{\rm max}^4 dk^2)\sqrt{k}}{\lambda_{\min}^2 \gamma\epsilon\delta}
\end{equation*}
Then, with probability $1-o(1) - O(\gamma^{1/2})$ randomness of $c, \hat c$, for initializations satisfying $\langle u_0, u\rangle\cdot \mathrm{sign}(h(0)) \geq \frac{\beta}{\sqrt{d}}$, online SGD run with step size $\eta$ for time $T$ where
\begin{equation*}
    \eta = \frac{\lambda_{\max}^4\delta}{dk^2} \qquad \text{ and }\qquad T=\lceil \alpha \lambda_{\max}^4dk^2 \rceil = \tilde O\Bigl(\frac{\lambda_{\max}^4}{\lambda_{\min}^4\gamma^2}\cdot \frac{dk^3}{\epsilon^4}\Bigr)
\end{equation*}
 satisfies $\langle u_T, u\rangle^2\geq 1- \epsilon$ with high probability over the randomness of the data.
\end{theorem}
\section{Technical overview} \label{sec:technical-overview}

In this section we give an overview of the ingredients that go into our proofs. While our results are about online SGD, the main novelty of our analysis \--- and the bulk of the technical work in this paper \--- lies in establishing estimates for the gradients of the \emph{population loss}. In Section~\ref{sec:overview_reduce}, we explain why such estimates are essential to establishing our main results, and in Sections~\ref{sec:overview_orth} and~\ref{sec:overview_sep} we describe the key steps that go into showing these bounds. In Section~\ref{sec:overview_finite}, we describe how to put everything together to give a finite-sample analysis that proves our main results; this part largely draws on existing techniques introduced by~\cite{arous2021online}. Finally, in Section~\ref{sec:lowerbound} we give a self-contained proof of Theorem~\ref{thm:hard-from-scratch} separating fine-tuning and learning from scratch.

% Our proof technique relies on decomposing the stochastic gradients into a population gradient term and a noise term, which forms a martingale. Hence, there are two main components to our proof. First, we establish lower bounds on the magnitude of the population gradient in the ground truth direction $u$, and afterwards we follow a similar approach to \cite{arous2021online} to show convergence under noisy gradients. 

\subsection{Reducing convergence analysis to bounding population gradients}
\label{sec:overview_reduce}

We begin by providing some intuition regarding the gradient dynamics. First, recall the iteration for $u_t$:
\begin{equation*}
    u_{t+1}=\frac{u_t - \eta \hat \nabla L(u_t;x_t)}{\norm{u_t - \eta \hat \nabla L(u_t;x_t)}}
\end{equation*}

We formally analyze the SGD dynamics in Appendix~\ref{appendix:finite-sample}; for the sake of intuition, in this overview we will use $\hat{Q}_t$ to denote the error in approximating $u_{t+1}$ by ignoring the normalization, by writing
\begin{equation}
    u_{t+1} = u_t - \eta \hat \nabla L(u_t;x_t) + \hat Q_t\,. \label{eq:pre_projerror}
\end{equation}
Furthermore, decompose $\hat \nabla L(u_t;x_t)= \hat \nabla \Phi(u_t) + \hat \nabla E(u_t;x_t)$ where $\nabla E$ is a stochastic error term with mean 0, and $\hat \nabla \Phi$ is the population gradient. 

Now, we will show that the population gradient term $\hat \nabla \Phi(u_t)$ depends non-linearly on the projection $m_t \triangleq \langle u_t, u\rangle$, which will later help us derive a self-governing dynamics for $\langle u_t, u\rangle$. We begin by calculating the population gradient using Hermite analysis (see Appendix~\ref{app:gradient_derivation} for the derivation):

\begin{proposition}\label{prop:pop_gradient_pre}
    Given $l, s\in\mathbb{Z}_{\ge 0}$, define
    \begin{equation*}
        T(l,s)= \begin{cases}
            \norm{\sum_i \lambda_i w_i^{\otimes s}}_F^2 & l \text{ odd }\\
            k\left\langle \sum_i \lambda_i c_i w_i^{\otimes s}, \sum_i \lambda_i \hat c_i w_i^{\otimes s}\right\rangle& \text{ otherwise }
        \end{cases}
    \end{equation*}
    Define $h: \R\to\R$ by
    \begin{equation*}
        h(m) = 2 \sum_{l=0}^\infty \left(\frac{\xi^2}{k}\right)^{l+1} \left(\sum_{s=0}^\infty {l+s \choose l} (l+s+1) \mu_{l+s+1}(\sigma)^2 \left(\frac{1}{1+\xi^2/k}\right)^{l+s+1} T(l,s)\right) m^l\,.\label{eq:hdef_pre}
    \end{equation*}
    Then at any $\hat{u}\in\mathbb{S}^{d-1}$, the population spherical gradient is given by
    \begin{equation*}
        \hat{\nabla}\Phi(\hat{u}) \triangleq (I - \hat{u}\hat{u}^\top)\nabla \Phi(\hat{u}) = -h(\langle u,\hat{u}\rangle)(u - \hat{u}\langle\hat{u}, u\rangle)\,.
    \end{equation*}
    % \label{prop:pop-grad-derivation}
\end{proposition} 

\noindent We pause here to highlight the aforementioned connection to the literature on learning generalized linear models:

\begin{remark}[Generalizing GLM regression]
    Recall that the fine-tuning setting we study is a generalization of learning generalized models, and Proposition~\ref{prop:pop_gradient_pre} recovers a standard calculation in the literature on the latter. Indeed, if we let $\xi = \bar \xi \sqrt{k}$ and send $\bar \xi \to \infty$ the term
    \begin{equation*}
        \sum_{s=1}^\infty {l+s \choose l} (l+s+1) \mu_{l+s+1}(\sigma)^2 \left(\frac{1}{1+\xi^2 / k}\right)^s T(l,s)
    \end{equation*}
    in Eq.~\eqref{eq:hdef_pre} will vanish for all $l$. Then, $h(m)$ around $0$ reduces to
\begin{equation*}
    h(m) \approx \sum_{l=0}^\infty l \mu_{l+1} (\sigma)^2 m^l
\end{equation*}
This reproduces a well-known finding for learning generalized linear models, namely that the dynamics at initialization is governed by the information exponent, i.e. the degree of the first non-vanishing Hermite coefficient $\mu_p(\sigma)$. In that setting, if the information exponent is $l^*$, then $h(m)$ at initialization scales as $1/d^{l^*/2}$ and noisy gradient descent requires iteration complexity scaling as $d^{\Omega(l^*)}$~\cite{damian2022neural}. In contrast, as we will see, the behavior of $h(m)$ in our fine-tuning setting when $\bar \xi$ is bounded is very different, and the dynamics will not be sensitive to the information exponent.
\end{remark}

\noindent Going back to the expression in Proposition~\ref{prop:pop_gradient_pre}, note that the population gradient can be reduced to a term in the direction of spherical projection of $u$ onto $\hat u$, scaled by some non-linear function that only depends on $\langle \hat u, u\rangle$. With that, notice that $\langle \hat \nabla \Phi (\hat u), u\rangle = -h(\langle u, \hat u\rangle)(1-\langle u,\hat u\rangle^2)$, so that the population gradient in the ground truth direction only depends on the projection of the estimate $\hat u$ onto the ground truth $u$. The scaling factor $h(\langle u,\hat{u}\rangle)$ thus dictates the rate at which gradient descent moves towards the ground truth, but $h$ depends on the unknown level of correlation $\langle u, \hat{u}\rangle$ in a complicated, highly nonlinear fashion.

Let $m_t = \langle u_t, u\rangle$ denote the alignment between the $t$-th iterate with the ground truth direction $u$, so that by Eq.~\eqref{eq:pre_projerror},
\begin{equation*}
    m_{t+1} = m_t + \eta h(m_t)(1-m_t^2) + \eta \langle \hat \nabla E_t(u_t;x_t), u\rangle + Q_t\,,
\end{equation*}
where $Q_t = \langle \hat{Q}_t, u\rangle$ is the error from ignoring the normalization.

Before discussing how to analyze this, let us first consider what happens with the \emph{population dynamics}. Ignoring the error terms $E_j$ and $Q_j$ in the above, we get the recursion
% Nevertheless, even after ignoring the noise and assuming we have a population dynamics, it is not immediately clear from the form of $h$ that the dynamics should converge to the ground truth (or its negation, due to the inherent symmetry in the problem). For the sake of intuition, consider the population dynamics, ignoring the spherical projection
\begin{equation*}
    \overline{m}_{t+1}= \overline{m}_t + \eta h(\overline{m}_t)(1-\overline{m}_t^2)\,.
\end{equation*}
Rearranging, we have 
\begin{equation*}
    \frac{|1-\overline{m}_{t+1}|}{|1-\overline{m}_t|} = |1-\eta h(\overline{m}_t)(1+\overline{m}_t)|\,.
\end{equation*}
So if $h(\overline{m}_t)$ is uniformly bounded from below by some quantity $s$, then the alignment $m_t$ contracts towards $1$ at a geometric rate of $1 - \eta s$ in each step, regardless of whatever complex behavior $h$ exhibits over the course of training. 
% In particular, suppose $h(\overline{m}_t) \geq s$ throughout training, then we have $|1-\overline{m}_{t+1}|\leq |1-\overline{m}_t||1-\eta s|$, in which case the population dynamics is greatly simplified. 
Furthermore, notice that $\overline{m}_t$ would converge to $1$ only if $h(\overline{m}_t)$ is non-vanishing across the trajectory since otherwise, the algorithm would converge to a different stationary point.

Let us now turn back to analyzing the finite-sample dynamics. Unrolling the recursive expression and defining $E_t = \hat \nabla E(u_t;x_t)$, we obtain
\begin{equation}
    m_{t} = m_0 + \eta \sum_{j=0}^{t-1}h(m_j)(1-m_j^2) + \eta \sum_{j=0}^{t-1} \langle E_j, u\rangle + \sum_{j=0}^{t-1} Q_j\,. \label{eq:finitesample_recursion}
\end{equation}
Note that the terms $M_t =\eta \sum_{j=0}^{t-1} \langle E_j, u\rangle$ form a martingale, and $\sum_{j=0}^{t-1} Q_j$ is a stochastic error term. Over short time scales, these terms could potentially dominate the dynamics. However, over the course of training for $T$ iterations, their overall contribution scales with $\eta \sqrt{T}$, whereas the contribution of the population gradient term $\eta \sum_{j=0}^{t-1} h(m_j)(1-m_j^2)$ scales with $\eta T$, provided that the population gradients are not too small. By choosing $\eta, T$ appropriately, we can ensure that $\eta T = \Theta(1)$ while $\eta \sqrt{T} = o(1)$. The exact choice depends crucially on the \emph{signal-to-noise ratio} of the problem. In particular, if we have a lower bound $S_k$ on the correlations between the population gradient and $u$ (signal), and an upper bound $dV_k$ on the variances $\E_x[E_t^2]$ of the martingale increments (noise), then we show the sample complexity scales with $\frac{dV_k}{S_k^2}$, the inverse of the signal-to-noise ratio. 

Now, we formalize the intuition from the previous paragraphs and state the generic conditions under which we can show convergence of the online SGD dynamics to the ground truth and analyze the number of samples required. First, we require the following mild technical condition which is standard in the literature: 

\begin{condition}[Unbiased gradient estimates]    \label{condition:unbiased-gradients-modular}
    For all $\hat u$, the sample gradient is an unbiased estimate of the population gradient, i.e.,
    \begin{equation*}
        \hat \nabla_{\hat u}\Phi (\hat u) \triangleq \hat \nabla_{\hat u}\E_x[L(\hat u;x)] = \E_x[\hat \nabla_{\hat u}L(\hat u;x)]\,.
    \end{equation*}
\end{condition}
\noindent Note this holds when $\sigma$ is differentiable almost everywhere w.r.t. Gaussian measure, and $\sigma'$ has almost linear polynomial growth. In particular, it holds under~\Cref{assumption:activation}. This is because $\nabla_{\hat u} L(\hat u;x)$ has at most linear polynomial growth and thus can be bounded by a function $g_k(\langle \hat u,x\rangle)$ which has finite expectation under $x$. Then, the interchange of derivative and expectation follows from dominated convergence. 

Next, we need upper bounds on the variance of the empirical gradient and on the norm of the population gradient, in order to bound the martingale term $M_t$ and stochastic error term $Q_t$ in Eq.~\eqref{eq:finitesample_recursion}:
\begin{condition}[Gradient variance upper bounds]\label{condition:variances-modular}
For each $k$, and $p$, there exists some constant $V_k\geq 1$ that has at most polynomial growth in $k$ such that 
\begin{enumerate}[leftmargin=*]
    \item \emph{\bf Variance bound}: For all $u, \hat u$, we have $d^{-p}\,\E_x\|\hat \nabla_{u} L(\hat u; x)\|_2^{2p} \vee \E_x \langle\hat \nabla_{\hat u} L(\hat u; x), u\rangle^{2p} \leq (\mu_p V_{k})^p$ 
    \item \emph{\bf Population gradient bound}: For all $\hat u$, we have $\|\hat \nabla_{\hat u}\Phi(\hat u)\|^2\leq V_{k}$
\end{enumerate}   
where the $\mu_p$ is a constant that may depend on $p$ and the activation $\sigma$, but on nothing else.
\end{condition}

\noindent We will later prove that these upper bounds hold under the settings we consider, by appealing to appropriate tail bounds in \Cref{app:variances}.

Finally, as described above, a uniform lower bound on the function $h$ ensures that online SGD converges to the correct solution. We formalize this condition as follows. Recall that the population gradient is of the form $\hat \nabla_{\hat u} \Phi(\hat u) = -h(\langle \hat u, u\rangle) (u - \hat u \langle u,\hat u\rangle)$. 
 
\begin{condition}[Population gradient lower bound] \label{condition:pop-grad-form-modular}
There exists a constant $\max\{S_k, S_k^2\}\leq V_k$ that has at most polynomial decay, such that $h$ satisfies
\begin{equation*}
    h(\mathrm{sign}(h(0)) m) \mathrm{sign}(h(0)) \geq \frac{|h(0)|}{2}\geq S_k
\end{equation*}
for all $m \ge 0$.
\end{condition}

\noindent Establishing such a lower bound on the population gradients is the key technical step in our proofs, and in Sections~\ref{sec:overview_orth} and~\ref{sec:overview_sep} we describe the ideas that go into proving this condition holds in the settings we consider.

The three conditions above are sufficient conditions to conclude that online SGD can recover the ground truth and give an upper bound on the number of samples required. Formally, we have the following generic statement under~\Crefrange{condition:unbiased-gradients-modular}{condition:pop-grad-form-modular}, which we will use later to get formal finite-sample guarantees for the different fine-tuning regimes we consider.

\begin{theorem}\label{theorem:modular-convergence-dynamics}
Let \Crefrange{condition:unbiased-gradients-modular}{condition:pop-grad-form-modular} hold. Let $0 < \epsilon < 1$. Let $m_t=\langle u_t, u\rangle$ and set the learning rate $\eta = \frac{\delta}{dV_k}$ with scaling $\delta = \min\bigl\{\frac{S_k \epsilon^3}{4\mu_1 (\log dV_k)^2}, 1\bigr\}$, for total time $T=\lceil \alpha d V_k\rceil$ with time scaling $\alpha = \frac{4(\log dV_k)}{\epsilon \delta S_k}$ and initialization at $|m_0| \geq \beta/\sqrt{d}$ with $m_0 h(0) > 0$. 
    With probability at least $1 - o(1)$ the following holds for $T = \lceil\alpha dV_k\rceil$ and $T_{\rm weak} = \bigl\lceil \frac{4d V_k}{\delta S_k}\bigr\rceil = o(T)$:
    \begin{itemize}[leftmargin=*]
        \item (Weak recovery): $\sup_{t \le T_{\rm weak}} |m_t| \ge r$
        \item (Strong recovery): $|m_{T}| \ge 1-\epsilon$
    \end{itemize}
\end{theorem}

% remark that if you just want weak recovery, you can take eps accordingly to be O(1)

\noindent We prove this in Appendix~\ref{appendix:finite-sample}, drawing heavily upon the analysis in~\cite{arous2021online}. In the next two subsections, we show how to establish the conditions needed to apply~\Cref{theorem:modular-convergence-dynamics} in the settings we consider.

% Hence, the main goal of the subsequent analysis is to prove that $h$ indeed satisfies this lower bound property, and quantitatively determine what the lower bound is.  With that in mind, note that the complicated expression in \Cref{prop:pop_gradient_pre}, simplifies when we can simplify the expressions involving the moment tensors $\sum_{i=1}^k \lambda_i c_i w_i^{\otimes s}$. We can do this when (i) $w_i$ are orthonormal or (ii) $w_i$ are angularly separated so that for large $s$, the tensors $w_i^{\otimes s}$ are approximately orthonormal. Then, we initially consider orthonormal feature vectors $w_i$.

\subsection{Gradient bounds under orthonormal features}
\label{sec:overview_orth}

Here we assume $w_1,\ldots,w_k$ are orthonormal, so that the form of $T(l,s)$ in \Cref{prop:pop_gradient_pre} reduces to:
\begin{equation*}
    T(l,0) = \begin{cases}
    k\sum_{i,j=1}^k \lambda_i\lambda_j c_i \hat c_j & l \text{ even }\\
    \Bigl(\sum_{i=1}^k \lambda_i\Bigr)^2 & l \text{ odd}
    \end{cases} \qquad\text{ and }\qquad T(l, s\geq 1)= \begin{cases}
        k\sum_{i=1}^k \lambda_i^2 c_i \hat c_i & l \text{ even }\\
        \norm{\lambda}_2^2 & l \text{ odd }
    \end{cases}
\end{equation*}
which greatly simplifies our analysis since all the terms where $s\geq 1$ scale with the same expression. Then, notice that we can decompose $h$ into the odd powers of $l$ and even powers of $l$ as
\begin{align*}
    h(m) &= 2\Bigl[k\sum_{i,j=1}^k \lambda_i \lambda_j c_i \hat c_j\Bigr]\sum_{\substack{l\ge 0 \ \text{even}}} \Bigl(\frac{\xi^2}{k}\Bigr)^{l+1}(l+1)\mu_{l+1}(\sigma)^2 \Bigl(\frac{k}{k+\xi^2}\Bigr)^{l+1} m^l\\
    &\qquad\qquad +2\Bigl[k\sum_{i=1}^k \lambda_i^2 c_i \hat c_i \Bigr]\sum_{\substack{l \ge 0 \ \text{even}\\
    s\geq 1}} \Bigl(\frac{\xi^2}{k}\Bigr)^{l+1}{l+s \choose l}(l+s+1)\mu_{l+s+1}(\sigma)^2 \Bigl(\frac{k}{k+\xi^2}\Bigr)^{l+s+1} m^l +\sum_{\substack{l\ge 1\ \text{odd}}} b_l m^l\,,
\end{align*}
for some non-negative coefficients $b_l \geq 0$. Informally, over the randomness of $c,\hat c$, the typical magnitude of $k\sum_{i,j=1}^k \lambda_i \lambda_j c_i \hat c_j$ is $\Theta(k)$, and the typical magnitude of $k\sum_{i=1}^k \lambda_i^2 c_i \hat c_i$ is $\Theta(\sqrt{k})$. Hence, first, we show that the odd terms with $s=1$ dominate the even terms with $l>0, s=0$. Then, this will mean that the sign of all of these terms will be governed by the sign of $m$. Hence, as long as $m$ has the same sign as $h(0)$, we expect $h(m)$ to be lower bounded throughout training. To this end, we first state bound on the even terms with $s=0$ and then discuss how the $s>0$, $l=0$ terms are bounded (see \Cref{app:orthonormal-lower-bound} for proof).
\begin{claim}[Even $l$, $s=0$ contribution]\label{claim:even-s-0}
 Let \Cref{assumption:orthogonal} hold. With probability $1-\exp\{-\frac{2k}{e\xi^2}\}$ over the randomness of $c, \hat c$, the following holds.
    \begin{align*}
    \mathrm{sign}(m) &\Bigl(2\sum_{\substack{l \text{ odd }\\ s=1}} \Bigl(\frac{\xi^2}{k}\Bigr)^{l+1} {l+s \choose l}(l+s+1)\mu_{l+s+1}(\sigma)^2 \Bigl(\frac{1}{1+\xi^2/k}\Bigr)^{l+s+1} T(l,s) \\
   & +2\sum_{\substack{l>0\\ \text{ even }\\ s=0}} \Bigl(\frac{\xi^2}{k}\Bigr)^{l+1} {l+s \choose l}(l+s+1)\mu_{l+s+1}(\sigma)^2 \Bigl(\frac{1}{1+\xi^2/k}\Bigr)^{l+s+1} T(l,s)\Bigr)\geq 0\,.
\end{align*}
\end{claim}
\noindent The proof of this claim relies on using concentration bounds on $T(l,s)$ for even $s$, and showing that the coefficients of the even terms are bounded by the odd terms up to some constant. This proof is deferred to Appendix~\ref{app:orthonormal-lower-bound}. Then, this result should hold with high probability when $\xi = \overline{\xi} \sqrt{k}$, where the probability is exponentially high as $\overline{\xi} \to 0$. Then, we can group these terms together with the odd terms, so that we have the following:
\begin{align}
    h(m) &= 2 \left[k \sum_{i,j=1}^k \lambda_i \lambda_j c_i \hat c_j\right] \mu_{1}(\sigma)^2 \Bigl(\frac{\xi^2}{k+\xi^2}\Bigr) \label{eq:h-first-term-frob} \\
    &\qquad\qquad +2\Bigl[k\sum_{i=1}^k \lambda_i^2 c_i \hat c_i \Bigr]\sum_{\substack{l \ge 0 \ \text{even}\\
    s\geq 1}} \Bigl(\frac{\xi^2}{k}\Bigr)^{l+1}{l+s \choose l}(l+s+1)\mu_{l+s+1}(\sigma)^2 \Bigl(\frac{k}{k+\xi^2}\Bigr)^{l+s+1} m^l \label{eq:h-second-term-frob}\\
    &\qquad \qquad \qquad + \tilde h(m),
\end{align}
where $\tilde h(m) m > 0$, with high probability due to \Cref{claim:even-s-0}. In this case, note that we have two cases
\begin{enumerate}
    \item $\mu_{1}(\sigma) \neq 0$: In this case, note the term in \cref{eq:h-first-term-frob} is $\Theta(k)$, whereas the term in \cref{eq:h-second-term-frob} is $O(\sqrt{k})$ since the sum in \cref{eq:h-second-term-frob} is $O(1)$. Hence, we show that the second term is negligible relative to the first, and furthermore that the sign of $h(0)$ is the same as the sign of  \cref{eq:h-first-term-frob}.
    \item $\mu_{1}(\sigma) = 0$: In this case, $h(0)$ has the same sign as all the even terms, so as long as $m h(0) > 0$, we can lower bound $|h(m)|$ with $\Theta(\sqrt{k})$. 
\end{enumerate}
Then, the rest of the proof is giving a uniform lower bound on $|h(m)|$ training when $m h(0) > 0$. We do this by showing that the quantities
\begin{equation*}
    \left(\sum_{i=1}^k \lambda_i c_i\right) \left(\sum_{i=1}^k \lambda_i \hat c_i \right) \qquad\text{ and }\qquad \sum_{i=1}^k \lambda_i^2 c_i \hat c_i
\end{equation*}
are upper and lower bounded with high probability. The upper bounds follow directly since these are quadratic polynomials of Rademacher variables. For the lower bounds, we prove the anti-concentration of these quantities by showing that as functions of $c_i$, these are ``low-influence'' quadratic polynomials, so a central limit theorem type result allows us to swap the $c_i$ with Gaussian variables (see \Cref{app:anti-concentration}). Here, we show the following:

\begin{lemma}[Population gradient lower bounds, orthonormal case]\label{lemma:pop-gradient-bound-orthonormal}
Suppose \Crefrange{assumption:normalize}{assumption:quantization} hold and $\langle w_i, w_j\rangle = 0$ for all $i\neq j$. Let $s^*$ be the smallest $s\geq 1$ such that $\mu_s(\sigma) \neq 0$.  Then, with probability $1-o(1) - \exp\left\{-\frac{2k}{e\xi^2}\right\}- O(\gamma^{1/2})$, for $m  h(0)\geq 0$ we have $h(m) \mathrm{sign}(h(0)) \geq |h(0)|/2$ and
\begin{equation*}
    |h(0)| \geq \left(\frac{k}{k+\xi^2}\right)^{s^*} C_\sigma \gamma \xi^2 \cdot \left.\begin{cases}
        1 & \mu_1(\sigma) \neq 0\\
        \frac{1}{\sqrt{k}} & \mu_1(\sigma) = 0
    \end{cases}\right\}
\end{equation*}
\end{lemma}
\noindent The proof of this lower bound is deferred to Appendix~\ref{app:orthonormal-lower-bound}.

\subsection{Gradient bounds under angularly separated features}
\label{sec:overview_sep}

Whem $\xi = \Theta(1)$, we can relax the orthogonality assumption, and instead work with angularly separated features $w_i$, as stated in \Cref{assume:angle}. For notational simplicity, we consider the case $\xi = 1$, but the analysis is identical when $\xi = \Theta(1)$. Recall the non-linear function $h$ that drives the population dynamics, which simplifies to:
\begin{equation*}
    h(m) = 2 \sum_{l=0}^\infty \left(\frac{1}{k}\right)^{l+1} \sum_{s=0}^\infty {l+s \choose l} (l+s+1) \mu_{l+s+1}(\sigma)^2 \left(\frac{k}{k+1}\right)^{l+s+1} T(l,s) m^l
\end{equation*}
Now, note if one could truncate the infinite sum over $s$ at some finite value, the contribution of the higher order terms could be bounded. However, this is not obvious. In fact, one can verify that
\begin{equation*}
    \sum_{s=0}^\infty {l+s \choose l} \left(\frac{k}{k+1}\right)^{s+1} = \Theta(k^l)
\end{equation*}
so a naive truncation of the sum over $l$ will not allow to simplify the population gradient. However, note that within $T(l,s)$, the contribution of the non-diagonal ($i\neq j$) terms in the inner products $\langle \sum_i \lambda_i c_i w_i^{\otimes s},\sum_i \lambda_i \hat c_i w_i^{\otimes s} \rangle = \sum_{i,j=1}^k \lambda_i \lambda_j c_i\hat c_j \langle w_i, w_j\rangle^s$ vanishes as $s\to \infty$, as $\langle w_i, w_j\rangle$ is bounded away from $1$ by Assumption~\ref{assume:angle}. Indeed, if $|\langle w_i, w_j\rangle| \leq \nu$ for $i\neq j$, then $|\langle w_i, w_j\rangle|^{\frac{\gamma}{\nu}}\leq \exp\{-\gamma\}$. 
Furthermore, note that for the diagonal ($i=j$) terms, we can use the decay of the $\mu_{p}$ to bound their contribution to the terms in $h(m)$ corresponding to large $s$.
Finally, for the small $s$ terms (e.g. $s= O(\sqrt{k})$) we simply use the fact that the sum over $l$ in the definition of $h(m)$ has a factor $(1/k)^{l+1}$ which is decaying as $l$ increases. 

In summary, we can handle the small $s$ terms because their contribution is not that large, and once we reach $s$ that gives non-vanishing contribution, we utilize the decay of $\mu_p(\sigma)$ and the angular separation of the $w_i$'s. In particular, we prove the following lemma that bounds the contribution of higher order terms in $m$ have an even exponent (see \Cref{app:separated} for proof):

\begin{proposition}
    % Let $\sigma$ be such that there exists some $\rho, C_\sigma > 0$ for which $|\mu_{p}(\sigma)|\leq C_\sigma p^{-1-\rho}$. 
    Suppose \Crefrange{assumption:normalize}{assumption:activation} hold. Then, with probability at least $1-\frac{1}{k^3}$ over the randomness of $c,\hat c$, for $\epsilon = \min \{\frac{\rho}{4}, 1-\frac{1}{1+2\rho}\}$ we have
    \begin{equation*}
        \sum_{n=0}^\infty \left(\frac{1}{k}\right)^{2n+2} \sum_{s=0}^\infty {2n+2+s \choose 2n+2} (2n+s+3) \mu_{2n+s+3}(\sigma)^2 \left(\frac{k}{k+1}\right)^{2n+s+3} \left\langle \sum_{i=1}^k \lambda_i c_i w_i^{\otimes s}, \sum^k_{i=1} \lambda_i \hat c_i w_i^{\otimes s}\right\rangle = O(\lambda_{\max}^2 k^{-\frac{1}{2}-\epsilon})\,.
    \end{equation*}
    \label{proposition:separated-even-terms-bound}
\end{proposition}

\noindent Once we bound the higher order terms in $m$ with even exponent, it remains to show the constant term $h(0)$ dominates the even terms. This follows immediately using the anti-concentration results for $h(0)$, since $|h(0)| = \Omega(\frac{1}{\sqrt{k}})$ with high probability. Then, since the odd terms will have the same sign as $m$, the rest of the proof is similar to the orthonormal case. In this case, we get the following lower bound on $h$:

\begin{lemma}\label{lemma:pop-gradient-bound-separated}
     With probability $1-O(\gamma^{1/2})-o\left(\frac{\lambda_{\max}^2}{\lambda_{\min}^2}\right)$, for $m h(0)\geq 0$ we have
        \begin{equation*}
            h(m) \mathrm{sign}(h(0)) \geq \frac{|h(0)|}{2}\geq  \frac{\gamma \lambda_{\min}^2}{2\sqrt{k}}
        \end{equation*}
\end{lemma}

\subsection{Finite sample analysis and putting everything together}
\label{sec:overview_finite}

Now that we have shown that the signal term $h(m)$ is non-vanishing and proven lower bounds on it, it remains to upper bound the variances of the gradients and the population gradient. To that end, we prove the following:
\begin{lemma}[General upper bounds] \label{lemma:variance-bounds}
Under \Crefrange{assumption:normalize}{assumption:activation}, we have the following:
    \begin{enumerate}[leftmargin=*]
        \item \emph{\bf Variance upper bound}: $\sup_{u, \hat u, c,\hat c}\Bigl\{\E_{x} \norm{\frac{\hat \nabla L(\hat u;x)}{\sqrt{d}}}^{2p} \ \vee \  \E_{x}  |\langle \hat \nabla L(\hat u;x), u\rangle|^{2p}\Bigr\}^{1/p}\leq  C_{p,\sigma} \lambda_{\max}^4 \frac{k^3 \xi^2 \min\{k, 4\xi^2\}}{k + \xi^2}$
        \item \emph{\bf Population gradient upper bound}: $\norm{\hat \nabla \Phi(\hat u)} \leq C_\sigma \lambda_{\max}^2 \frac{k \xi^2}{1+\xi^2/k}$
\end{enumerate}
\end{lemma}
\noindent We prove this statement in \Cref{app:variances}. Now, we are at a position to combine the lower bounds for the population gradient and the upper bounds for the variance to get a convergence and iteration complexity bound. We begin by proving our main results for the case of orthogonal pretrained features:

\begin{proof}[Proof of \Cref{thm:orth-spectral-main}]
    For the first point, notice that the variance upper bound in~\Cref{lemma:variance-bounds} and the population gradient lower bound in~\Cref{lemma:pop-gradient-bound-orthonormal} imply that Conditions~\ref{condition:variances-modular} and~\ref{condition:pop-grad-form-modular} hold with 
    \begin{align*}
        S_k &= \frac{\gamma \lambda_{\min}^2\mu_1(\sigma)^2}{1+\xi^2/k}\\
        V_k &= C_{p,\sigma}\lambda_{\max}^4 k^2
    \end{align*}
    for some small $\gamma$ with probability $1-o\left(\frac{\lambda_{\max}^2}{\lambda_{\min}^2}\right) - O(\gamma^{1/2})$. Then, applying \Cref{theorem:modular-convergence-dynamics} with the set $S_k, V_k$ and $\epsilon$, we get the desired result. The second case follows similarly.
\end{proof}

\begin{proof}[Proof of \Cref{thm:orth_frob_main}]
For $\overline{\xi}\leq 1$, the results in \Cref{lemma:variance-bounds} and \Cref{lemma:pop-gradient-bound-orthonormal} imply that \Cref{condition:variances-modular}, \Cref{condition:pop-grad-form-modular} hold with 
    \begin{align*}
        S_k &= \frac{\gamma k \lambda_{\min}^2\mu_1(\sigma)^2}{2}\\
        V_k &= C_{p,\sigma}\lambda_{\max}^4 \overline{\xi}^2 k^4
    \end{align*}
    for some small $\gamma$ with probability $1-o(1) - \exp\{-\frac{2k}{e\overline{\xi}^2}\}-O(\gamma^{1/2})$. Then, applying \Cref{theorem:modular-convergence-dynamics} with the set $S_k, V_k$ and $\epsilon$, we get the desired result. The second case follows similarly.
\end{proof}

\noindent Finally, we can prove our main result for the case of angularly separated pretrained features:

\begin{proof}[Proof of \Cref{thm:separated_main}]
    Note that \Cref{lemma:variance-bounds} and \Cref{lemma:pop-gradient-bound-separated} imply that \Cref{condition:variances-modular}, \Cref{condition:pop-grad-form-modular} hold with 
    \begin{align*}
        S_k &= \frac{\gamma \lambda_{\min}^2}{\sqrt{k}}\\
        V_k &= C_{p,\sigma}\lambda_{\max}^4 k^2
    \end{align*}
    for some small $\gamma$ with probability $1-o\left(\frac{\lambda_{\max}^2}{\lambda_{\min}^2}\right) - O(\gamma^{1/2})$. Then, applying \Cref{theorem:modular-convergence-dynamics} with the set $S_k, V_k$ and $\epsilon$, we get the desired result. The second case follows similarly.
\end{proof}

\subsection{Separations between fine-tuning and feature-learning}
\label{sec:lowerbound}

Note that if the teacher model is of the form $f(x) = \sum_{i=1}^k \lambda_i \sigma(\langle \tilde w_i, x\rangle)$ where the $\tilde w_i$ are orthonormal, $f$ will be hard to learn with a CSQ algorithm if the information exponent of $\sigma$ is large. Hence, in this section we show an example of a base network whose perturbation has orthonormal weights, and show separations between fine-tuning and learning from scratch using this example. We aim to construct $w_i + c_i u \perp w_j + c_j u$ for $i \neq j$. Notice that when $u\perp w_i$, this is equivalent to $\langle w_i, w_j\rangle = - c_i c_j$. Hence, if we can control the pairwise correlations of the $w_i$ as we want, we can construct this example. 

For the sake of intuition, consider the following example for $k=4$, where each row is a $w_i$, with $c_i = (-\frac{1}{2})^i$.
\begin{align*}
W
    =\left[\begin{array}{ccccccccccc}
         \frac{1}{2} & \frac{1}{2} & \frac{1}{2} & 0 & 0 & 0 & \frac{1}{2} & 0 & 0 & 0 & 0 \\
         \frac{1}{2} & 0 & 0 & \frac{1}{2} & \frac{1}{2} & 0 & 0& \frac{1}{2} & 0 & 0 &0 \\
         0 &  -\frac{1}{2} &0 & \frac{1}{2} & 0 & \frac{1}{2} & 0 & 0 & \frac{1}{2} & 0 & 0\\
         0 & 0 & \frac{1}{2} & 0 & -\frac{1}{2} & \frac{1}{2} & 0 & 0 & 0 & \frac{1}{2} & 0 \\
    \end{array}\right]
\end{align*}
We aim to generalize this example to general $k$ in the following proposition.

\begin{claim}
    When $d > 1+\frac{k(k+1)}{2}$, for $\lambda_i = 1$, there exists unit norm weights $\{w_i\}_{i=1}^k$, a perturbation $u \perp \mathrm{span}(w_i)$, weights $c_i \in \left\{\pm \frac{1}{\sqrt{k}}\right\}$, such that $\frac{\langle w_i + c_i u, w_j + c_j u\rangle}{\norm{w_i + c_iu} \norm{w_j + c_j u}} = \delta_{ij}$. \label{claim:hard-from-scratch}
\end{claim}
\begin{proof}
    We are looking for a setup where $\langle w_i, w_j\rangle = - c_i c_j$. We will construct $k$ vectors that pairwise only share one non-zero coordinate. For $l\in [d], l \leq k$, let $(w_{l})_l = \frac{1}{\sqrt{k}}$. Then, for a given coordinate $l \in [d], l > k$, we want exactly two $w_i,w_j$ to have non-zero $l$'th coordinate. Since $d-k > 1+ {k \choose 2}$, we can assign every pair $(i,j)$ with $i\neq j$ a coordinate, and we will have at least 1 coordinate left. Then, notice that the inner product $\langle w_i, w_j\rangle$ for $i\neq j$ only depends on $1$ coordinate, which is unique for every $(i,j)$. We choose the magnitude of this entry to be $\frac{1}{\sqrt{k}}$. Then, for any $c\in \left\{\pm \frac{1}{\sqrt{k}}\right\}^k$ we can simply choose the signs of these coordinates accordingly to ensure $\langle w_i, w_j\rangle = - c_i c_j$. Notice that each $w_i$ has unit norm, and there is a coordinate, which we can WLOG assume to be the $p\triangleq\frac{k(k+1)}{2}$'th coordinate, that is zero for all $w_i$. We let $u= e_p$.

    Then, notice that $\frac{\langle w_i + c_i u, w_j + c_j\rangle}{\norm{w_i+ c_iu}\norm{w_j + c_j u}}=\frac{\langle w_i, w_j\rangle + c_i c_j}{\norm{w_i+ c_iu}\norm{w_j + c_j u}}= 0$ for $i \neq j$, as desired. 
\end{proof}

\begin{proposition}
Let $\xi =1$, and consider the example in \Cref{claim:hard-from-scratch}. Suppose $\sigma = h_p$ is the $p$'th Hermite coefficient for some $p > 2$. Then, $h(m) = 2p\left(\frac{k}{k+1}\right)^{p} \tilde h(m)$ where
    \begin{equation*}
    \tilde h(m) = \sum_{i=1}^k \lambda_i^2 c_i\hat c_i + O\left(\frac{\lambda_{\max}^2}{k}\right)
\end{equation*}
Moreover, with high probability over the choice of $\hat c$, we have $ h(m)\mathrm{sign}(h(0)) \geq \frac{|h(0)|}{2}$. 
\end{proposition}
\begin{proof}
    Initially, note 
    \begin{equation*}
        h(m) = 2p\left(\frac{k}{k+1}\right)^{p} \sum_{i,j=1}^k \lambda_i \lambda_j c_i \hat c_j (\langle w_i, w_j\rangle + c_i \hat c_j m)^{p-1}
    \end{equation*}
    In this case, notice that because $|\langle w_i, w_j\rangle|\leq \frac{1}{k}$ for $i \neq j$, we have
\begin{equation*}
    \left|\sum_{i,j=1}^k \lambda_i \lambda_jc_i \hat c_j (\langle w_i, w_j\rangle + c_i \hat c_j \langle u, \hat u\rangle)^{p-1}-\sum_{i=1}^k \lambda_i^2 c_i \hat c_i (1+c_i \hat c_i \langle u, \hat u\rangle)^{p-1}\right|\leq  \sum_{i\neq j}^{k} \left| \lambda_i \lambda_j c_i \hat c_j \frac{2}{k^{p-1}}\right|\leq \frac{\lambda_{\max}^2}{k^{p-2}}
\end{equation*}
Hence, defining $\tilde h(m) = 2p \left(\frac{k}{k+1}\right)^{p}$ to factor out the constant, we have
\begin{equation*}
    \tilde h(m) = \sum_{i=1}^k \lambda_i^2 c_i \hat c_i (1+c_i\hat c_i m)^{p-1} + O\left(\frac{\lambda_{\max}^2}{k^{p-2}}\right)
\end{equation*}
Then, expanding the diagonal term, note
\begin{align*}
    \sum_{i=1}^k c_i \hat c_i \lambda_i^2 (1+c_i \hat c_i \langle u, \hat u\rangle)^{p-1}= \sum_{s=0}^{p-1} {p-1 \choose s}\sum_{i=1}^k \lambda_i^2 (c_i \hat c_i)^{s+1} \langle u, \hat u\rangle^{s}= \sum_{i=1}^k \lambda_i^2 c_i \hat c_i +O\left(\frac{\lambda_{\max}^2}{k}\right)
\end{align*}
Then, for $p\geq 3$, we have
\begin{align*}
    \tilde h(m) = \sum_{i=1}^k \lambda_i^2 c_i\hat c_i + O\left(\frac{\lambda_{\max}^2}{k}\right)
\end{align*}
Then, over the randomization of $\hat c$, with high probability, we have $h(0) = \Omega\left(\frac{\lambda_{\min}^2}{\sqrt{k}}\right)$ due to anti-concentration (\Cref{lemma:psd-anti-conc}). Then, with high probability $h(m)\mathrm{sign}h(0) \geq \frac{|h(0)|}{2}$ uniformly. 
\end{proof}

\noindent Hence, in the construction given in \Cref{claim:hard-from-scratch}, even though the $c_i$'s are non-random, we still have with high probability over the randomization of $\hat c$ that $h$ satisfies \Cref{condition:pop-grad-form-modular}. Then, we have the following
\begin{theorem}
For the teacher and base networks defined in \Cref{claim:hard-from-scratch}, fine-tuning with online SGD learns the teacher network perturbation $u$ in $O(\frac{dk^2}{\epsilon^4})$ samples, whereas
    training from scratch using any CSQ algorithm requires at least $O(d^{p/2})$ queries or $\tau= O(d^{-p/4})$ tolerance.
    \label{thm:hard-from-scratch}
\end{theorem}
\begin{proof}
    The first part follows directly from the fact that $h$ satisfies the gradient lower bound in \Cref{condition:pop-grad-form-modular} with a $\Omega(\frac{\lambda_{\min}^2}{\sqrt{k}})$ lower bound, and \Cref{theorem:modular-convergence-dynamics}. For training from scratch, notice that the target model is of the form
    \begin{equation*}
        f(x)=\sum_{i=1}^k\lambda_i h_p(\langle v_i,x\rangle)
    \end{equation*}
    where the $v_i$ are orthonormal. Fix $k$. Then, we can embed $f$ into a random $k$ dimensional subspace $M$ by rotating the $v_i$ (since the vectors $w_i + c_i u$ can all be rotated without effecting the construction). The CSQ lower bound in \citep[Proposition 6]{abbe2023sgd} states that any CSQ algorithm using $n$ queries with tolerance $\tau$ cannot achieve less than some small $c > 0$ error with probability $1-\frac{Cn}{\tau^2}d^{-\frac{p}{2}}$. Hence, to achieve constant probability of success, one either needs $n = \Theta(d^{p/2})$ queries or tolerance $\tau =\Theta(d^{-p/4})$. 
\end{proof}

\section{Numerical simulations}

In this section we empirically demonstrate (i) the robustness of our theory even as one deviates from the simplifying assumptions we make, (ii) the distinction between the regime we study versus the kernel regime, and (iii) the robustness of the training dynamics to the choice of activation function. In particular, for (i) we provide simulations in which we relax the assumption that $c_i$ are quantized, and we also compare the cases when $\hat c$ is frozen versus jointly trained with $\hat u$. For (ii), we show that linearized networks (kernel approximation) fail at $\xi = \Theta(\sqrt{k})$, and also illustrate some interesting behavior in the joint training of $\hat u$ and $\hat c$.  For (iii), we corroborate our theoretical finding that the dynamics of fine-tuning are benign essentially regardless of the choice of $\sigma$, in sharp contrast to what happens when one trains from scratch.

We use the ReLU activation throughout our simulations. We let $f(x) =\frac{1}{\xi} \sum_{i=1}^{k} \lambda_i \sigma(\langle v_i, x\rangle)$ where $v_i = \frac{k}{k+\xi^2} (w_i + \xi c_i u)$ where the $1/\xi$ is to keep the magnitude of gradients consistent. Throughout our simulations, we set $d=2000$, $k=50$, and sample the $w_i\in \S^{d-1}$ and $c \in \S^{k-1}$ uniformly at random.
 
First, in the $\xi = \Theta(1)$ scaling, we plot 10 training curves for random problem instances (see below) for joint training (Figure~\ref{fig:train-plots}(a)) and when $\hat c$ is frozen (Figure~\ref{fig:train-plots}(b)). Notably, we see that while freezing $\hat c$ leads to somewhat longer time scales in training, the qualitative behavior of $\langle u_t, u\rangle$ is similar across the two settings. 

Next, we consider the $\xi = \Theta(\sqrt{k})$ scaling, keeping the rest of the setup the same as above. We plot low-rank fine-tuning in orange ($\hat u$ and $\hat c$ are jointly trained) and linearized training in blue. For the linearization, we Taylor expand around the base model. In Figure~\ref{fig:linearized-network}(a), We demonstrate that linearized dynamics do not explain fine tuning in this regime. Furthermore, when jointly training $\hat u$ and $\hat c$, we observe there is an initial phase where the loss is high and $\langle u_t, u\rangle$ is increasing but $\langle c_t, c\rangle$ stays at a low level (Figure~\ref{fig:linearized-network}(b)). This suggests that the initial phase of joint training might be similar to the training with frozen $\hat c$.

\begin{figure}[H]
\centering
\subfloat[Joint training of $\hat c$ and $\hat u$, $\xi=\Theta(1)$]{\includegraphics[width = 0.45\textwidth]{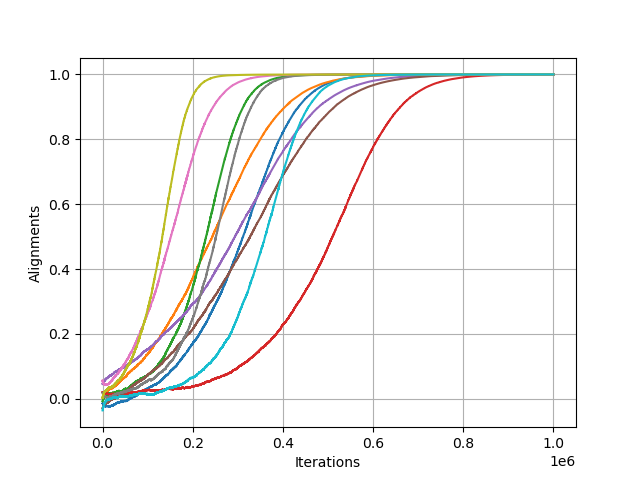}}  \hspace{0.5cm}
\subfloat[Freezing $\hat c$ and training $\hat u$, $\xi=\Theta(1)$]{\includegraphics[width = 0.45\textwidth]{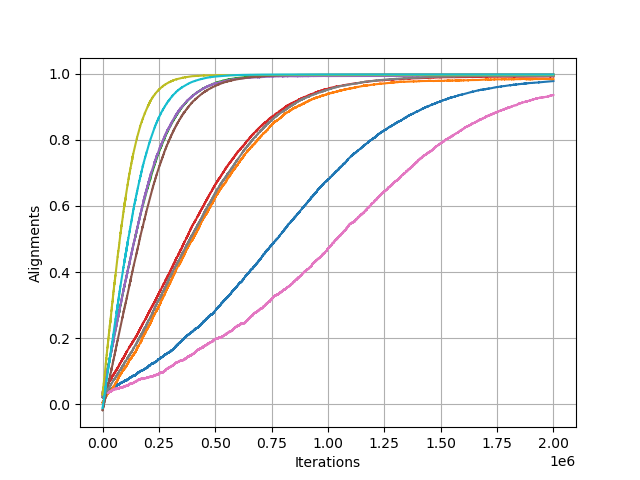}}
\caption{Evolution of $\langle u_t, u\rangle$ during online SGD for 10 random instances with joint and frozen-$\hat c$ training. Though time scales differ between (a) and (b), trajectories exhibit similar behavior.}
\label{fig:train-plots}
\end{figure}

\begin{figure}[H]
\centering
\subfloat[Loss curves for linearized vs. original network]{\includegraphics[width = 0.45\textwidth]{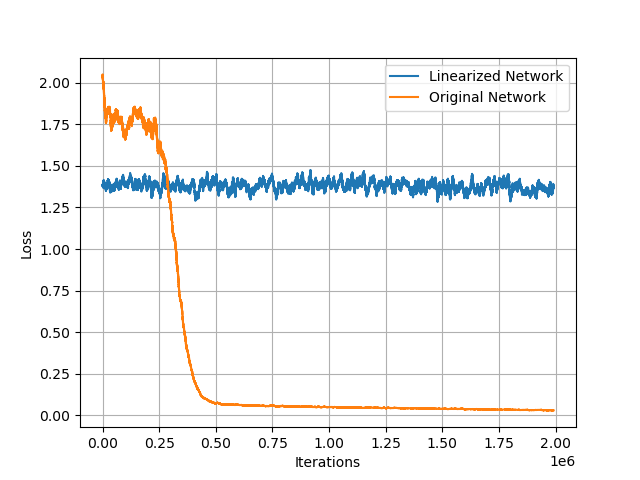}}  \hspace{0.5cm}
\subfloat[Trajectories of $\langle u_t, u\rangle$ and $\langle c_t, c\rangle$ when $\xi = \Theta(\sqrt{k})$]{\includegraphics[width = 0.45\textwidth]{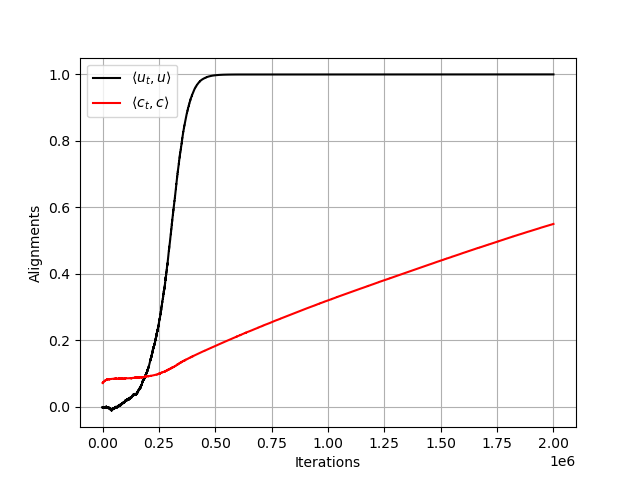}}
\caption{Linearized networks fail in low-rank fine-tuning, and cannot achieve small loss. When jointly training $\hat u$ and $\hat c$, we observe incremental behavior in learning, where learning $c$ becomes easier when $u$ is learned to a certain level.
% \sitan{grouping these under one figure might send the wrong message that the figure on the right is also for linearized training}
}
\label{fig:linearized-network}
\end{figure}

% PART AFTER THIS IS FOR THE ARXIV VERSION

% ACTIVATION DEPENDENCE

\noindent \textbf{Activation robustness.} As our analysis suggests, our findings are not too sensitive to the choice of the activation in the fine-tuning regime ($\xi = O(\sqrt{k})$). To reinforce this point, we overlay training plots for various choices of activations, namely $\sigma \in \{\mathrm{ReLU}, \mathrm{Sigmoid}, \mathrm{Quadratic}, \mathrm{He}_3\}$. In this setting, we jointly train $\hat c, \hat u$ and do not enforce $\hat u$ to be orthogonal to $w_i$ and plot $\langle u_t, u\rangle$ over time in \Cref{fig:activation-robustness}. We observe that the qualitative behavior (e.g. no saddle behavior at initialization) is universal across the different choices of activation, which supports our theoretical findings.

\begin{figure}[H]
\centering
\subfloat[Evolution of overlap $\langle u_t, u\rangle$ for $\xi=1$ and different activations]{\includegraphics[width = 0.45\textwidth]{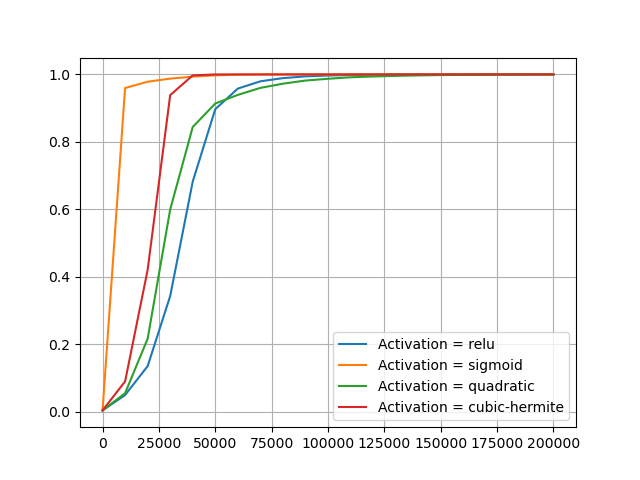}} \hspace{0.5cm}
\subfloat[Evolution of overlap $\langle u_t, u\rangle$ (y-axis) for $\xi=\Theta(\sqrt{k})$ and different activations]{\includegraphics[width = 0.45\textwidth]{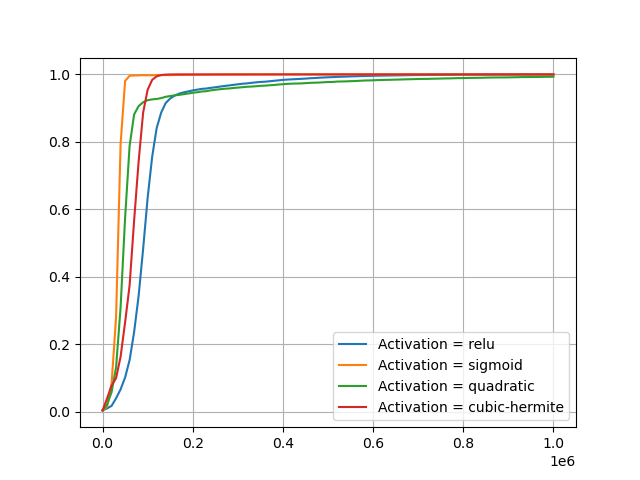}}
\caption{Evolution of overlap $\langle u_t, u\rangle$ during online SGD, under scaling (a) $\xi =1$ and (b) $\xi = \Theta(\sqrt{k})$ for different activations. As opposed to learning single-index models, or multi-index models from scratch, the fine-tuning regime is not too sensitive to the choice of activation. Namely, the iteration complexity of fine tuning with SGD does not depend sensitively on information exponent.}
\label{fig:activation-robustness}
\end{figure}

% ORTHOGONALITY ASSUMPTION

\noindent\textbf{Testing the orthogonality assumption.} In addition to the robustness checks w.r.t. to the choice of activation in the fine-tuning regime, we simulate the non-orthogonal case, i.e. when $u$ is not orthogonal to the $w_i$ \Cref{fig:orthogonality-assumption}. In particular, define $\Pi_W$ to be the orthogonal projector onto the span of $W$. We let $\alpha = \norm{\Pi_W u}$, and sample $u$ as 
\begin{equation*}
    u = \alpha u_1 + \sqrt{1-\alpha^2} u_2
\end{equation*}
where the unit vector $u_1$ is sampled uniformly from the span of $W$, and the unit vector $u_2$ is sampled uniformly from the orthogonal subspace.  With the ReLU activation and when $k = 100$, $d=1000$, we see that the evolution of the overlap $\langle u_t, u\rangle$ over time is qualitatively the same across different levels of parallel component $\norm{\Pi_w u}$. This might indicate that under non-pathological activations and initializations, the non-orthogonal setting might be similar to the orthogonal setting. In principle, one can obtain an analytical expression for the population loss with the ReLU activation, and we leave this as an interesting direction for future work.

\begin{figure}[H]
\centering
\subfloat[Violating the orthogonality assumption and varying $\norm{\Pi_{W} u}$, when $k \ll d$ with $\sigma = \mathrm{ReLU}$.]{\includegraphics[width = 0.45\textwidth]{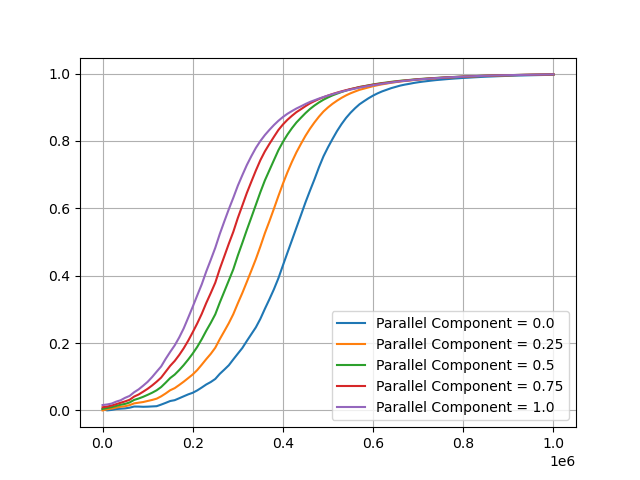}} \hspace{0.5cm}
\subfloat[Same setup as (a), but different (random) problem instances for $\norm{\Pi_{W} u}=1/2$]{\includegraphics[width = 0.45\textwidth]{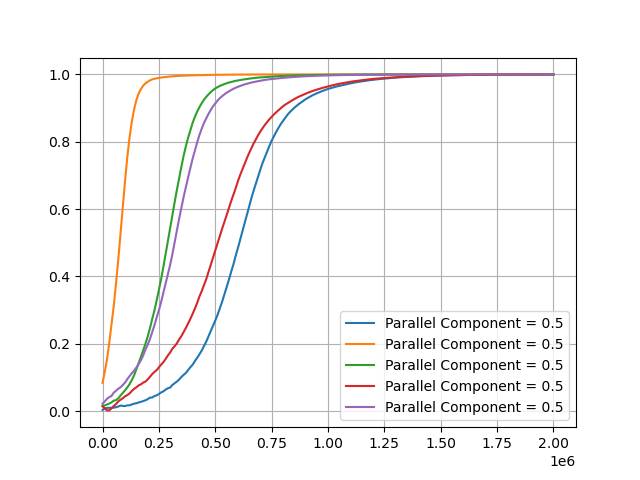}}
\caption{Evolution of $\langle u_t, u\rangle$ during online SGD for (a) varying levels of $\norm{\Pi_W u}$ (violating the orthogonality assumption) (b) multiple runs for $\norm{\Pi_W u}=\frac{1}{2}$. (a) In certain non-pathological initializations and scales $\xi$, the orthogonal case might capture behavior related to the non-orthogonal case. (b) Over multiple runs with $\norm{\Pi_W u}=\frac{1}{2}$ we see a generic S-curve behavior for ReLU activation.} 
\label{fig:orthogonality-assumption}
\end{figure}

% FEATURE LEARNING
\noindent\textbf{Interpolating between fine-tuning and feature learning.} Recall that our proposed model for low-rank fine tuning allows us to interpolate between feature learning and fine tuning. Now, we demonstrate that training the low-rank perturbation can take significantly many samples when the perturbation is large. To do this, we run various experiments with $\sigma(a) = \mathrm{He}_3(a)$ and vary the scaling $\xi$. We choose the third hermite activation to clearly illustrate the differences in sample complexity between feature learning and fine tuning. For these simulations, we jointly train $c$ and $u$, and plot the evolution of the overlap $\langle u_t, u \rangle$ over the course of training. We also do not necessarily enforce $u_t$ to be orthogonal to the $w_i$, to illustrate the robustness of the results to the choice of algorithm \Cref{fig:xi-scaling-feature-learning}. 

\begin{figure}[H]
\centering
\subfloat[Joint training of $\hat c$ and $\hat u$ for various choices of $\xi$ (short time-scale)]{\includegraphics[width = 0.45\textwidth]{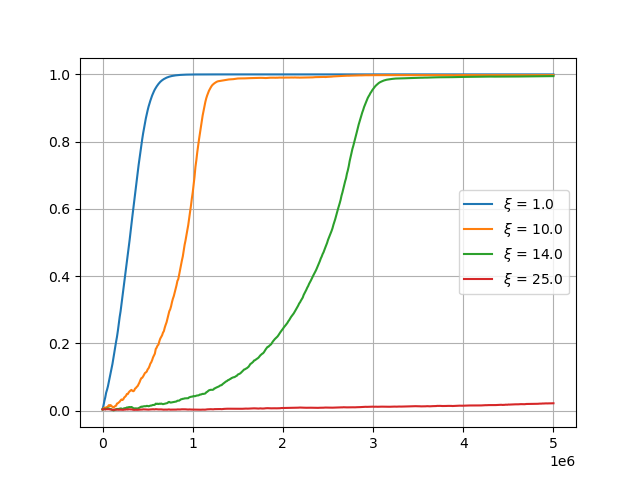}} \hspace{0.5cm}
\subfloat[Joint training of $\hat c$ and $\hat u$ for various choices of $\xi$ (long time-scale)]{\includegraphics[width = 0.45\textwidth]{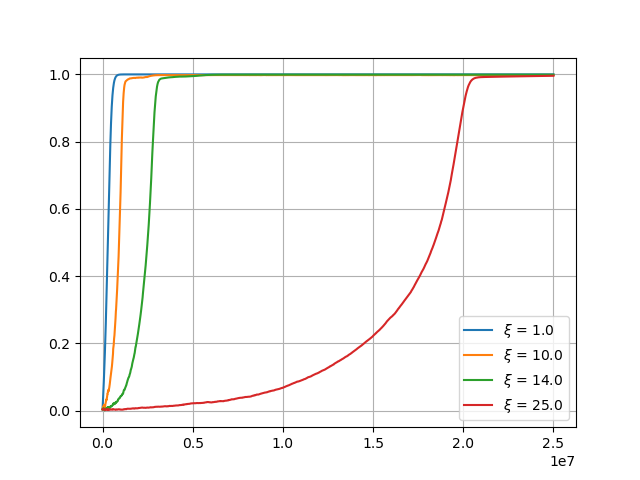}} 

\caption{Evolution of $\langle u_t, u\rangle$ during online SGD for  varying scales for $\xi$  (a) short timescale, (b) long timescale. Empirically, we see that while for ``small'' $\xi$ (e.g. $\xi = O(\sqrt{k})$) online SGD can quickly achieve strong recovery, as $\xi \to \infty$ we see that the time scales for both weak and strong recovery get larger. This illustrates how our low-rank fine-tuning setup allows to interpolate between different regimes (e.g. fine tuning to feature learning in single index models).} 
\label{fig:xi-scaling-feature-learning}
\end{figure}

For the purpose of illustrating the dimension dependence, we set $d=1000$ and keep the number of neurons smaller ($k=20$). While we do not formally analyze the effect of the magnitude of $\xi$, we expect fine-tuning to get more difficult when the correlations $\frac{k}{k+\xi^2}(\langle w_i + c_i u, w_j + \hat c_j \hat u\rangle)$ are small (e.g. $O(1/\sqrt{d})$). This happens around $\xi = \Theta(d^{1/4} k^{1/2})$, and we empirically see that the plateau in obtaining the initial weak recovery of $u$ gets larger as $\xi \to \infty$. An interesting future direction is obtaining rigorous results that quantify when this transition happens, and whether one can prove that fine-tuning works up to the feature learning threshold.

% \textbf{Deeper Networks}: TODO

\section{Outlook}

In this work we took the first steps towards understanding the gradient dynamics of low-rank fine-tuning beyond NTK. We identified a rich student-teacher framework, specialized to two-layer networks, and proved in various settings that online SGD efficiently finds the ground truth low-rank perturbation. This student-teacher framework is also appealing because it offers a natural way of interpolating between fine-tuning in the lazy training regime and generalized linear model regression in the feature learning regime. The parameter regime we consider occupies an intriguing middle ground between these extremes where the dynamics are nonlinear yet tractable and not overly sensitive to fine-grained properties like the Hermite coefficients of the activation function. 

Our results open up a number of future directions. Firstly, it is important to try to lift our assumptions, in particular the orthogonality of the perturbation relative to the pre-trained features, the assumption that $c$ is quantized to have equal-magnitude entries, and the assumption that $c$ is random.

For these questions, a fruitful starting point could be to target a specific, analytically tractable activation function like quadratic activation, especially given that based on our findings, the dynamics of low-rank fine-tuning do not depend heavily on particulars of $\sigma$. For this special case, we could hope to go beyond Hermite analysis and potentially even obtain an exact characterization of the dynamics.

Other important directions include analyzing the dynamics when $\hat{c}$ and $\hat{u}$ are jointly trained \--- Figure~\ref{fig:train-plots} suggests that this is roughly twice as efficient as freezing $\hat{c}$ and training $\hat{u}$ in isolation \--- as well as going beyond two layers and rank-1 perturbations. Finally, it would be interesting to understand the \emph{worst-case complexity} of fine-tuning: are there computational-statistical gaps in this setting?

\section*{Acknowledgments}

SC is grateful to Adam Klivans, and Vasilis Kontonis, and Raghu Meka for enlightening discussions about low-rank fine-tuning. AKD acknowledges support from the Summer Program for Undergraduates in Data Science, during which much of this work was completed.

% One point: the fact that the analysis is not too dependent on the hermite coefficients might motivate studying the quadratic case (exactly characterize rank-1 dynamics, higher-rank perturbations)

% Open directions: quadratic case with no assumptions, joint training of $c, u$, beyond quantization assumption for $c$, beyond orthogonality, beyond two layers

% Simulations
% \begin{enumerate}
%     \item Scaling in $k$ when first Hermite coefficent is 0
%     \item Scaling in $k$ in the online general case
%     \item Negative example, try the CSQ lower bound
%     \item Change the activation
%     \item Convince that the phenomena are robust
%     \item Things that fall out of our assumptions (e.g. global minima, deeper network)
%     \item 
% \end{enumerate}

\bibliographystyle{alpha}
\bibliography{bibliography}

\newcommand{\etalchar}[1]{$^{#1}$}
\begin{thebibliography}{HWAZ{\etalchar{+}}21}

\bibitem[AAM23]{abbe2023sgd}
Emmanuel Abbe, Enric~Boix Adsera, and Theodor Misiakiewicz.
\newblock Sgd learning on neural networks: leap complexity and saddle-to-saddle dynamics.
\newblock In {\em The Thirty Sixth Annual Conference on Learning Theory}, pages 2552--2623. PMLR, 2023.

\bibitem[ADK{\etalchar{+}}24]{arnaboldi2024repetita}
Luca Arnaboldi, Yatin Dandi, Florent Krzakala, Luca Pesce, and Ludovic Stephan.
\newblock Repetita iuvant: Data repetition allows sgd to learn high-dimensional multi-index functions.
\newblock {\em arXiv preprint arXiv:2405.15459}, 2024.

\bibitem[AGJ21]{arous2021online}
Gerard~Ben Arous, Reza Gheissari, and Aukosh Jagannath.
\newblock Online stochastic gradient descent on non-convex losses from high-dimensional inference.
\newblock {\em Journal of Machine Learning Research}, 22(106):1--51, 2021.

\bibitem[BBSS22]{bietti2022learning}
Alberto Bietti, Joan Bruna, Clayton Sanford, and Min~Jae Song.
\newblock Learning single-index models with shallow neural networks.
\newblock {\em Advances in Neural Information Processing Systems}, 35:9768--9783, 2022.

\bibitem[CDG{\etalchar{+}}23]{chen2023learning}
Sitan Chen, Zehao Dou, Surbhi Goel, Adam~R Klivans, and Raghu Meka.
\newblock Learning narrow one-hidden-layer relu networks.
\newblock {\em arXiv preprint arXiv:2304.10524}, 2023.

\bibitem[CKM22]{chen2022fpt}
Sitan Chen, Adam~R Klivans, and Raghu Meka.
\newblock Learning deep relu networks is fixed-parameter tractable.
\newblock In {\em 2021 IEEE 62nd Annual Symposium on Foundations of Computer Science (FOCS)}, pages 696--707. IEEE, 2022.

\bibitem[CM20]{chen2020learning}
Sitan Chen and Raghu Meka.
\newblock Learning polynomials in few relevant dimensions.
\newblock In {\em Conference on Learning Theory}, pages 1161--1227. PMLR, 2020.

\bibitem[CN24]{chen2024faster}
Sitan Chen and Shyam Narayanan.
\newblock A faster and simpler algorithm for learning shallow networks.
\newblock In {\em The Thirty Seventh Annual Conference on Learning Theory}, pages 981--994. PMLR, 2024.

\bibitem[CW01]{carbery2001distributional}
Anthony Carbery and James Wright.
\newblock Distributional and $\mathrm{L}^{q}$ norm inequalities for polynomials over convex bodies in $\mathbb{R}^n$.
\newblock {\em Mathematical research letters}, 8(3):233--248, 2001.

\bibitem[DH18]{dudeja2018learning}
Rishabh Dudeja and Daniel Hsu.
\newblock Learning single-index models in gaussian space.
\newblock In {\em Conference On Learning Theory}, pages 1887--1930, 2018.

\bibitem[DK24]{diakonikolas2024efficiently}
Ilias Diakonikolas and Daniel~M Kane.
\newblock Efficiently learning one-hidden-layer relu networks via schurpolynomials.
\newblock In {\em The Thirty Seventh Annual Conference on Learning Theory}, pages 1364--1378. PMLR, 2024.

\bibitem[DKKZ20]{diakonikolas2020algorithms}
Ilias Diakonikolas, Daniel~M Kane, Vasilis Kontonis, and Nikos Zarifis.
\newblock Algorithms and sq lower bounds for pac learning one-hidden-layer relu networks.
\newblock In {\em Conference on Learning Theory}, pages 1514--1539. PMLR, 2020.

\bibitem[DLS22]{damian2022neural}
Alexandru Damian, Jason Lee, and Mahdi Soltanolkotabi.
\newblock Neural networks can learn representations with gradient descent.
\newblock In {\em Conference on Learning Theory}, pages 5413--5452. PMLR, 2022.

\bibitem[DNGL24]{damian2024smoothing}
Alex Damian, Eshaan Nichani, Rong Ge, and Jason~D Lee.
\newblock Smoothing the landscape boosts the signal for sgd: Optimal sample complexity for learning single index models.
\newblock {\em Advances in Neural Information Processing Systems}, 36, 2024.

\bibitem[DPHZ24]{dettmers2024qlora}
Tim Dettmers, Artidoro Pagnoni, Ari Holtzman, and Luke Zettlemoyer.
\newblock Qlora: Efficient finetuning of quantized llms.
\newblock {\em Advances in Neural Information Processing Systems}, 36, 2024.

\bibitem[DPVLB24]{damian2024computational}
Alex Damian, Loucas Pillaud-Vivien, Jason~D Lee, and Joan Bruna.
\newblock The computational complexity of learning gaussian single-index models.
\newblock {\em arXiv preprint arXiv:2403.05529}, 2024.

\bibitem[DTA{\etalchar{+}}24]{dandi2024benefits}
Yatin Dandi, Emanuele Troiani, Luca Arnaboldi, Luca Pesce, Lenka Zdeborov{\'a}, and Florent Krzakala.
\newblock The benefits of reusing batches for gradient descent in two-layer networks: Breaking the curse of information and leap exponents.
\newblock {\em arXiv preprint arXiv:2402.03220}, 2024.

\bibitem[Fre75]{freedman1975tail}
David~A Freedman.
\newblock On tail probabilities for martingales.
\newblock {\em the Annals of Probability}, pages 100--118, 1975.

\bibitem[FYS{\etalchar{+}}23]{fu2023effectiveness}
Zihao Fu, Haoran Yang, Anthony Man-Cho So, Wai Lam, Lidong Bing, and Nigel Collier.
\newblock On the effectiveness of parameter-efficient fine-tuning.
\newblock In {\em Proceedings of the AAAI conference on artificial intelligence}, volume~37, pages 12799--12807, 2023.

\bibitem[GGJ{\etalchar{+}}20]{goel2020superpolynomial}
Surbhi Goel, Aravind Gollakota, Zhihan Jin, Sushrut Karmalkar, and Adam Klivans.
\newblock Superpolynomial lower bounds for learning one-layer neural networks using gradient descent.
\newblock In {\em International Conference on Machine Learning}, pages 3587--3596. PMLR, 2020.

\bibitem[HWAZ{\etalchar{+}}21]{hulora}
Edward~J Hu, Phillip Wallis, Zeyuan Allen-Zhu, Yuanzhi Li, Shean Wang, Lu~Wang, Weizhu Chen, et~al.
\newblock Lora: Low-rank adaptation of large language models.
\newblock In {\em International Conference on Learning Representations}, 2021.

\bibitem[JLR24]{jang2024lora}
Uijeong Jang, Jason~D Lee, and Ernest~K Ryu.
\newblock Lora training in the ntk regime has no spurious local minima.
\newblock {\em arXiv preprint arXiv:2402.11867}, 2024.

\bibitem[LMSR23]{lialin2023relora}
Vladislav Lialin, Sherin Muckatira, Namrata Shivagunde, and Anna Rumshisky.
\newblock Relora: High-rank training through low-rank updates.
\newblock In {\em The Twelfth International Conference on Learning Representations}, 2023.

\bibitem[LOSW24]{lee2024neural}
Jason~D Lee, Kazusato Oko, Taiji Suzuki, and Denny Wu.
\newblock Neural network learns low-dimensional polynomials with sgd near the information-theoretic limit.
\newblock {\em arXiv preprint arXiv:2406.01581}, 2024.

\bibitem[MOO05]{mossel2005noise}
Elchanan Mossel, Ryan O'Donnell, and Krzysztof Oleszkiewicz.
\newblock Noise stability of functions with low influences: invariance and optimality.
\newblock In {\em 46th Annual IEEE Symposium on Foundations of Computer Science (FOCS'05)}, pages 21--30. IEEE, 2005.

\bibitem[MWY{\etalchar{+}}23]{malladi2023kernel}
Sadhika Malladi, Alexander Wettig, Dingli Yu, Danqi Chen, and Sanjeev Arora.
\newblock A kernel-based view of language model fine-tuning.
\newblock In {\em International Conference on Machine Learning}, pages 23610--23641. PMLR, 2023.

\bibitem[O'D14]{o2014analysis}
Ryan O'Donnell.
\newblock {\em Analysis of boolean functions}.
\newblock Cambridge University Press, 2014.

\bibitem[Sz{\"o}09]{szorenyi2009characterizing}
Bal{\'a}zs Sz{\"o}r{\'e}nyi.
\newblock Characterizing statistical query learning: simplified notions and proofs.
\newblock In {\em International Conference on Algorithmic Learning Theory}, pages 186--200. Springer, 2009.

\bibitem[ZL24]{zengexpressive}
Yuchen Zeng and Kangwook Lee.
\newblock The expressive power of low-rank adaptation.
\newblock In {\em The Twelfth International Conference on Learning Representations}, 2024.

\end{thebibliography}

\begin{appendices}

% \documentclass{article}
% \usepackage{preamble}
% \usepackage{todonotes}

% \begin{document}
\section{Deferred proofs for analyzing the online SGD dynamics from \Cref{sec:technical-overview}}

\subsection{Computing the population gradient in a general setting}
\label{app:gradient_derivation}
Recall that for $\lambda \in \R^k, w_i \in \R^{d}$ with $\norm{w_i}=1$, $c \in \{\pm \frac{1}{\sqrt{k}}\}^k$, and $u \in \S^{d-1}$ we have the target model
\begin{equation}
    f^*(x) = \sum_{i=1}^k \lambda_i \sigma\left(\langle v_i, x\rangle\right)
\end{equation}
where $v_i = \frac{w_i + \xi c_i u}{\norm{w_i + \xi c_i u}}$. Furthermore, since $u\perp w_i$, we have $v_i = \frac{w_i + \xi c_i u}{\sqrt{1+\frac{\xi^2}{k}}}$. Initially, we derive the population loss and gradient without imposing additional assumptions. Because $\sigma$ admits a hermite expansion, for $v, \hat v \in \S^{d-1}$ we can evaluate expectations of the form $\E_x[\sigma(\langle v, x\rangle)\sigma(\langle \hat v,x\rangle)]= \sum_{p=0}^\infty \mu_p(\sigma)^2 \langle v, \hat v\rangle^p$. 

\begin{proof}[Proof of \Cref{prop:pop_gradient_pre}]
Note that $\E_x[(f^*(x)-\hat f(x))^2]= \sum_{i,j=1}^k \lambda_i \lambda_j f_i^*(x) f_j^*(x) + \sum_{i,j=1}^k\lambda_i \lambda_j  \hat f_i(x) \hat f_j(x)-2\sum_{i,j=1}^k \lambda_i \lambda_j f_i^*(x) \hat f_j(x)$. Then,
\begin{equation*}
    \langle f_i^*, \hat f_j \rangle = \sum_{p=0}^\infty \mu_p(\sigma)^2 \langle v_i, \hat v_j\rangle^p
\end{equation*}
Working similarly for $\langle f_i^*, f_j^*\rangle$ and $\langle \hat f_i, \hat f_j\rangle$, we have
\begin{align*}
    \E[(f^*(x)-\hat f(x))^2]&= \left(\sum_{i,j=1}^k \lambda_i \lambda_j \sum_{p=0}^\infty \mu_p(\sigma)^2 \langle \hat v_i, \hat v_j\rangle^p \right)+\left(\sum_{i,j=1}^k \lambda_i \lambda_j \sum_{p=0}^\infty \mu_p(\sigma)^2 \langle v_i, v_j\rangle^p \right) \\
    &-2 \sum_{i,j=1}^k \lambda_i \lambda_j \sum_{p=0}^\infty \mu_p(\sigma)^2 \langle v_i, \hat v_j\rangle^p
\end{align*}
Then, under the constraints $u, \hat u \perp w_i$ and $\norm{u}=\norm{\hat u}=1$, notice that $\langle v_i, v_j\rangle = \frac{\langle w_i, w_j\rangle + \xi^2 c_i c_j}{(1+\frac{\xi^2}{k})}$ and similarly $\langle \hat v_i, \hat v_j\rangle = \frac{\langle w_i, w_j\rangle + \xi^2 \hat c_i \hat c_j}{(1+\frac{\xi^2}{k})}$. Since we are restricting training and gradients to this constrained space, the gradients of the first two terms with respect to $\hat u$ vanish. Then,
\begin{align*}
    &\hat \nabla_{\hat u} \E[(f^*(x)-\hat f(x))^2]=-2 \sum_{i,j=1}^k\lambda_i \lambda_j \frac{\xi^2}{1+\frac{\xi^2}{k}} c_i \hat c_j \sum_{p=1}^\infty p\mu_p(\sigma)^2 \langle v_i, \hat v_j\rangle^{p-1} (u -\hat u \langle u,\hat u\rangle)\\
    &= -2 \sum_{i,j=1}^k\lambda_i \lambda_j \xi^2 c_i \hat c_j \sum_{p=1}^\infty p\mu_p(\sigma)^2 \left( \frac{1}{1+\frac{\xi^2}{k}}\right)^p (\langle w_i, w_j\rangle + \xi^2 c_i \hat c_j \langle u, \hat u\rangle)^{p-1} (u -\hat u \langle u,\hat u\rangle)
\end{align*}
Then, notice that since $\sum_{p=1}^\infty p \mu_p(\sigma)^2 < \infty$, the expression above converges absolutely (and uniformly) for any $|\langle u,\hat u\rangle |\leq 1$. Let $m= \langle u, \hat u\rangle$ and define. 
\begin{equation*}
    h(m) = 2\sum_{i,j=1}^k\lambda_i \lambda_j \xi^2 c_i \hat c_j \sum_{p=1}^\infty p\mu_p(\sigma)^2 \left( \frac{1}{1+\xi^2}\right)^p (\langle w_i, w_j\rangle + \xi^2 c_i \hat c_j m)^{p-1}
\end{equation*}
Because this expression converges absolutely and uniformly for $|m|\leq 1$, we can write its power series expansion around $m=0$, to get
\begin{equation*}
    h(m) = 2 \sum_{i,j=1}^k \lambda_i\lambda_j \sum_{l=0}^\infty (\xi^2)^{l+1} (c_i \hat c_j)^{l+1} m^l \sum_{s=0}^\infty (l+s+1) \mu_{l+s+1}(\sigma)^2 {l+s \choose l} \left(\frac{1}{1+\frac{\xi^2}{k}}\right)^{l+s+1} \langle w_i, w_j\rangle^s 
\end{equation*}
Then, notice that for odd $l$, we have $(c_i \hat c_j)^{l+1} = \frac{1}{k^{l+1}}$. For even $l$, we have $(c_i \hat c_j)^{l+1} = \frac{c_i \hat c_j}{k^l}$. Then, we can write
\begin{equation*}
    h(m) = 2 \sum_{l=0}^\infty \left(\frac{\xi^2}{k}\right)^{l+1} \left(\sum_{s=0}^\infty {l+s \choose l} (l+s+1) \mu_{l+s+1}(\sigma)^2 \left(\frac{1}{1+\frac{\xi^2}{k}}\right)^{l+s+1} T(l,s)\right) m^l
\end{equation*}
where
\begin{equation*}
    T(l,s)= \begin{cases}
        \norm{\sum_i \lambda_i w_i^{\otimes s}}_F^2 & l \text{ odd }\\
        k \left\langle \sum_i \lambda_i c_i w_i^{\otimes s}, \sum_i \lambda_i \hat c_i w_i^{\otimes s}\right\rangle& \text{ otherwise }
    \end{cases}
\end{equation*}
as claimed.
\end{proof}

\begin{remark}[Role of moment tensors]\label{remark:momenttensors}
The $T(l,s)$ terms in the expression for $h(m)$ involve moment tensors like $\sum_i \lambda_i w_i^{\otimes s}$ and $\sum_i \lambda_i c_i w^{\otimes s}_i$. As mentioned in Section~\ref{sec:related}, there exist networks for which these tensors vanish and for which noisy gradient descent takes a long time to learn them from scratch~\cite{diakonikolas2020algorithms,goel2020superpolynomial}. As such, their appearance in Proposition~\ref{prop:pop_gradient_pre} might seem to suggest that in the worst case over $c$ and $(\lambda_i, w_i)$, the complexity of fine-tuning could be as bad as the complexity of learning from scratch. While we do not formally address this worst case setting in this work, we expect that the complexity of the former should be dictated by the smallest $l$ for which the sum over $s$ in the definition of $h(m)$ is nonzero. Even if the moment tensors above vanish for many choices of $s$ so that $T(l,s) = 0$ unless $s$ is large, note that such $s$ will still contribute non-negligibly to the aforementioned sum. For this reason, we expect that the worst-case complexity landscape of fine-tuning should be very different (and in general far more benign) than that of learning from scratch.
\end{remark}

\subsection{Upper bounds on the variances of gradients and the magnitude of population gradient}\label{app:variances}

Note that the sample gradient is
\begin{equation*}
    \frac{k}{\xi^2 + k} \nabla_{\hat u} L(\hat u;x) = 2(f^*(x) - \hat f(x)) \sum_{i=1}^k \lambda_i \hat c_i \sigma'\left(\frac{k}{\xi^2 + k}\langle w_i + c_i u, x \rangle\right) x
\end{equation*}
So, to bound the moments of this quantity, we would like to bound the moments of $(f^*(x) - \hat f(x))$ and $\sum_{i=1}^k \lambda_i \hat c_i \sigma'\left(\frac{k}{\xi^2 + k}\langle w_i + c_i u, x \rangle\right) x$ respectively, and then apply Cauchy-Schwarz. To that end, we first prove the following:

\begin{proposition}[Moments of squared error]\label{proposition:moment-squared-error}
Let $p$ be given, and \Cref{assumption:activation} hold. Then, there exists some constant $C_{p, \sigma}$ that only depends on $p$ and $\sigma$ such that
\begin{equation*}
 \E_{x}[(f^*(x)-\hat f(x))^{2p}]^{1/p} \leq C_{p, \sigma} \lambda_{\max}^2 \min\left\{ k^2, 4 k\xi^2 \right\}
\end{equation*}    
\end{proposition}

\begin{proof}
Let $C_{p, \sigma}$ be a constant that only depends on $p$ and $\sigma$, that will change throughout the proof. Note that
\begin{align*}
    \E_x[(f^*(x)-\hat f(x))^{2p}]&\leq k^{2p-1} \sum_{i=1}^{k} \lambda_i^{2p} \E_x(\sigma(\langle v_i, x\rangle)-\sigma(\langle \hat v_i, x\rangle))^{2p}\\
&\leq C_{p, \sigma} \lambda_{\max}^{2p} k^{2p-1}  \sum_{i=1}^{k} \sqrt{\E_x[|\langle v_i, x\rangle - \langle \hat v_i, x\rangle|^{4p}]}\\
&\leq C_{p, \sigma} \lambda_{\max}^{2p} k^{2p} \norm{v_i - \hat v_i}^{2p} 
\end{align*}
Then, note that apriori, $\norm{v_i - \hat v_i} \leq 2$. Otherwise, 
\begin{align*}
    \norm{v_i - \hat v_i}&\leq \norm{\xi c_i u - \xi \hat c_i \hat u} + 2\left(1-\frac{1}{\sqrt{1+\xi^2 c_i^2}}\right)\\
&\leq \frac{2\xi}{\sqrt{k}} + \frac{2 \xi^2}{k}= \frac{2\xi}{\sqrt{k}} \left(1+\frac{\xi}{\sqrt{k}}\right)
\end{align*}
However, notice that if $\xi \leq \sqrt{k}$, this is bounded by $\frac{4\xi}{\sqrt{k}}$. Otherwise, we use the bound $\norm{v_i - \hat v_i} \leq 2$. Then,
\begin{equation*}
    \norm{v_i - \hat v_i} \leq \min\left\{2, \frac{4\xi}{\sqrt{k}}\right\}
\end{equation*}
Combining with the above and taking $p$'th root, we have
\begin{align*}
    \E_x[(f^*(x)-\hat f(x))^{2p}]^{1/p}&\leq C_{p, \sigma} \lambda_{\max}^2 k^2 \min\left\{4, \frac{16\xi^2}{k}\right\}\\
    &\leq C_{p, \sigma} \lambda_{\max}^2 \min\left\{ k^2, 4 k\xi^2 \right\}
\end{align*}
as desired.
\end{proof}

Now, we bound the other quantity of interest, which is the moments of squares of the gradient $\hat \nabla_{\hat u} \hat f(x)$. We have the following:
\begin{proposition}[Bound on the expected magnitude of $\hat \nabla \hat f$]\label{proposition:moment-gradient}
Let $p$ be given. Then, we have 
    \begin{equation*}
        \max\left\{\E_{x}\left|\frac{\hat \nabla_{\hat u} \hat f(x)}{\sqrt{d}}\right|^{2p}, \E_x\langle\hat \nabla_{\hat u} \hat f(x), u\rangle^{2p} \right\}^{1/p}\leq C_{\sigma, p} \lambda_{\max}^{2} \frac{k^2 \xi^2}{k+\xi^2} 
    \end{equation*}
\end{proposition}
\begin{proof}
    Let $C_{\sigma, p}$ be a constant whose value can change throughout the proof. Initially, note that
    \begin{align*}
        \hat \nabla_{\hat u} \hat f(x) &= (I - \hat u \hat u^\top) x \left[\sum_{i=1}^{k}\lambda_i \frac{\xi \hat c_i}{\sqrt{1+\xi^2 \hat c_i^2}} \sigma'(\langle v_i,x\rangle) \right]
    \end{align*}
    Then, since the spherical projection always leads to a smaller gradient
    \begin{equation*}
        \norm{\hat \nabla_{\hat u} \hat f(x)}^2\leq \norm{\nabla_{\hat u} \hat f(x)}^2
    \end{equation*}
    And furthermore,
    \begin{align*}
        \E_x \norm{\nabla_{\hat u} \hat f(x)}^{2p} &\leq \sqrt{ \E_x \norm{x}^{4p}} \sqrt{\E_x\left[\sum_{i=1}^{k}\lambda_i \frac{\xi \hat c_i}{\sqrt{1+\xi^2 \hat c_i^2}} \sigma'(\langle v_i,x\rangle)\right]^{4p}}\\
        &\leq C_{\sigma, p} d^{p} k^{2p} \lambda_{\max}^{2p}\frac{(\xi^2/k)^{p}}{(1+\xi^2/k)^{p}} \max_i \E_x \sqrt{\sigma'(\langle \hat v_i, x\rangle)^{4p}}
    \end{align*}
    However, since $\sigma'$ has at most polynomial growth, so does $(\sigma')^{4p}$ and since $\hat v_i$ is unit norm, the last quantity is finite and only depends on $\sigma$ and $p$. Then,
    \begin{align*}
        \left\{\E_x \norm{\nabla_{\hat u} \hat f(x)}^{2p}\right\}^{1/p}& \leq  C_{\sigma, p} \lambda_{\max}^2 d \frac{k^2 \xi^2}{k+\xi^2} 
    \end{align*}
    For the other case, note that the only step that changes is the bound on $\E_x \langle x, u\rangle^{4p}$ does not depend on the dimension, but only on $p$. So, we lose the dimension dependence.
\end{proof}

\begin{proposition}[Population gradient upper bound]
    \label{proposition:population-gradient-upper}
    We have
    \begin{equation*}
        \norm{\hat \nabla_{\hat u} \Phi(\hat u)}\leq C_\sigma \lambda_{\max}^2 \frac{k \xi^2}{1+\xi^2/k}\,.
    \end{equation*}
\end{proposition}
\begin{proof}
    Initially, note the non-expanded form of the population gradient:
    \begin{equation*}
        \hat \nabla \Phi = \frac{\xi^2}{1+\xi^2/k} \sum_{i,j=1}^{k} \lambda_i \lambda_j c_i \hat c_j \sum_{p=1}^{\infty}p \mu_p(\sigma)^2 \langle v_i, \hat v_j\rangle^{p-1} (u - \hat u \langle u, \hat u \rangle) 
    \end{equation*}
    Then, note $\left|\sum_{i,j=1}^k\lambda_i\lambda_j c_i\hat c_j \right|\leq k \lambda_{\max}^2$, and $\sum_{p=1}^\infty p \mu_p(\sigma)^2 \leq C_\sigma$. Furthermore, $\norm{u - \hat u \langle u,\hat u\rangle}\leq 1$ and $|\langle v_i, \hat v_j\rangle|\leq 1$. Then, $\norm{\hat \nabla \Phi}\leq C_\sigma\lambda_{\max}^2 \frac{k \xi^2}{1+\xi^2/k}$ as desired.
\end{proof}

\begin{proof}[Proof of \Cref{lemma:variance-bounds}] 
Noting that
\begin{equation*}
    \E_x  \norm{\frac{\hat \nabla_{\hat u} L(\hat u; x)}{\sqrt{d}}}_2^{2p} \leq \sqrt{\E_x (f^*(x) - \hat f(x))^{4p} \cdot \E_x \norm{\frac{\hat \nabla_{\hat u} \hat f(x)}{\sqrt{d}}}_2^{4p}}
\end{equation*} 
and similarly for $\langle \hat \nabla_{\hat u} L(\hat u;x), u\rangle$ the result immediately follows from \Cref{proposition:moment-squared-error}, \Cref{proposition:moment-gradient} and \Cref{proposition:population-gradient-upper}. 
\end{proof}

\subsection{Anti-concentration tools for lower bounding the population gradient}\label{app:anti-concentration}
\subsubsection{Anti-concentration for quadratic polynomials with low influences}

In this section, we prove some results related to the anti-concentration of certain quadratic functions on the hypercube. In particular, consider functions $f: \{\pm1 \}^k \times \{\pm1\}^k \to \R$ of the form $f(x,y) = x^\top Q y$. These functions capture the random behavior of the function $h$ (due to the randomness in $c,\hat c$) by determining the magnitudes of the constant term. We will control the magnitudes of functions of boolean variables by relating them to functions of Gaussians, and then applying anti-concentration inequalities for Gaussian polynomials. To that end, we first state some known bounds from literature. Note that these functions fall into the family of multilinear polynomials, which are defined as follows:

\begin{definition}[Multilinear polynomial]
We define a normalized degree $d$ multilinear polynomial as
\begin{align*}
    Q(x_1, x_2,\dots, x_n) = \sum_{S\subset [n], |S|\leq d} a_S \prod_{i \in S} x_i
\end{align*}
with $\mathrm{Var}(Q) = \sum_{S\subset [n], |S| > 0} a_S^2 = 1$.
\label{def:multilinear-polynomial}
\end{definition}
\noindent Initially, we would have the following anti-concentration result if our randomness was \textit{Gaussian} instead of \textit{rademacher}.
\begin{lemma}[Carbery-Wright inequality~\cite{carbery2001distributional}]\label{lemma:carbery-wright}
    Let $Q$ be a normalized multilinear polynomial with degree $d$ as in \Cref{def:multilinear-polynomial}. There exists an absolute constant $B$ such that for $g\sim \mN(0, I_n)$ we have 
    \begin{align*}
        \Pr[|Q(g_1, g_2,\dots, g_n)|\leq \epsilon]\leq B\epsilon^{1/d}
    \end{align*}
\end{lemma}

\noindent For this class of functions, we know the following bound that helps us replace the randomness from Rademacher variables with Gaussian randomness.
\begin{lemma}[Invariance principle, \protect{\cite[Theorem 2.1]{mossel2005noise}}] 
Let $P$ be as in \Cref{def:multilinear-polynomial}. Furthermore, define the maximum influence as $\tau = \max_{i\in [n]} \sum_{S \ni i} a_S^2$. Then, for $\xi \sim \mathrm{Unif}\left\{\pm 1\right\}^n$ and $g \sim \mN(0,I_n)$, we have
\begin{align*}
    \sup_t |\Pr[P(\xi_1, \dots, \xi_n) \leq t] - \Pr[P(g_1, \dots, g_n) \leq t]| \leq O(d \tau^{1/8d})
\end{align*}
\label{lemma:invariance-principle}
\end{lemma}
\noindent To be able to leverage these results, we need to quantify the influence of functions $x^\top Qy$ with $Q$ being p.s.d. Normalizing $\norm{Q}_F^2 = 1$ and noting $x^\top Qy = \sum_{i,j=1}^k x_i y_j Q_{ij}$, note that the influence of $x_i$ (or similarly $y_i$) is $\sum_{j} Q_{ij}^2$. Hence, bounding the influence of this family of quadratic functions is the same as bounding the ratio of a row norm of $Q$ to the frobenius norm. Hence, we prove the following claim that does exactly this.
% new version
\begin{claim}[Influence of row of PSD matrix]\label{claim:influence-psd-matrix}
Consider all $k\times k$ PSD matrices $Q$ with diagonal entries equal to $1$. We have the following uniform bound on the maximum norm of a row to the frobenius norm of $Q$:
\begin{equation*}
\sup_{\substack{Q \in \R^{k\times k},\\ Q \text{ psd },\\ Q_{ii}=1}}\frac{\max_{i} \sum_{j\in [k]} Q_{ij}^2}{\sum_{i,j=1}^k Q_{ij}^2}\leq 2k^{-1/2}
\end{equation*}
\end{claim}

\begin{proof}
To bound the given expression, we can write the numerator in an alternative way.  Let $N = Q^2$ and note that $N_{jj} = \sum_l Q_{jl}Q_{jl}= \sum_{l} Q_{jl}^2$ which is the norm of the $j$'th row of $Q$. Hence, we are interested in bounding $\max_j N_{jj}/\norm{Q}_F^2$. This is equivalently bounding the ratio of a diagonal entry of $N$ to the sum of its eigenvalues.

Concretely, let $v_1, \dots, v_k \in \R^k$ be the eigenvectors of $Q$ and $s_1, \dots, s_k$ be the corresponding eigenvalues. Note that $N$ has the same eigenvectors as $Q$ and the eigenvalues $s_1^2, \dots, s_k^2$. Let $\mS = \{i \in [k]: s_i^2 \geq k\}$ and consider $\zeta = \sum_{i \in \mS} s_i^2$, which is the sum of 'large' eigenvalues of $N$. Apriori, note that $\norm{Q}_F^2 \geq \max\{k , \zeta\}$ since the diagonals of $Q$ are $1$ and $\zeta$ is less than the sum of the squares of eigenvalues of $Q$.

Now, note $Q = \sum_i s_i v_iv_i^\top$ and $N = \sum_i s_i^2 v_iv_i^\top$. Hence
\begin{equation*}
    N_{jj} = \sum_{i \in \mS} s_i^2 (v_i)_j^2 + \sum_{i \not \in \mS} s_i^2 (v_i)_j^2 
\end{equation*}

To bound the first term in the expression, note that $s_i^2 \geq k$, but $Q_{jj} = \sum_j s_i (v_i)_j^2 = 1$ so that $(v_i)_j^2 \leq k^{-1/2}$. Then, $\sum_{i \in \mS} s_i^2 (v_i)_j^2\leq \sum_{i\in \mS} s_i^2 \max_j (v_i)_j^2 \leq k^{-1/2} \zeta \leq k^{-1/2} \norm{Q}_F^2$. Similarly, for the second term, note simply that $s_i \leq k^{1/2}$ and $\sum_{i\not \in \mS} s_i (v_i)_j^2 \leq 1$. Hence, $N_{jj} \leq k^{-1/2}\norm{Q}_F^2 + k^{1/2}$. Noting $\norm{Q}_F^2 \geq k$, we have $k^{1/2} \leq \norm{Q}_F^2 k^{-1/2}$. Hence, $N_{jj}/\norm{Q}_F^2 \leq 2k^{-1/2}$. Since none of these bounds depend on the choice of $Q$, we have proven the claim.
\end{proof}

\begin{corollary}

For a given $Q$, let $0<q_{\min}^2\leq q_{\max}^2$ be absolute constants such that for all $k$, we have $q_{\min}^2 \leq Q_{ii}\leq q_{\max}^2$. Then, we have
\begin{equation*}
\sup_{\substack{Q \in \R^{k\times k},\\ Q \text{ psd },\\ q_{\min}^2\leq Q_{ii}\leq q_{\max}^2}}\frac{\max_{i} \sum_{j\in [k]} Q_{ij}^2}{\sum_{i,j=1}^k Q_{ij}^2}\leq 2 \cdot \frac{q_{\max}^2}{q_{\min}^2} k^{-1/2}
\end{equation*}

\label{corollary:psd-influence-balanced}
\end{corollary}
\begin{proof}
    The proof follows immediately by modifying the proof of \Cref{claim:influence-psd-matrix} slightly. Note $\sum_{i} s_i (v_i)_j^2 \leq q_{\max}^2$ and $\norm{Q}_F^2 \geq \max\{\zeta, k q_{\min}^2\}$. Then, following the proof similarly, we get
    \begin{equation*}
        N_{jj}\leq k^{-1/2} \zeta + q_{\max}^2 k^{1/2} \leq k^{-1/2} \norm{Q}_F^2 + q_{\max}^2/q_{\min}^2 \norm{Q}_F^2 k^{-1/2}
    \end{equation*}
    Then, noting $1\leq q_{\max}^2/q_{\min}^2$ and dividing by $\norm{Q}_F^2$, we have that the desired quantity is bounded by $2\frac{q_{\max}^2}{q_{\min}^2} k^{-1/2}$ as desired.
\end{proof}
\noindent Now, we will use the above results to prove the following, which will help us directly quantify the random behavior of the function $h$.

\begin{lemma}[Anti-Concentration of Normalized P.S.D. Quadratics on the Hypercube]
    \label{lemma:psd-anti-conc}
    Let $Q \in \R^{k\times k}$ be positive semi-definite and normalized such that $q_{\min}^2 \leq Q_{ii} \leq q_{\max}^2$ for some $q_{\min}, q_{\max} >0$. Then,
    \begin{equation*}
        \sup_{\substack{Q \in \R^{k\times k},\\ Q \text{ psd },\\ q_{\min}^2\leq Q_{ii}\leq q_{\max}^2}} \Pr_{x, y \sim \mathrm{Unif}\{\pm 1\}^{k}}[|x^\top Q y| \leq \epsilon\norm{Q}_F]\leq o(1) + O(\epsilon^{1/2})
    \end{equation*}
    where the $o(1)$ is in $k$.
\end{lemma}
\begin{proof}
    First, note that we have the uniform bound on the influence of a row of $Q$ from \Cref{corollary:psd-influence-balanced}, so that $\tau = O(k^{-1/2})$. Hence, by the invariance principle (\Cref{lemma:invariance-principle}), for any $Q$, we have 
    \begin{equation*}
        \sup_{t} |\Pr_{x,y \sim \mathrm{Unif}\{\pm\}^k}[x^\top Q y \leq t] - \Pr_{g_1, g_2 \sim \mN(0, I_k)}[g_1^\top Q g_2 \leq t]| \leq o(1)
    \end{equation*}
    However, applying Carbery-Wright inequality for the anti-concentration of the Gaussian polynomial $g_1^\top Q g_2$ (\Cref{lemma:carbery-wright}), we get the desired result.
\end{proof}

\subsubsection{Concentration and anti-concentration of $h(0)$}
In this section, we analyze the behavior of $h(0)$. In particular, we aim to quantify the magnitude of $h(0)$ with high probability. Recall that when $\xi = 1$, we have
\begin{equation*}
    h(0) = \sum_{i,j=1}^k \lambda_i \lambda_j c_i \hat c_j \sum_{s=0}^\infty (s+1) \mu_{s+1}(\sigma)^2 \left(\frac{k}{k+1}\right)^{s+1}\langle w_i, w_j\rangle^s 
\end{equation*}
Then, notice that writing $c_i = \frac{1}{\sqrt{k}}b_i$ and similarly $\hat c_i = \frac{1}{\sqrt{k}} \hat b_i$, this is a quadratic function of rademacher variables. In particular, let
\begin{equation}
    f(b,\hat b) = \sum_{i,j}^k b_i \hat b_j \left(\frac{\lambda_i \lambda_j}{k}\sum_{s=0}^\infty (s+1)\mu_{s+1}(\sigma)^2 \left(\frac{k}{k+1}\right)^{s+1}\langle w_i, w_j\rangle^s\right) \label{eq:h0_form}
\end{equation}
so that $f(b,\hat b) = b^\top Q \hat b$ for some $Q$. Hence, we can use the anti-concentration results from the previous section to quantify the random behavior of $h(0)$ in the $\xi=1$ case.
\begin{claim}[Variance of $h(0)$, spectral scaling $\xi=1$] 
Let $f : \{\pm 1\}^{k} \times \{\pm 1\}^k \to \R$ be as in \cref{eq:h0_form}. Then, we have $\Omega(\lambda_{\min}^2/k)\leq \norm{f}_2\leq O(\lambda_{\max}^2)$
\label{claim:constant-term-variance}
\end{claim}
\begin{proof}
    Notice that since each term in the sum is a different basis element of $\{\pm 1\}^{2k}$, we have 
    \begin{align*}
        \norm{f}_2^2 = \sum_{i,j=1}^k  Q_{ij}^2
    \end{align*}
    For the first part of the Claim, it suffices to show $\sum Q_{ij}^2 = \Omega(\frac{\lambda_{\min}^4}{k})$ for any choice of $\lambda, w_i$. Notice that, for $k\geq 2$,
    \begin{align*}
        \sum_{i,j=1}^k Q_{ij}^2 \geq \sum_{i=1}^k Q_{ii}^2 &= \left(\sum_{s=0}^\infty (s+1) \mu_{s+1}(\sigma)^2 \left(\frac{k}{k+1}\right)^{s+1}\right)^2\sum_{i=1}^k \frac{\lambda_i^4}{k^2}\\
        &\geq \left(\sum_{s=0}^\infty \frac{s+1}{2^s} \mu_{s+1}(\sigma)^2 \right)^{2} \frac{\lambda_{\min}^4}{k}
    \end{align*}
    as desired. 
    % Then, $\norm{f}_2 = \Omega(1/\sqrt{k})$. Since the result above is uniform over all choices of $\lambda$ and $w_i$, we get the first part of the Claim. 
    The other follows directly from 
    \begin{align*}
        \sum_{i,j=1}^k Q_{ij}^2 &\leq \sum_{i,j=1}^k \frac{1}{k^2} \lambda_i^2 \lambda_j^2 \left(\sum_{s=0}^\infty (s+1) \mu_{s+1}(\sigma)^2\right)^2\leq \lambda_{\max}^4\left(\sum_{s=0}^\infty (s+1) \mu_{s+1}(\sigma)^2\right)^2\,.\qedhere
    \end{align*}
\end{proof}

\noindent Hence, if we can show that $f$ anti-concentrates around $0$ with high probability, then we will have given a lower bound for $h(0)$ in the $\xi=1$ scaling regime. This is what we do now.

\begin{proposition}[Anti-concentration of $h(0)$, $\xi = 1$]
If $\xi = 1$, we have the following:
Let $f$ be of the form in \Cref{eq:h0_form}. Then, 
    \begin{align*}
        \sup_{w_i, \lambda_i}\Pr_{b,\hat b}[|f(b, \hat b)| < \epsilon \norm{f}_2]=\frac{\lambda_{\max}^2}{\lambda_{\min}^2}\cdot o\left(1\right) + O(\epsilon^{1/2})
    \end{align*}
   where $b,\hat{b}$ are independent uniform draws from $\{-1,1\}^k$. \label{proposition:spectral-constant-anti-concentration}
\end{proposition}
\begin{proof}
 Note that entrywise powers of psd matrices are psd, so $(W^TW)^{\odot s}$ is psd.
Notice that 
\begin{equation*}
    Q_{ij}=(\lambda_{i}\lambda_j)  \left(\sum_{s=0}^\infty (s+1) \mu_{s+1}(\sigma)^2 \left(\frac{k}{k+1}\right)^{s+1}\langle w_i, w_j\rangle^s\right)\,,    
\end{equation*}
so $Q$ is a psd matrix since it is the sum of psd matrices (for $s$). This is due to the fact 
\begin{align*}
    Q = \lambda \lambda^\top * \tilde Q
\end{align*}
where $\tilde Q_{ij} = \sum_{s=0}^\infty (s+1) \mu_{s+1}(\sigma)^2 \left(\frac{k}{k+1}\right)^{s+1}\langle w_i, w_j\rangle^s$ since it is the non-negative sum of psd matrices. Furthermore, $(q_{\max}/q_{\min})^2 = O(\frac{\lambda_{\max}^2}{\lambda_{\min}^2})$.
    The final result follows immediately once we normalize as $\frac{f}{\norm{f}_2}$ and apply \Cref{lemma:psd-anti-conc}.
\end{proof}  

\begin{proposition}[Anti-concentration of $(\sum_i \lambda_i c_i)(\sum_i \lambda_i \hat c_i)$ and $\sum_i \lambda_i^2 c_i \hat c_i$]
\label{proposition:anti-concentration-frobenius}
We have
\begin{equation*}
    \Pr\left[\left|(\sum_i \lambda_i c_i)(\sum_i \lambda_i \hat c_i)\right| \leq \gamma \lambda_{\min}^2 \right]\leq o\left(\frac{\lambda_{\max}^2}{\lambda_{\min}^2}\right) + O(\gamma^{1/2})
\end{equation*}
and 
\begin{equation*}
    \Pr\left[\left|\sum_{i}\lambda_i^2 c_i \hat c_i\right| \leq \gamma \frac{\lambda_{\min}^2}{\sqrt{k}}\right]\leq o\left(\frac{\lambda_{\max}^2}{\lambda_{\min}^2}\right) + O(\gamma^{1/2})
\end{equation*}
\end{proposition}
\begin{proof}
    For the first one let $Q=\frac{1}{k} \lambda \lambda^\top$. and for the second one let $Q= \frac{1}{k}I (\lambda \odot \lambda)$. Both are balanced psd matrices, and the anti concentration result \Cref{lemma:psd-anti-conc} holds. Then, the results follow. 
\end{proof}

\subsection{Orthonormal case: population gradient lower bounds}
\label{app:orthonormal-lower-bound}

% \begin{proposition}
% Suppose $\mu_{1}(\sigma) \neq 0$. Then, with probability $1-o(1)$, we have
% \begin{align*}
%     h(m) &\geq C_1(1+O(1/\sqrt{k}))  + \sum_{l\text{ odd }} b_l m^l
% \end{align*}
% where $b_l\geq 0$.
% \end{proposition}
% \begin{proof}
%     We will show that if $\mu_1(\sigma) \neq 0$, the magnitude of the constant term is determined by $(\sum_i \lambda_i c_i)(\sum_i \lambda_i \hat c_i)$. Then, note that with high probability this term is $O(\sqrt{k})$ larger than $\sum_i \lambda_i^2 c_i \hat c_i$, so we will show all the other terms can be ignored, and the constant term is dictated by $(\sum_i \lambda_i c_i)(\sum_i \lambda_i \hat c_i)$
% \end{proof}

% Now, we prove a bound when we potentially have $\mu_{1}(\sigma) = 0$. 
% \begin{proposition}
%     !!TODO!!
% \end{proposition}
% \begin{proof}
    
% \end{proof}

Recall the function $h$. 
\begin{equation*}
    h(m) = 2 \sum_{l=0}^\infty \left(\frac{\xi^2}{k}\right)^{l+1} \left(\sum_{s=0}^\infty {l+s \choose l} (l+s+1) \mu_{l+s+1}(\sigma)^2 \left(\frac{1}{1+\frac{\xi^2}{k}}\right)^{l+s+1} T(l,s)\right) m^l
\end{equation*}
with $T(l,s)$ being defined as 
\begin{equation*}
    T(l,s)\triangleq \begin{cases}
        \norm{\sum_i \lambda_i w_i^{\otimes s}}_F^2 & l \text{ odd }\\
        k \left\langle \sum_i \lambda_i c_i w_i^{\otimes s}, \sum_i \lambda_i \hat c_i w_i^{\otimes s}\right\rangle& \text{ otherwise }
    \end{cases}
\end{equation*}

However, in the orthogonal case, for $s\geq 1$, $T(l,s)$ reduces to 
\begin{equation*}
    T(l, s\geq 1) = \begin{cases}
        \sum_{i=1}^{k}\lambda_i^2 & l \text{ odd }\\
        k\sum_{i=1}^k \lambda_i^2 c_i \hat c_i & \text{ otherwise }
        \end{cases}
\end{equation*}
And for $s = 0$, these reduce to 
\begin{equation*}
    T(l, 0) = \begin{cases}
\left(\sum_{i=1}^{k}\lambda_i\right)^2 & l \text{ odd }\\
        k\left(\sum_{i=1}^k \lambda_ic_i\right)\left(\sum_{i=1}^k \lambda_i \hat c_i\right) & \text{ otherwise }
        \end{cases}
\end{equation*}
Notice that for all odd $l$, the power series coefficients are always non-negative. And for all even $l$, all the power series coefficients have the same sign. Now, we prove \Cref{claim:even-s-0} as stated in \Cref{sec:technical-overview} which bounds the maximum contribution coming from even $l$, $s=0$ terms.

% We initially bound the maximum possible contribution coming from the even $l$ terms with $s=0$.
% \begin{claim}[Even $l$, $s=0$ contribution] 
% \label{claim:even-s-0}
%  With probability $1-\exp\{-\frac{2k}{e\xi^2}\}$ over the randomness of $c, \hat c$, the following holds.
%     \begin{align*}
%     \mathrm{sign}(m) &\left(2\sum_{\substack{l \text{ odd }\\ s=1}} \left(\frac{\xi^2}{k}\right)^{l+1} {l+s \choose l}(l+s+1)\mu_{l+s+1}(\sigma)^2 \left(\frac{1}{1+\xi^2/k}\right)^{l+s+1} T(l,s)\right. \\
%    &\left. +2\sum_{\substack{l>0\\ \text{ even }\\ s=0}} \left(\frac{\xi^2}{k}\right)^{l+1} {l+s \choose l}(l+s+1)\mu_{l+s+1}(\sigma)^2 \left(\frac{1}{1+\xi^2/k}\right)^{l+s+1} T(l,s)\right)\geq 0
% \end{align*}
% \end{claim}
\begin{proof}[Proof of \Cref{claim:even-s-0}]
    Note first that $\E_{c,\hat c} \left(\sum_{i=1}^k\lambda_i c_i \right)\left(\sum_{i=1}^k\lambda_i \hat c_i \right)=0$ and moreover,
\begin{equation*}
    \E_{c,\hat c} \left(\sum_{i=1}^k\lambda_i c_i \right)^2\left(\sum_{i=1}^k\lambda_i \hat c_i \right)^2= \frac{\norm{\lambda}_2^4}{k^2}
\end{equation*}
so the standard deviation is $\norm{\lambda}_2^2/k$. Hence, $T(l,0)$ has standard deviation $\norm{\lambda}_2^2$ in $c,\hat c$. Then, note that 
\begin{align*}
    &2\sum_{\substack{l >0\\ \text{ even }}} \left(\frac{\xi^2}{k}\right)^{l+1}(l+1) \mu_{l+1}(\sigma)^2 \left(\frac{1}{1+\frac{\xi^2}{k}}\right)^{l+1} \left(k\sum_{i,j=1}^k \lambda_i\lambda_j c_i \hat c_j\right)m^l\\
    &=\frac{2m\xi^2}{k}\left(k\sum_{i,j=1}^k \lambda_i\lambda_j c_i \hat c_j\right) \sum_{\substack{l >0\\ \text{ even }}} \left(\frac{\xi^2}{k}\right)^{l}(l+1) \mu_{l+1}(\sigma)^2 \left(\frac{1}{1+\frac{\xi^2}{k}}\right)^{l+1} m^{l-1}\\
    &= \frac{2m\xi^2}{k}\left(k\sum_{i,j=1}^k \lambda_i\lambda_j c_i \hat c_j\right) \sum_{l \text{ odd }} \left(\frac{\xi^2}{k}\right)^{l+1}(l+2) \mu_{l+2}(\sigma)^2 \left(\frac{1}{1+\frac{\xi^2}{k}}\right)^{l+2} m^{l}\\
    &= \frac{2m\xi^2}{k}\left(k\sum_{i,j=1}^k \lambda_i\lambda_j c_i \hat c_j\right) \sum_{l \text{ odd }} \frac{1}{l+1} {l+1 \choose l} \left(\frac{\xi^2}{k}\right)^{l+1}(l+2) \mu_{l+2}(\sigma)^2 \left(\frac{1}{1+\frac{\xi^2}{k}}\right)^{l+2} m^{l}\\
    &=\frac{2m\xi^2}{k}\left(k\sum_{i,j=1}^k \lambda_i\lambda_j c_i \hat c_j\right) \sum_{l \text{ odd }, s=1} \frac{1}{l+s} {l+s \choose l} \left(\frac{\xi^2}{k}\right)^{l+1}(l+s+1) \mu_{l+s+1}(\sigma)^2 \left(\frac{1}{1+\frac{\xi^2}{k}}\right)^{l+s+1} m^{l}
\end{align*}
However, notice that the sum precisely corresponds to all odd $l$ with $s=1$. Then, bounding $l\geq 1$ so that $\frac{1}{l+1}\leq \frac{1}{2}$, we can elementwise compare the odd $l$ terms with $s=1$ and even $l$ terms with $s=0$. The odd terms are
\begin{equation*}
    2\norm{\lambda}_2^2 \sum_{l\text{ odd }} \left(\frac{\xi^2}{k}\right)^{l+1} {l+1\choose l} (l+2) \mu_{l+2}(\sigma)^2 \left(\frac{1}{1+\frac{\xi^2}{k}}\right)^{l+2} m^l
\end{equation*}
where all the terms in the sum have the same sign as $m$. 
Then, note that it suffices to show that, with high probability, we have
\begin{align*}
    \frac{m\xi^2}{k}\left(k\sum_{i,j=1}^k \lambda_i\lambda_j c_i \hat c_j\right) \leq 2\norm{\lambda}_2^2
\end{align*}
Then, note that using the standard deviation bound and  \cite[Theorem 9.23]{o2014analysis}, we have
\begin{equation*}
    \Pr\left[\left|\frac{m\xi^2}{k}\left(k\sum_{i,j=1}^k \lambda_i\lambda_j c_i \hat c_j\right) \right|\geq 2\norm{\lambda}_2^2\right]\leq \exp\left\{-\frac{2k}{em\xi^2}\right\}\leq \exp\left\{-\frac{2k}{e\xi^2}\right\}
\end{equation*}
Hence, with probability $1-\exp\left\{-\frac{2k}{e\xi^2}\right\}$, the even $s=0$ terms will not effect the sign of the odd terms. In particular, we have, with probability at least $1-\exp\{-\frac{2k}{e\xi^2}\}$, we have
\begin{align*}
    \mathrm{sign}(m) &\left(2\sum_{l \text{ odd }, s=1} \left(\frac{\xi^2}{k}\right)^{l+1} {l+s \choose l}(l+s+1)\mu_{l+s+1}(\sigma)^2 \left(\frac{1}{1+\xi^2/k}\right)^{l+s+1} T(l,s)m^l\right. \\
   &\left. +2\sum_{l \text{ even }, s=0} \left(\frac{\xi^2}{k}\right)^{l+1} {l+s \choose l}(l+s+1)\mu_{l+s+1}(\sigma)^2 \left(\frac{1}{1+\xi^2/k}\right)^{l+s+1} T(l,s)m^l\right)\geq 0
\end{align*}
as desired.
\end{proof}

% \begin{proposition}
% Suppose $\mu_1(\sigma)\neq 0$, then with probability $1-\exp\{-\frac{2k}{e\xi^2}\}-o(1)-O(\frac{\lambda_{\max}}{\lambda_{\min}}\gamma^{1/2})$, we have $h(\mathrm{sign}(h(0)) m) \mathrm{sign}(h(0)) \geq \frac{|h(0)|}{2} \geq \frac{\gamma k\xi^2 \mu_1(\sigma)^2}{k + \xi^2}$ for $m\geq 0$.
% \end{proposition}
\noindent Now, we prove the lower bound on the population gradient in the orthonormal case, as stated in \Cref{lemma:pop-gradient-bound-orthonormal}, \Cref{sec:technical-overview}.

\begin{proof}[Proof of \Cref{lemma:pop-gradient-bound-orthonormal}]
Initially, suppose $\mu_{1}(\sigma)\neq 0 $. 
In this case, using \Cref{claim:even-s-0}, with probability $1-\exp\{-\frac{2k}{e\xi^2}\}$ we have 
\begin{align*}
    \mathrm{sign}(m)h(m) &\geq \mathrm{sign}(m) \xi^2 \left(\sum_{i=1}^k \lambda_i c_i\right)\left(\sum_{i=1}^k \lambda_i \hat c_i\right)\mu_1(\sigma)^2 \frac{1}{1+\frac{\xi^2}{k}} + \sum_{l\text{ odd }} b_l |m|^l\\
    &\qquad\qquad+\mathrm{sign}(m)\frac{\xi^2}{1+\xi^2/k} \langle c_\lambda, \hat c_\lambda \rangle \sum_{l \text{ even }, s\geq 1} \left(\frac{\xi^{2}}{k}\right)^{l} {l+s \choose l} (l+s+1) \mu_{l+s+1}(\sigma)^2 
    \left(\frac{1}{1+\frac{\xi^2}{k}}\right)^{l+s}|m|^{l}
\end{align*}
Now, we investigate the second term. Note that the sum in the second term is bounded by
\begin{align*}
    &\sum_{l, s\geq 0} \left(\frac{\xi^2}{k}\right)^{l}{l+s \choose l} (l+s+1) \mu_{l+s+1}(\sigma)^2 \left(\frac{1}{1+\frac{\xi^2}{k}}\right)^{l+s}\\
    &=\sum_{p=0}^\infty \sum_{s=0}^{p}\left(\frac{\xi^2}{k}\right)^{p-s}{p \choose s}(p+1)\mu_{p+1}(\sigma)^2 \left(\frac{1}{1+\frac{\xi^2}{k}}\right)^p\\
    &=\sum_{p=0}^\infty (p+1)\mu_{p+1}(\sigma)^2 \left(\frac{k}{k+\xi^2}\right)^{p}\left(1+\frac{\xi^2}{k}\right)^{p}\\
    &\leq \sum_{p=0}^\infty (p+1)\mu_{p+1}(\sigma)^2 \leq C_\sigma
\end{align*}
Then, notice the the second term is bounded in magnitude by $\frac{C_\sigma \xi^2}{1+\xi^2/k}|\langle c_\lambda, \hat c_\lambda\rangle|$. Then, notice that
\begin{align*}
    \Pr\left[|\langle c_\lambda, \hat c_\lambda\rangle|\geq \frac{\gamma \lambda_{\max}^2}{\sqrt{k}} \log k\right]\leq k^{-\frac{\gamma}{e}}
\end{align*}
Set $\gamma = 10$. So, with high probability this term is $O\left(\frac{\log k}{\sqrt{k}}\frac{C_\sigma \lambda_{\max}^2 \xi^2 }{1+\xi^2/k}\right)$
However, by anti-concentration of the constant term (\Cref{proposition:anti-concentration-frobenius}), we have that the constant term is order $\frac{\gamma \lambda_{\max}^2 \mu_1(\sigma)^2 \xi^2}{1+\xi^2/k}$ with probability $1-o(1) - O(\frac{\lambda_{\max}}{\lambda_{\min}}\gamma^{1/2})$. Then, the constant term is $O(\sqrt{k}(\log k)^{-1})$ larger than the even terms, and it's sign is dictated by $\mathrm{sign}(m)(\sum_i \lambda_ic_i)(\sum_i \lambda_i \hat c_i) > 0$. Then, we can bound the even terms by half of the constant term, and get the desired result. 

Now, suppose $\mu_1(\sigma) = 0$. In this case, with probability $1-\exp\{-\frac{2k}{e\xi^2}\}$ note that 
\begin{equation*}
    \mathrm{sign}(m)h(m) \geq \mathrm{sign}(m)\xi^2 \langle c_\lambda, \hat c_\lambda \rangle \sum_{\substack{l \text{ even }\\
    s\geq 1}} \left(\frac{\xi^{2}}{k}\right)^{l} {l+s \choose l} (l+s+1) \mu_{l+s+1}(\sigma)^2 
    \left(\frac{1}{1+\frac{\xi^2}{k}}\right)^{l+s+1}|m|^{l} + \sum_{l\text{ odd }} b_l |m|^l
\end{equation*}
where the $b_l$ are non-negative coefficients. Then, under the sign assumption on $m$, note that $\mathrm{sign}(m)\xi^2 \langle c_\lambda, \hat c_\lambda \rangle= |\langle c_\lambda, \hat c_\lambda\rangle|\xi^2$. Then, by anti-concentration (\Cref{proposition:anti-concentration-frobenius}), note that with probability $1-o(1)-O(\gamma^{1/2})$, $|\langle c_\lambda, \hat c_\lambda\rangle|\geq \frac{\gamma \xi^2 }{\sqrt{k}}$. Hence, we have $h(\mathrm{sign}(h(0)m)\mathrm{sign}(h(0)) \geq |h(0)|$ for all $m\geq 0$, and $|h(0)|\geq\frac{\gamma C_{s^*}\xi^2 }{\left(1+\frac{\xi^2}{k}\right)^{s^*} \sqrt{k}}$ where $s^*$ is the smallest $s$ for which $\mu_{s} \neq 0$.
\end{proof}

\subsection{Angularly separated case: population gradient lower bounds}
\label{app:separated}

\subsubsection{Computation of the population gradient}
Note that specializing $\xi=1$ in Proposition~\ref{prop:pop_gradient_pre}, we get
\begin{align*}
h(m)=\sum_{l=0}^\infty \left(\frac{1}{k}\right)^{l+1} \sum_{s=0}^\infty {l+s \choose l} (l+s+1) \mu_{l+s+1}(\sigma)^2 \left(\frac{k}{k+1}\right)^{l+s+1} T(l,s) m^l
\end{align*}

\subsubsection{Bounding the higher order even terms}
Initially, we aim to bound the even terms in the power series (i.e. $l > 1$).
We first prove the statement about bounding the higher order even terms as stated in \Cref{sec:technical-overview}, \Cref{proposition:separated-even-terms-bound}.
% \begin{lemma}
%     % Let $\sigma$ be such that there exists some $\rho, C_\sigma > 0$ for which $|\mu_{p}(\sigma)|\leq C_\sigma p^{-1-\rho}$. 
%     Suppose \Crefrange{assumption:normalize}{assumption:activation} hold. Then, with probability at least $1-\frac{1}{k^3}$ over the randomization of $c,\hat c$, for $\epsilon = \min \{\frac{\rho}{4}, 1-\frac{1}{1+2\rho}\}$ we have
%     \begin{align*}
%         \sum_{n=0}^\infty \left(\frac{1}{k}\right)^{2n+2} \sum_{s=0}^\infty {2n+2+s \choose 2n+2} (2n+s+3) \mu_{2n+s+3}(\sigma)^2 \left(\frac{k}{k+1}\right)^{2n+s+3} \left\langle \sum_{i=1}^k \lambda_i c_i w_i^{\otimes s}, \sum^k_{i=1} \lambda_i \hat c_i w_i^{\otimes s}\right\rangle\\
%         = O(\lambda_{\max}^2 k^{-\frac{1}{2}-\epsilon})
%     \end{align*}
% \end{lemma}
\begin{proof}[Proof of \Cref{proposition:separated-even-terms-bound}]
Let $s^* = 10\sqrt{k}$. This proof will involve bounding contributions from the following three types of terms:
\begin{enumerate}[label=(\roman*)]
    \item The contribution from the terms where $s\leq s^*$. These can be bounded naively since there are at most $O(\sqrt{k})$ of them, and the $(1/k)^{2n+2}$ will dominate the growth in $k$ in these terms.
    \item The contribution for $s \geq s^*$ from diagonal terms: These terms scale with $\sum_{i=1}^k \lambda_i^2 c_i \hat c_i$, so it suffices to show the coefficient is $O(k^{-\epsilon})$ for some small $\epsilon > 0$. This is due to the fact that the Hermite coefficients decay at rate $(s^*)^{-1-\rho}$, so the contribution of the large $s$ coefficients have to decay in $k$ at some small rate.
    \item The contribution for $s\geq s^*$ from non-diagonal terms: Due to the assumption of angular separation between the $w_i$'s, when $s$ is sufficiently large, the decay of the terms $\langle w_i, w_j\rangle^s$ means these terms will be small. 
\end{enumerate}

\noindent \textbf{(i) Contribution from terms with $s\leq s^*=O(\sqrt{k})$:} Initially, we bound the magnitudes of the randomized terms. Since there are at most $\sqrt{k}$ of them and they concentrate exponentially around their means, we can bound their magnitude by $O(\log k)$ with exponentially high probability. Specifically, 
\begin{align*}
    \E \left[\sum_{i,j=1}^k \lambda_i \lambda_j c_i \hat c_j \langle w_i, w_j\rangle^s\right]&= \sum_{i,j=1}^k \lambda_i \lambda_j  \langle w_i, w_j\rangle^s \E[c_i \hat c_j]= 0\\
    \E \left[\left(\sum_{i,j=1}^k \lambda_i \lambda_j c_i \hat c_j \langle w_i, w_j\rangle^s\right)^2\right]
    &= \sum_{i, i' =1}^k \sum_{j, j'=1}^k \lambda_i \lambda_{i'} \lambda_j \lambda_{j'} \langle w_i, w_j\rangle^s \langle w_{i'}, w_{j'}\rangle^s \E[c_i c_{i'} \hat c_j \hat c_{j'}]\\
     &=\sum_{i, i' =1}^k \sum_{j, j'=1}^k \lambda_i \lambda_{i'} \lambda_j \lambda_{j'} \langle w_i, w_j\rangle^s \langle w_{i'}, w_{j'}\rangle^s \E[c_i c_{i'}]\E[\hat c_j \hat c_{j'}]\\
    &=\frac{1}{k^2}\sum_{i=1}^k \sum_{j=1}^k \lambda_i^2 \lambda_j^2 \langle w_i, w_j\rangle^{2s}\\
    &\leq \frac{\norm{\lambda}_2^4}{k^2}\leq \lambda_{\max}^4\,.
\end{align*}
Then, define $f_s:\{-1,1\}^{2k} \to \R$ as $f_s(b, \hat b)=\frac{1}{k}\sum_{i,j=1}^k \lambda_i \lambda_j b_i \hat b_i \langle w_i,w_j\rangle^s$ which is a quadratic polynomial in $b_i, \hat b_i$. We have just proved that $\norm{f_s}_2 \leq \lambda_{\max}^2$. Then, by \cite[Theorem 9.23]{o2014analysis} we have
\begin{align*}
    \Pr_{b, \hat b} \left[|f_s(b,\hat b)| \geq \gamma \log k \norm{f}_2\right]\leq \exp\{-\frac{\gamma}{e}\log k\}= k^{-\frac{\gamma}{e}}
\end{align*}
where $\gamma>0$ is to be chosen later. Then, using the union bound, we have
\begin{align*}
    \Pr \left[\max_{s\leq s^*} \left|\sum_{i,j=1}^k\lambda_i \lambda_j c_i \hat c_j \langle w_i, w_j\rangle^s\right|\geq \gamma \lambda_{\max}^2 \log k \right]\leq s^* k^{-\frac{\gamma}{e}}
\end{align*}
As $s^* = O(\sqrt{k})$, then with probability at least $1-k^{-\frac{\gamma}{e} +\frac{1}{2}}$, we have
\begin{align}    
&
\left|\sum_{n=0}^\infty \left(\frac{1}{k}\right)^{2n+2} \sum_{s=0}^{s^*} {2n+2+s \choose 2n+2} (2n+s+3) \mu_{2n+s+3}(\sigma)^2 \left(\frac{k}{k+1}\right)^{2n+s+3} \left\langle \sum_{i=1}^k \lambda_i c_i w_i^{\otimes s}, \sum_{i=1}^k \lambda_i \hat c_i w_i^{\otimes s}\right\rangle\right| \nonumber\\
     &\qquad\leq\gamma  \lambda_{\max}^2 \log k \sum_{n=0}^\infty \left(\frac{1}{k}\right)^{2n+2} \sum_{s=0}^{s^*} {2n+2+s \choose 2n+2} (2n+s+3) \mu_{2n+s+3}(\sigma)^2 \left(\frac{k}{k+1}\right)^{2n+s+3} \label{eq:infinitesum}
\end{align}
Now, it suffices to give a $O(k^{-\frac{1}{2}-c\epsilon})$ bound for the infinite sum for $c > 1$. We will separate it into cases $s\leq (s^*)^{1-\epsilon}$ and $(s^*)^{1-\epsilon} \leq s \leq s^*$. The reason for this is that we have to use the decay of the Hermite coefficients as $s$ approaches $\sqrt{k}$, so the two cases need to be handled separately. Hence, for $l\triangleq 2n + 2$ using the binomial coefficient bound ${n \choose k}\leq \left(\frac{en}{k}\right)^k$ we have 
\begin{align*}
    \sum_{s=0}^{(s^*)^{1-\epsilon}} {l+s \choose l} (l+s+1) \mu_{l+s+1}(\sigma)^2 \left(\frac{k}{k+1}\right)^{l+s+1}&\leq \sum_{s=0}^{(s^*)^{1-\epsilon}} C_\sigma \left(e\frac{l+s}{l}\right)^l\\
    &\leq C_\sigma e^l \sum_{s=0}^{(s^*)^{1-\epsilon}} (1+s)^l\\
    &\leq C_\sigma e^l (s^*)^{1-\epsilon} (1+(s^*)^{1-\epsilon})^l \\
    &\leq C_\sigma (s^*)^{1-\epsilon} (2e(s^*)^{1-\epsilon})^l
\end{align*}
Then, notice that for $k$ larger than some absolute constant, we have
\begin{align*}
    C_\sigma (s^*)^{1-\epsilon}\sum_{n=0}^\infty \left(\frac{1}{k}\right)^{2n+2} \left(2e(s^*)^{1-\epsilon}\right)^{2n+2}&\leq C_\sigma (s^*)^{1-\epsilon} \left(\frac{2e(s^*)^{1-\epsilon}}{k}\right)^{2} \frac{1}{1+o(1)}= O(k^{-\frac{1}{2}-\frac{3}{2}\epsilon})
\end{align*}
since $(s^*)^{3(1-\epsilon)}k^{-2}= O(k^{-\frac{1}{2}-\frac{3}{2}\epsilon})$. 

Now, we look at the remaining terms. For $(s^*)^{1-\epsilon} \leq s\leq s^*$, we have
\begin{align*}
    \left(\frac{1}{k}\right)^{l}\sum_{(s^*)^{1-\epsilon}\leq s\leq s^*} {l+s \choose l} (l+s+1) \mu_{l+s+1}(\sigma)^2 \left(\frac{k}{k+1}\right)^{l+s+1}&\leq C_\sigma (s^*)^{-(1-\epsilon) (1+2\rho)}  \sum_{(s^*)^{1-\epsilon}\leq s\leq s^*}\left(\frac{2es^*}{k}\right)^{l}\\
    &\leq C_\sigma (s^*)^{1-(1-\epsilon) (1+2\rho)}
    \left(\frac{2es^*}{k}\right)^l
\end{align*}
Taking the sum over all $l\triangleq 2n+2$, we have
\begin{align*}
    C_\sigma (s^*)^{1-(1-\epsilon)(1+2\rho)} \sum_{n=0}^\infty 
    \left(\frac{2es^*}{k}\right)^{2n+2}&\leq C_\sigma (s^*)^{1-(1-\epsilon)(1+2\rho)} \left(\frac{2e s^*}{k}\right)^2 \frac{1}{1+o(1)}\,.
\end{align*}
Choosing $\epsilon=1-\frac{1}{1+2\rho} > 0$ for simplicity\footnote{There are more optimal choices of $\epsilon$ that lead to better bounds}, we have that the sum is bounded by $C_\sigma \left(\frac{2s^*}{k}\right)^2 \frac{1}{1+o(1)}= O(\frac{1}{k})$. Hence, combining with previous steps, we can upper bound the infinite sum in Equation~\eqref{eq:infinitesum} by $O(\lambda_{\max}^2 k^{-\frac{1}{2}-3\epsilon})$ where $\epsilon = 1-\frac{1}{1+2\rho}$.

\noindent \textbf{(ii) The contribution of $s\geq s^*$ for diagonal terms:} We first note that
\begin{align*}
    \sum_{p=1}^\infty p \mu_{p}(\sigma)^2 \left(\frac{k}{k+1}\right)^p \langle w_i + c_i u, w_i + \hat c_i \hat u\rangle^{p-1}&=\sum_{p=1}^\infty p \mu_{p}(\sigma)^2 \left(\frac{k}{k+1}\right)^p (\langle w_i, w_j\rangle + c_i \hat c_j m)^{p-1}
\end{align*}
Then, notice that the RHS is maximized in absolute value when $w_i=w_j$, $c_i = \hat c_j$ and $m=1$. In this case, we get
\begin{align*}
    \left|\sum_{p=1}^\infty p \mu_{p}(\sigma)^2 \left(\frac{k}{k+1}\right)^p \langle w_i + c_i u, w_i + \hat c_i \hat u\rangle^{p-1}\right|\leq \sum_{p=1}^\infty p \mu_p(\sigma)^2 \triangleq \tilde C_\sigma
\end{align*}
In particular, we have absolute convergence of the LHS for all $|m|\leq 1$, so we can freely interchange order of sums. However, notice all steps in this argument works if we replace $\mu_p(\sigma)^2$ with something else that has sufficiently fast decay. In particular, writing $p=l+s+1$ we have
\begin{align}
    \sum_{l=0}^\infty \left(\frac{1}{k}\right)^l\sum_{s=0}^\infty {l+s \choose l} (l+s+1) \mu_{l+s+1}(\sigma)^2 \left(\frac{k}{k+1}\right)^{l+s+1}&=\sum_{p=1}^\infty \left(\frac{k}{k+1}\right)^p p \mu_p(\sigma)^2 \sum_{l=0}^{p-1} \left(\frac{1}{k}\right)^l{p-1 \choose l}\nonumber \\
    &=\sum_{p=1}^\infty \left(\frac{k}{k+1}\right)^p\left(1+\frac{1}{k}\right)^{p-1}p \mu_p(\sigma)^2 \nonumber \\
    &\leq \sum_{p=1}^\infty p \mu_p(\sigma)^2 = \tilde C_\sigma \label{eq:infinitesum2}
\end{align}
However, since all the terms in the sum are non-negative, using the same steps, we have
\begin{align*}
    &\sum_{l=0}^\infty \left(\frac{1}{k}\right)^l\sum_{s=s^*}^\infty {l+s \choose l} (l+s+1) \mu_{l+s+1}(\sigma)^2 \left(\frac{k}{k+1}\right)^{l+s+1}\\
    &\leq \sum_{l=0}^\infty \left(\frac{1}{k}\right)^l\sum_{s=s^*}^\infty {l+s \choose l} (l+s+1)^{-1-2\rho} \left(\frac{k}{k+1}\right)^{l+s+1}\\
    &\leq (s^*)^{-\rho} \sum_{l=0}^\infty \left(\frac{1}{k}\right)^l\sum_{s=s^*}^\infty {l+s \choose l} (l+s+1)^{-1-\rho} \left(\frac{k}{k+1}\right)^{l+s+1}\\
    &\leq (s^*)^{-\rho}\sum_{p=1}^\infty p^{-1-\rho} = \hat C_\sigma (s^*)^{-\rho}
\end{align*}
where $\hat C_\sigma = \sum_{p=1}^\infty \frac{1}{p^{1+\rho}}$.\footnote{$\hat{C}_\sigma$ depends on $\sigma$ through the definition of $\rho$ in \Cref{assumption:activation}.}
Then,
\begin{align*}
    &\left|\sum_{n=0}^\infty \left(\frac{1}{k}\right)^{2n+2} \sum_{s=s^*}^{\infty} {2n+2+s \choose 2n+2} (2n+s+3) \mu_{2n+s+3}(\sigma)^2 \left(\frac{k}{k+1}\right)^{2n+s+3}\sum_i \lambda_i^2 c_i \hat c_i\right|\\
    &\leq \hat C_\sigma (s^*)^{-\rho}|\sum_i \lambda_i^2 c_i \hat c_i|\,.
\end{align*}
Then, notice that since $\sqrt{\E[(\sum_i \lambda_i^2 c_i\hat c_i)^2]}=\sqrt{\frac{1}{k^2} \sum_{i=1}^k \lambda_i^4}\leq \lambda_{\max}^2/\sqrt{k}$, we have
\begin{align*}
    \Pr[|\sum_i \lambda_i^2 c_i \hat c_i|\geq \gamma \lambda_{\max}^2\frac{\log k}{\sqrt{k}}]&\leq k^{-\frac{\gamma}{e}}
\end{align*}
by another application of \cite[Theorem 9.23]{o2014analysis}.
Then, with probability at least $1-\frac{1}{k^{\gamma/e}}$, we have
\begin{align*}
    \hat C_\sigma (s^*)^{-\rho}|\sum_i \lambda_i^2 c_i \hat c_i|\leq \hat C_\sigma (s^*)^{-\rho} \gamma \lambda_{\max}^2 \frac{\log k}{\sqrt{k}}= O(\lambda_{\max}^2 k^{-\frac{1}{2}-\frac{\rho}{4}})
    \end{align*}
as claimed.

\noindent \textbf{(iii) Bounding the non-diagonal terms for $s\geq s^*$}: Notice that
\begin{align*}
    \left|\sum_{i\neq j}^k \lambda_i \lambda_j c_i \hat c_j \langle w_i, w_j\rangle^s\right| 
    % &= \sqrt{\left(\sum_{i\neq j}^k \lambda_i \lambda_j c_i \hat c_j \langle w_i, w_j\rangle^s\right)^2}\\
    & \leq\sqrt{k^2 \sum_{i\neq j} \lambda_i^2 \lambda_j^2 c_i^2 \hat c_j^2\langle w_i, w_j\rangle^{2s}}\\
    & \leq \left(1-\frac{\log k}{\sqrt{k}}\right)^s \norm{\lambda}_2^2\,.
\end{align*}
Then, let $s\geq s^*=\gamma \sqrt{k}$. Then,
\begin{equation*}
    \left(1-\frac{\log k}{\sqrt{k}}\right)^s\norm{\lambda}_2^2
    % &\leq \left(1-\frac{\log k}{\sqrt{k}}\right)^{\gamma \frac{\sqrt{k}}{\log k} \log k}\norm{\lambda}_2^2\\
    \leq e^{-\gamma \log k} \norm{\lambda}_2^2=\frac{\norm{\lambda}_2^2}{k^\gamma}\,,
\end{equation*}
so setting $\gamma > \frac{3}{2}$ will suffice. I.e, we have
\begin{align*}
    &\left|\sum_{n=0}^\infty \left(\frac{1}{k}\right)^{2n+2} \sum_{s=s^*}^\infty {2n+2+s \choose 2n+2} (2n+s+3) \mu_{2n+s+3}(\sigma)^2 \left(\frac{k}{k+1}\right)^{2n+s+3} \left(\sum_{i\neq j} \lambda_i \lambda_j c_i \hat c_j \langle w_i, w_j\rangle^s\right)\right|\\
    &\leq \frac{\norm{\lambda}_2^2}{k^\gamma}\sum_{n=0}^\infty \left(\frac{1}{k}\right)^{2n+2} \sum_{s=s^*}^\infty {2n+2+s \choose 2n+2} (2n+s+3) \mu_{2n+s+3}(\sigma)^2 \left(\frac{k}{k+1}\right)^{2n+s+3}\\
    &\leq \frac{\tilde C_\sigma \norm{\lambda}_2^2}{k^\gamma}\,,
\end{align*}
where in the last step we used Equation~\eqref{eq:infinitesum2}.
Combining all the bounds, for $\epsilon = \min \{\frac{\rho}{4}, 1-\frac{1}{1+2\rho}\}$, with probability at least $1-\gamma\frac{1}{k^{\gamma/e-\frac{1}{2}}}$, we have
\begin{align*}
    \sum_{n=0}^\infty \left(\frac{1}{k}\right)^{2n+2} \sum_{s=0}^\infty {2n+2+s \choose 2n+2} (2n+s+3) \mu_{2n+s+3}(\sigma)^2 \left(\frac{k}{k+1}\right)^{2n+s+3} \left\langle \sum_{i=1}^k \lambda_i c_i w_i^{\otimes s}, \lambda_i \hat c_i w_i^{\otimes s}\right\rangle\\
    = O(\lambda_{\max}^2 \gamma k^{-\frac{1}{2}-\epsilon})
\end{align*}
Specifically, setting $\gamma = 10$, the result holds with probability at least $1-\frac{1}{k^3}$.
\end{proof}

\subsubsection{Proving the lower bound on $h$}

Now, we prove \Cref{lemma:pop-gradient-bound-separated} from \Cref{sec:technical-overview}. 
\begin{proof}[Proof of \Cref{lemma:pop-gradient-bound-separated}]
    Let $\gamma > 0$ be a small constant. Then, notice that using the anti-concentration of $h(0)$ (\Cref{proposition:spectral-constant-anti-concentration}) and the variance bound $\Omega(\frac{\lambda_{\min}^2}{k})$ for $h(0)$ from \Cref{claim:constant-term-variance}, with probability $1- O(\gamma^{1/2}) - o(\frac{\lambda_{\max}^2}{\lambda_{\min}^2})$, we have 
    \begin{equation*}
        |h(0)| \geq \frac{\gamma \lambda_{\min}^2}{\sqrt{k}}\,.
    \end{equation*}
    However, note that all the even terms are $O(k^{-\frac{1}{2}-\epsilon})$. Hence, we can bound the even terms by $|h(0)|/2$ with high probability. Then, we get that whenever the sign condition $m h(0) > 0$ is satisfied, we have $\mathrm{sign}(m) h(m) \geq \frac{\gamma \lambda_{\min}^2}{2}$ with high probability, as desired.
\end{proof}

% \documentclass{article}
% \usepackage{preamble}
% \usepackage{todonotes}
% \begin{document}

\section{Finite-sample analysis}\label{appendix:finite-sample}
The goal of this section is to prove \Cref{theorem:modular-convergence-dynamics}. First, define the following notation for the noisy gradients and gradient error:
\begin{align*}
    L_t &= \hat \nabla_{u_t} L(u_t;x_t)\\
    E_t&= L_t - \hat \nabla_{u_t} \Phi(u_t).
\end{align*}
Now, in the next sections we initially bound the terms contributed by the spherical projection error, and then the error martingale. Finally, we combine our results to show weak recovery and strong recovery under \Crefrange{condition:unbiased-gradients-modular}{condition:pop-grad-form-modular}. Throughout this section, we use $\mF_t$ to denote the sigma algebra generated by the iterates $u_t$.
\subsection{Analysis of dynamics under the generic assumptions}

Recall the online SGD dynamics
\begin{align*}
    u_{t+1} = \frac{u_t - \eta \hat \nabla_{u_t} L(u_t;x_t)}{\norm{u_t - \eta \hat \nabla_{u_t} L(u_t;x_t)}}
\end{align*}
where $x_t \sim \mN(0, I_d)$ is a fresh Gaussian sample at each time iteration $t$. Then, define the correlation with ground truth $m_t = \langle u_t, u\rangle$ and the projection magnitude $\Pi_t = \norm{u_t - \eta \hat \nabla_{u_t} L(u_t;x_t)}$. Then, notice
\begin{align*}
    m_{t+1} &= \frac{m_t - \eta \langle \hat \nabla_{u_t} L(u_t;x_t), u\rangle}{\Pi_t}\\
    &=m_t -\eta \hat \nabla_{u_t}\Phi(u_t) - \eta \langle \hat \nabla_{u_t} E(u_t;x_t), u\rangle - \left(1-\frac{1}{\Pi_t}\right) \left(m_t - \eta \langle \hat \nabla_{u_t} L(u_t;x_t), u\rangle \right)\,.
\end{align*}
Hence, initially, we bound the effect of the spherical projection term.
\subsubsection{Bounding spherical projection error}
We initially bound the spherical projection error in terms of $L_t$ and $\langle L_t, u\rangle$. This will later allow us to use the tail bounds on $L_t$ to bound the spherical projection error.
\begin{claim}[Relating spherical projection error to $L_t$]
    \begin{equation*}
\sum_{j=0}^{t-1}\left|\left(1-\frac{1}{\Pi_j}\right)(m_j - \eta \langle \hat \nabla_{u_j} L(u_j;x_j), u \rangle)\right|\leq \eta^3 \sum_{j=0}^{t-1} \norm{L_j}^2 |\langle L_j, u\rangle| + \eta^2 \sum_{j=0}^{t-1} \norm{L_j}^2
\end{equation*}
\label{claim:spherical-projection-apriori-bound}
\end{claim}
\begin{proof}
    First, notice that because $u_t$ is perpendicular to the spherical gradient $\hat \nabla_{u_t}\Phi(u_t)$, we have
\begin{equation*}
    1\leq \Pi_t \leq \sqrt{1+\eta^2 \norm{\hat \nabla_{u_t} L(u_t;x_t)}_2^2}\leq 1+\eta^2 \norm{\hat \nabla_{u_t} L(u_t;x_t)}_2^2
\end{equation*}
Then, due to $\left|1-\frac{1}{1+x}\right|\leq x$ for $x\geq 0$, we have

\begin{align*}
\left|\left(1-\frac{1}{\Pi_t}\right)(m_t - \eta \langle \hat \nabla_{u_t} L(u_t;x_t), u \rangle)\right|\leq \eta^2 \norm{L_t}^2 (|m_t|+ \eta |\langle L_t, u\rangle|)
\end{align*}
Bounding $|m_t|\leq 1$, notice that the total contribution of these terms up to time $t$ can be upper bounded by
\begin{align*}
    \eta^3 \sum_{j=0}^{t-1} \norm{L_t}^2 |\langle L_t, u\rangle| + \eta^2 \sum_{j=0}^{t-1} \norm{L_t}^2\,.\qedhere
\end{align*}
\end{proof}

Then, notice that $\eta^3$ gives a $\frac{\delta^3}{d^3 V_k^3}$ scaling, but $\norm{L_t}^2 |\langle L_t, u\rangle|$ scales only in $dV_k^2$, and there are $T = \alpha d V_k$ of these. Then, we can use a simple Markov bound to bound these terms when $\alpha \delta^2 \leq \epsilon$.

\begin{claim}[Bounding cubic terms] 
    \label{claim:cubic-projection-modular}
Let $\alpha, \delta$ be such that $\alpha \delta^2 \leq \epsilon$ and $\delta \leq 1$. Then, we have
    \begin{equation*}
        \Pr\left[\sup_{0\leq t \leq T}\eta^3 \sum_{j=0}^{t} \norm{L_j}^2 |\langle L_j, u\rangle| > \frac{\beta}{10\sqrt{d}}\right] \lesssim \frac{1}{\beta \sqrt{d}}
    \end{equation*}
    Similarly, we have
    \begin{equation*}
     \Pr\left[\sup_{0\leq t \leq T}\eta^3 \sum_{j=0}^{t} \norm{L_j}^2 |\langle L_j, u\rangle| > \frac{\epsilon}{18}\right] \lesssim \frac{1}{d}  
    \end{equation*}
\end{claim}
\begin{proof}
Notice that in both cases the maximum is achieved at $t=T$ due to the non-negativity of the terms in the sum. Then, by Markov
\begin{equation*}
    \Pr\left[\sup_{t \leq T}\eta^3 \sum_{j=0}^{t} \norm{L_j}^2 |\langle L_j, u\rangle| >\gamma \right] =\Pr\left[\eta^3 \sum_{j=0}^{T} \norm{L_j}^2 |\langle L_j, u\rangle| > \gamma\right] \leq \frac{\eta^3 T  \sup_j \E[\norm{L_j}^2 |\langle L_j, u\rangle|]}{\gamma}\,.
\end{equation*}
Now, using Cauchy-Schwarz to bound the expectation, we have
\begin{align*}
    \E[\norm{L_j}^2 |\langle L_j, u\rangle|]\leq \norm{\norm{L_j}^2}_2 \sqrt{\norm{|\langle L_j, u\rangle|^2}_1}
\end{align*}
Hence, using the moment bounds (\Cref{condition:variances-modular}) on $\norm{L_t}^2$ and $|\langle L_t, u\rangle|^2$, for $p=2,1$ respectively, we have 
\begin{align*}
    \E[\norm{L_j}^2 |\langle L_j, u\rangle|] \lesssim d V_k^2  
\end{align*}
Finally, using $\eta = \frac{\delta}{dV_k}$, $T = \alpha dV_k$ and $\alpha \delta^2 \leq \epsilon$, $\delta \leq 1$, we have 
\begin{equation*}
    \Pr\left[\sup_{t \leq T}\eta^3 \sum_{j=0}^{t} \norm{L_j}^2 |\langle L_j, u\rangle| > \gamma \right]\lesssim \frac{\alpha d^2 V_k^3 \eta^3}{\gamma} = \frac{\alpha \delta^3}{d\gamma}\leq \frac{1}{d\gamma}
\end{equation*}
Setting $\gamma = \frac{\beta}{10\sqrt{d}}$ gives us the first result. For the second, we can use $\alpha \delta^2 \leq \epsilon$ and $\delta \leq 1$ to bound the probability by $\frac{1}{d}$.
\end{proof}
\noindent Now, we turn to the quadratic term. Notice that with the quadratic term, we are not necessarily getting the extra scaling in $1/d$ from $\eta$ we need, so we need to be more careful while bounding this term. For these terms, we will show that their cumulative effect at any given iteration is smaller than the drift contribution. To do this we need to uniformly bound the cumulative effect up to iteration $t$. Recall Freedman's inequality \cite{freedman1975tail} for submartingales with almost sure bounds:
\begin{lemma}[Freedman's inequality]
    \label{lemma:freedman}
    Let $M_t$ be a submartingale with $\E[(M_{t+1} - M_{t})^2 | \mF_t] \leq V$ and $|M_{t+1} - M_t| \leq K$ almost surely. Then,
    \begin{align*}
        \Pr[S_t \leq -\lambda] \leq \exp\left\{\frac{-\lambda^2}{tV + \frac{\lambda}{3}K}\right\}
    \end{align*}
\end{lemma}

\noindent Hence, we will introduce an appropriate clipping of $\norm{L_t}$ and separate into cases when it is large and small. When it is large, we will use the fast decay of its tails due to bounded moments the bound the probability of being large. When it is small, we will use the almost sure bound and Freedman's inequality to control the total contribution.

\begin{claim}[Bounding the quadratic terms] 
    \label{claim:quadratic-projection-martingale-modular}
Suppose $\alpha$ has at most polynomial growth in $d, k$. Furthermore suppose, $\alpha \delta^2 \leq 1$, and that $V_k$ has polynomial growth in $k$. Then, for some constant $C$, we have
\begin{align*}
    \Pr\left[\inf_{0\leq t\leq T} \eta \sum_{j=0}^{t} \left(\frac{S_k}{4} -\eta \norm{L_t}^2\right) < \frac{\beta}{-5\sqrt{d}}\right] \leq \frac{C}{\beta \sqrt{d}} + \alpha (dV_k)^{-\frac{\beta^2}{C}(\log dV_k)+1}
\end{align*}
\end{claim}
\begin{proof}
Initially, define $Y_t=\frac{\norm{L_t}^2}{dV_k}$ and notice that $\norm{Y_t}_p \leq \mu_p$ for all $t \geq 0$ where $\mu_p$ do not grow in $d$ or $k$ as stated in \Cref{condition:variances-modular}. Then, notice that $\eta \norm{L_t}^2 = \delta Y_t$. We write $Y_t = Y_t\ind{Y_t \geq T^\nu}+Y_t\ind{Y_t < T^\nu}$. Then, we can decompose the term as
\begin{align*}
    \eta\sum_{j=0}^{t} \left(\frac{S_k}{2} -\eta \norm{L_t}^2\right)&=\eta\sum_{j=0}^{t} \left(\frac{S_k}{2} -\delta \norm{Y_t}^2 \ind{Y_t \geq T^\nu}\right)+\eta\sum_{j=0}^{t} \left(\frac{S_k}{2} -\delta \norm{Y_t}^2 \ind{Y_t < T^\nu}\right)\\
    &\geq -\eta\sum_{j=0}^{t}\delta \norm{Y_t}^2 \ind{Y_t \geq T^\nu}+\eta\sum_{j=0}^{t} \left(\frac{S_k}{2} -\delta \norm{Y_t}^2 \ind{Y_t < T^\nu}\right)
\end{align*}
where we used $\frac{S_k}{2} >0$ for the last inequality. Then, it suffices to show that the second line is at least $-\frac{\beta}{5\sqrt{d}}$. Hence, we will bound the probability of each term being less than $-\frac{\beta}{10\sqrt{d}}$ and use the union bound.

Then, notice that for fixed choice of $\nu, D > 0$ we have 
\begin{align*}
    \Pr[Y_t \geq T^{\nu}]=\Pr[Y_t^{D/\nu} \geq T^{D}]\leq \frac{\E[Y_t^{D/\nu}]}{T^D}
\end{align*}
Then, letting $D/\nu = p$ and using the $p$'th moment bound \Cref{condition:variances-modular}, there exists a constant $C_{\nu, D}$ such that 
\begin{align*}
    \Pr[Y_t \geq T^{\nu}] \leq \frac{C_{\nu, D}}{T^D}
\end{align*}
where we used $V_k \geq 1$.
Then, notice that, using Cauchy-Schwarz, we have
\begin{align*}
    \E[Y_t \ind{Y_t \geq T^\nu}] \leq \norm{Y_t}_2\sqrt{\Pr[Y_t \geq T^\nu]} \leq \frac{C_{\nu, D}}{T^{D/2}}
\end{align*}
where we absorbed the $\mu_2$ constant into the $C$.
Then, we have
\begin{equation*}
    \Pr\left[\eta \sum_{j=0}^{T-1}Y_t \ind{Y_t \geq T^\nu} > \gamma \right] \leq \frac{\eta T C_{\nu, D}}{\gamma T^{D/2}}
\end{equation*}
Then, we can choose $D=1$ (and get rid of the $D$ dependence on the constants), and $\gamma = \frac{\beta}{10\sqrt{d}}$ such that 
\begin{equation*}
    \Pr\left[\eta \sum_{j=0}^{T-1}Y_t \ind{Y_t \geq T^\nu} > \frac{\beta}{10\sqrt{d}} \right] \lesssim \frac{\sqrt{d}\eta C_{\nu}}{\beta}\leq \frac{\delta C_{\nu}}{\sqrt{d}V_k\beta}\leq \frac{\delta C_{\nu}}{\beta\sqrt{d}}
\end{equation*}

Then, notice that we are left with the term $Y_t \ind{Y_t \leq T^\nu}$ where $\nu$ can be chosen arbitrarily small. Consider setting $\delta \leq \frac{S_k}{4C_\delta \log (dV_k)}$ such that

\begin{align*}
    \eta \sum_{j=0}^t \left(\frac{S_k}{2}- \delta Y_t \ind{Y_t \leq T^\nu}\right)&\geq \frac{\eta S_k}{4} \sum_{j=0}^t \left(1-\frac{Y_t \ind{Y_t \leq T^\nu}}{C_\delta \log (dV_k)}\right)\\
    &\geq \frac{\eta S_k}{4 \log (dV_k)} \sum_{j=0}^t \left(1-\frac{Y_t \ind{Y_t \leq T^\nu}}{C_\delta}\right)
\end{align*}
However, since $\E Y_t$ is bounded by $1$, for $C_\delta > \mu_1$, the following forms an $\mF_t$ submartingale:
\begin{align*}
    Z_t = \frac{\eta S_k}{2 \log (dV_k)} \sum_{j=0}^t \left(1-\frac{Y_t \ind{Y_t \leq T^\nu}}{C_\delta}\right)
\end{align*}
Then, it suffices to show 
\begin{align*}
    \Pr\left[\inf_{0\leq t\leq T} Z_t  <-\frac{\beta}{10\sqrt{d}}\right]=o(1) 
\end{align*}
Then, note $\E[Y_t\ind{Y_t \leq T^\nu}]\leq \E[Y_t] = O(1)$, and we have the almost sure bound
\begin{align*}
    |Z_{t+1}-Z_t|\leq \frac{\eta S_k}{2\log (dV_k)}\left(1+\frac{T^\nu}{C_\delta}\right)\leq  \frac{\eta S_k}{\log (dV_k)} \frac{T^\nu}{C_\delta}
\end{align*}
and the conditional variances
\begin{align*}
    \E[(Z_{t+1}-Z_t)^2 | F_t] \leq \frac{\eta^2 S_k^2}{4(\log dV_k)^2} \left(1+\mu_2^2\right)\leq \frac{C \eta^2 S_k^2}{(\log dV_k)^2}
\end{align*}
where $C$ is a constant that can only depend on $\mu_2$.

Then, using Freedman's inequality for submartingales, for any $0\leq t\leq T$ we have
\begin{align*}
    \Pr\left[Z_t \leq -\frac{\beta}{10\sqrt{d}}\right]\leq \exp\left\{\frac{-\frac{\beta^2}{100d}}{\frac{C T \eta^2 S_k^2}{(\log dV_k)^2} + \frac{\beta \eta S_k}{30\sqrt{d} \log (dV_k)}\frac{T^\nu}{C_\delta}}\right\}
\end{align*}

Let's inspect the expression in the exponent. Note, using $\alpha \delta^2 \leq 1$ and equivalently $\delta \alpha^\nu \leq 1$, for some updated constant $C=C(\mu_2)$ we have
\begin{align*}
    \frac{-\frac{\beta^2}{100d}}{\frac{C T \eta^2 S_k^2}{(\log dV_k)^2} + \frac{\beta \eta S_k}{10\sqrt{d} \log (dV_k)}\frac{T^\nu}{C_\delta}}&= -\frac{\beta^2}{\frac{C \alpha \delta^2 S_k^2}{V_k (\log dV_k)^2}+\frac{10 \beta \delta\alpha^\nu S_k}{V_k^{1-\nu} d^{1/2-\nu}\log (dV_k)}}\\
    &\leq -\beta^2 \min\left\{\frac{V_k (\log dV_k)^2}{C S_k^2}, \frac{V_k^{1-\nu} d^{1/2-\nu} \log (dV_k)}{10\beta S_k}\right\}\\
    &\leq -\frac{\beta^2}{C} (\log dV_k)^2 V_k^{1/2}
\end{align*}
for sufficiently large $d$ greater than some $O(1)$, where we have $\frac{V_k}{S_k}\geq 1$ and $\frac{V_k}{S_k^2}\geq 1$ when $\nu = 1/4$. Hence, taking the exponent, we have $\exp\{-\frac{\beta^2}{C} (\log dV_k)^2 V_k^{1/2}\}=(dV_k)^{-\frac{\beta^2}{C}(\log dV_k) }$
Then, doing a union bound over all $t\leq T$, we have
\begin{align*}
    \Pr\left[\inf_{0\leq t\leq T-1}Z_t \leq -\frac{\beta}{10\sqrt{d}}\right] \leq T (dV_k)^{-\frac{\beta^2}{C}(\log dV_k) }= \alpha (dV_k)^{-\frac{\beta^2}{C}(\log dV_k)+1}
\end{align*}
which is $o(1)$ when $\alpha$ has at most polynomial growth and $V_k$ has polynomial growth in $k$.
\end{proof}

\begin{claim}
    \label{claim:quadratic-projection-modular}
    Let $\alpha \delta^2 \leq \frac{\epsilon^2}{\log d}$. Then
    \begin{equation*}
        \Pr\left[\sup_{0\leq t\leq T} \eta^2 \sum_{j=0}^t \norm{L_t}^2 >\frac{\epsilon}{18}\right] \lesssim \frac{1}{\log d}
    \end{equation*}
\end{claim}
\begin{proof}
    Note that the maximum is achieved at $T$ since all the summands are non-negative. In that case, 
    \begin{align*}
        \Pr\left[\eta^2 \sum_{j=0}^{T} \norm{L_t}^2 > \frac{\epsilon}{18}\right] \lesssim \frac{\eta^2 T \E[\norm{L_t}^2]}{\epsilon^2} \leq \frac{\mu_1 \alpha \delta^2 d^2 V_k^2}{d^2 V_k^2\epsilon^2} =\frac{\mu_1 \alpha \delta^2}{\epsilon^2} \leq \frac{1}{\log d} & = o(1)\,. \qedhere
    \end{align*}
\end{proof}

\subsection{Controlling the error martingale}

\begin{claim}
    \label{claim:error-martingale-modular}
Let $\alpha \delta^2 \leq \epsilon^2 (\log d)^{-1}$. Furthermore, let $M_t = \eta\sum_{0\leq j\leq t-1} \langle E_j, u\rangle$. Then, $M_t$ forms a $\mF_t$ martingale and 
    \begin{align*}
        \Pr\left[\sup_{0\leq t \leq T} |M_t|\geq \frac{\beta}{10\sqrt{d}}\right]\lesssim \frac{\epsilon^2}{\beta^2 \log d}
    \end{align*}
    Furthermore, we have
    \begin{align*}
        \Pr\left[\sup_{0\leq t \leq T_1} |M_t|\geq \frac{\epsilon}{18}\right]\lesssim \frac{1}{d\log d}
    \end{align*}
\end{claim}
\begin{proof}
The fact that $M_t$ is a martingale follows directly from \Cref{condition:unbiased-gradients-modular} and the fact that each $x_t$ is a fresh sample. By Doob's maximal inequality for martingales, we have
    \begin{equation*}
        \Pr\left[\sup_{0\leq t\leq T} |M_t|> \gamma\right]\leq \frac{\E{M_{T}^2}}{\gamma^2}
        \leq \frac{2 \mu_1 \eta^2 T V_k}{\gamma^2}=\frac{2\mu_1 \alpha \delta^2}{d\gamma^2}\,.
    \end{equation*}
    Setting $\gamma = \frac{\beta}{10\sqrt{d}}$, we get the probability is at most $\frac{\epsilon^2}{\beta^2 \log d}$ up to constants. For the second result, set $\gamma = \frac{\epsilon}{18}$ so that the probability is $O(\frac{1}{d\log d})$.
\end{proof}

\subsection{Weak recovery and strong recovery}

Before we prove weak and strong recovery, we would like to define events $\mA$ and $\mB$ that capture the probabilistic bounds on population gradient magnitude and the various error terms in the dynamics. 

\subsubsection{Good event for error bounds and initial correlation}

First, define the event $\mA$ as 
\begin{equation}
    \mA = \Bigl\{m_0 \geq \frac{\beta\cdot \mathrm{sign}(h(0))}{\sqrt{d}} \Bigr\}\,.
\end{equation}
Furthermore, define the event $\mB = \mB(\epsilon, d, \beta, k, T)$ that corresponds to the error bounds as the following
\begin{align}
    \mB &= \left\{\sup_{0\leq t\leq T} |M_t| \leq \min\left\{\frac{\beta}{10\sqrt{d}}, \frac{\epsilon}{36}\right\}\right\} \cap \left\{\sup_{0\leq t \leq T} \eta^3 \sum_{j=0}^{t-1} \norm{L_j}^2 |\langle L_j, u\rangle| \leq \min\left\{\frac{\beta}{10\sqrt{d}}, \frac{\epsilon}{18}\right\}\right\} \label{eq:event-b} \\
    &\qquad \cap \left\{\sup_{0\leq t \leq T} \eta^2 \sum_{j=0}^t \norm{L_t}^2 \leq \frac{\epsilon}{18}\right\} \cap \left\{\sup_{0\leq t \leq T} \eta \sum_{j=0}^t \left(\frac{S_k}{4}-\eta \norm{L_t}^2\right) \geq - \frac{\beta}{5\sqrt{d}}\right\}  \nonumber
\end{align}

\begin{proposition}
    \label{prop:dynamics-b-holds-whp}
    Let $\delta = \frac{\epsilon^3 S_k}{4C_\delta \log (dV_k)}$ where $C_\delta > \max\{1,\mu_1\}$. Furthermore suppose that $\alpha = \frac{4(\log d V_k)}{\epsilon \delta S_k}$. Then, for $T = \lceil \alpha dV_k\rceil$, we have $\Pr(\mB(\epsilon, d, \beta, k, T)) = 1- O\left(\max \left\{\frac{1}{\beta \sqrt{d}}, \alpha (dV_k)^{-\frac{\beta^2}{C}(\log dV_k) +1}, \frac{\epsilon^2}{\beta^2 \log d}, \frac{1}{d\log d}\right\}\right)=1-o(1)$.
\end{proposition}
\begin{proof}
    Notice that the given $\delta, \alpha$ satisfy $\alpha \delta^2 \leq \frac{\epsilon^2}{C_\delta \log (dV_k)}$. Hence, all of \Crefrange{claim:cubic-projection-modular}{claim:quadratic-projection-modular} hold. Then, combining the results of the claims with a union bound gives the result.
\end{proof}

\subsubsection{Stopping times for the dynamics}

Initially, for a real number $q > 0$, define the stopping times 
\begin{align*}
    \tau_q^+ &= \inf \{t \geq 0: m_t \geq q\}\\
    \tau_q^- &= \inf \{t \geq 0: m_t \leq q\}
\end{align*}
which correspond to the first time $m_t$ is above/below a certain threshold value $q$. In particular, we will define the following stopping times
\begin{align*}
    \tau_r^+ &= \inf \{t\geq 0: m_t > r\}\\
    \tau_{0}^- &=  \inf \{t \geq 0: m_t < 0\}\\
    \tau_{1-\epsilon/6}^+ &= \inf \{t\geq 0: m_t \geq 1-\frac{\epsilon}{6}\}
\end{align*}
$\tau_r^+$ is defined to analyze the initial stage of training, when $m_t$ is small. This allows us to lower bound the effect of the spherical projection of the gradients $1-m_t^2$. We will use $\tau_0^-$ to be able to lower bound the population gradient, but we will get rid of the requirement with an argument that $m_t$ has to always be non-negative when $\mB$ holds. Finally, $\tau_{1-\epsilon/6}^+$ is used to analyze the stage before we achieve the initial strong correlation, we will show $m_t$ will stay above $1-\epsilon$ after $t > \tau_{1-\epsilon/6}^+$. I.e. the progress made for strong recovery is not eliminated by the noisy gradients.

\subsubsection{Analyzing the dynamics conditioning on $\mB$}
Now, notice that we can WLOG assume $\mathrm{sign}(h(0)) = 1$, since all the proofs will be symmetric as long as the event $\mA$ holds. Furthermore, let $r < \frac{1}{\sqrt{2}}$
\begin{lemma}[Characterizing dynamics before weak recovery]
\label{lemma:dynamics-before-weak-modular}
Conditioning on $\mA, \mB$, for $t\leq T \wedge \tau_r^+ \wedge \tau_0^-$, we have
\begin{equation*}
    m_{t} \geq \frac{\beta}{2\sqrt{d}} + \frac{t\eta S_k}{2}
\end{equation*}
Furthermore, we have $\tau_0 > T \wedge \tau_r^+$.
\end{lemma}
\begin{proof}
    Condition on $\mA, \mB$. Then, as explained before, WLOG assume $\mathrm{sign}(h(0)) = 1$. Then, for all $t\leq \tau_0^-$, we must have $h(m_t) \geq S_k$. Furthermore, for all $t\leq \tau_r^+$, we have $1-m_t^2 > \frac{1}{2}$. Then, rearranging using \Cref{claim:spherical-projection-apriori-bound} and applying the inequalities that hold with the event $\mB$, for $t\leq \tau_r^+ \wedge \tau_0^- \wedge T$, we have
    \begin{align*}
        m_{t} &\geq m_0 + \eta \sum_{j=0}^{t-1} h(m_j) (1-m_j^2) - \eta \sum_{j=0}^{t-1} \langle E_j, u\rangle - \eta^2 \sum_{j=0}^{t-1} \norm{L_j}^2 - \eta^3 \sum_{j=0}^{t-1} \norm{L_j}^2|\langle L_j,u\rangle|\\
        &\geq m_0 + \frac{\eta t S_k}{4} +\eta \sum_{j=0}^{t-1}\left(\frac{S_k}{4} - \eta \norm{L_j}\right) - \frac{\beta}{5\sqrt{d}}
    \end{align*}
    Now, using the uniform lower bound on the summation term and $m_0 \geq \frac{\beta}{\sqrt{d}}$, we have
    \begin{align*}
        m_t \geq \frac{\beta}{2\sqrt{d}} + \frac{\eta t S_k}{4}
    \end{align*}
    which concludes the first part. For the second part, suppose for $j \leq \tau_r^+ \wedge T$, we have $j \leq \tau_0^-$. Then, for all $l\in [0,1,\dots, j-1]$ we have $m_l \geq 0$, meaning $h(m_l) \geq S_k$. Hence, the above inequality holds for $j$, meaning $m_j > 0$. Hence, this implies $j < \tau_0^-$. Then, we conclude that it must be the case that $\tau_0^- > \tau_r^+ \wedge T$.
\end{proof}

\begin{lemma}[Dynamics after weak recovery is well approximated by drift term]
    \label{lemma:dynamics-after-weak-modular}
    Conditioning on $\mA, \mB, \tau_r^+$, the following holds: For $t\geq \tau_r^+$ with $t\leq T \wedge \tau_0^-$, we have
\begin{equation*}
    \left|m_{t} - m_{\tau_r}^+ - \eta \sum_{j=\tau_r^+}^{t-1} h(m_j) (1-m_j^2)\right| < \frac{\epsilon}{6}
\end{equation*}
    Furthermore, $\tau_0^- > T$.
\end{lemma}
\begin{proof}
    Notice that under the event $\mB$, due to non-negativity of each of the summands, we have the following upper bounds
    \begin{align*}
        \eta^3 \sum_{j=\tau_r^+}^{t-1} \norm{L_j}^2 |\langle L_j, u\rangle |&\leq \sup_{0\leq t\leq T} \eta^3 \sum_{j=0}^{t-1} |\langle L_j, u\rangle | < \frac{\epsilon}{18}\\
       \eta^2 \sum_{j=\tau_r^+}^{t-1} \norm{L_j}^2 &\leq \sup_{0\leq t\leq T}\eta^2 \sum_{j=0}^{t-1} \norm{L_j}^2  < \frac{\epsilon}{18}
    \end{align*}
    For the martingale term, since the terms are not necessarily non-negative we decompose it as
    \begin{align*}
        \left|\eta \sum_{j=\tau_r^+}^{t-1} \langle E_j, u\rangle \right|&= \left|\eta \sum_{j=0}^{t-1} \langle E_j, u\rangle-\eta \sum_{j=0}^{\tau_r^+ - 1} \langle E_j, u\rangle \right|\\
        &\leq \left|\eta \sum_{j=0}^{t-1} \langle E_j, u\rangle\right| + \left|\eta \sum_{j=0}^{\tau_r^+ - 1} \langle E_j, u\rangle \right|\\
        &\leq 2\sup_{0\leq t\leq T} \left|\eta \sum_{j=0}^{t-1} \langle E_j, u\rangle\right| < \frac{\epsilon}{18}
    \end{align*}
     Then, notice that the following holds exactly
     \begin{equation*}
         m_t = m_{\tau_r^+} + \eta \sum_{j=\tau_r^+}^{t-1} h(m_j) (1-m_j^2) + \eta \sum_{j=\tau_r^+}^{t-1} \langle E_t, u\rangle + \sum_{j=\tau_r^+}^{t-1} \left(1-\frac{1}{r_j}\right)(m_j - \eta \langle L_j, u\rangle)
     \end{equation*}
     which after rearranging, using $\left|1-\frac{1}{r_j}\right|\leq \eta^3 \norm{L_j}^2 |\langle L_j,u\rangle| + \eta^2 \norm{L_j}^2$ gives us
     \begin{align*}
         \left|m_t - m_{\tau_r^+}- \eta \sum_{j=\tau_r^+}^{t-1} h(m_j) (1-m_j^2)\right| &= \left|\eta \sum_{j=\tau_r^+}^{t-1} \langle E_t, u\rangle + \sum_{j=\tau_r^+}^{t-1} \left(1-\frac{1}{r_j}\right)(m_j - \eta \langle L_j, u\rangle)\right|\\
         &\leq \left|\eta \sum_{j=\tau_r^+}^{t-1} \langle E_t, u\rangle \right| + \eta^3\sum_{j=\tau_r^+}^{t-1}\norm{L_j}^2 |\langle L_j, u\rangle| + \eta^2 \sum_{j=\tau_r^+}^{t-1} \norm{L_j}^2
     \end{align*}
     using the $\epsilon/18$ bound for each of the terms, we get a total bound of $\epsilon/6$. Then, to get rid of the requirement $t\leq \tau_0^-$, notice that 
     \begin{align*}
         m_{t} - m_{\tau_r^+} \geq -\frac{\epsilon}{3} + \sum_{j=\tau_r^+}^{t-1} h(m_j)(1-m_j^2)
     \end{align*}
     Then, notice that if $t\leq \tau_0^-$, we have $m_j \geq 0$ for all $j\leq t-1$, so the sum is non-negative, which gives us $m_t \geq m_{\tau_r^+} -\frac{\epsilon}{3}\geq r-\frac{\epsilon}{3}$. However, notice that choosing $r=\frac{1}{2}$, we always have $\epsilon/3 < r$ so $m_t \geq 0$ as well. Hence, $\tau_0^- > t$, so we must have $\tau_0^- > T$.
\end{proof}

\noindent Now, we are in a position to prove \Cref{theorem:modular-convergence-dynamics}. 

\begin{proof}[Proof of \Cref{theorem:modular-convergence-dynamics}] First, due to the initalization requirement in the theorem, $\mA$ holds. Then, per \Cref{prop:dynamics-b-holds-whp}, $\mB$ holds with probability $1-o(1)$.  Then, conditioning in $\mB$, per \Cref{lemma:dynamics-before-weak-modular} and \Cref{lemma:dynamics-after-weak-modular}, we can drop the requirement that $t \leq \tau_0^-$. So, let $t\leq T \wedge \tau_r^+$. Conditioning on $\mB$, per \Cref{lemma:dynamics-before-weak-modular}, we have 
\begin{align*}
    m_t \geq \frac{\beta}{2\sqrt{d}} + \frac{t \eta S_k}{2}
\end{align*}
Then, notice that at time $T_{\mathrm{weak}}= \lceil\frac{2}{\eta S_k}\rceil$, the RHS is larger than $1$. Then, it must be the case that $\tau_r^+ \wedge T \leq T_{\mathrm{weak}}$. Then, it suffices to show $T_{\mathrm{weak}} \leq T$. Notice that $T_{weak} = \lceil \frac{2dV_k}{\delta S_k}\rceil$ and $T =\lceil \alpha dV_k\rceil = \lceil \frac{4(\log dV_k)}{\epsilon \delta S_k}\rceil> T_{\mathrm{weak}}$ when $\epsilon<1, V_k > 1$ and $d > 3$. Then, we conclude $\tau_r^+ \leq T_{\mathrm{weak}} \leq T$. 

Now, conditioning on $\tau_r^+$, for all $t\geq \tau_r^+$, with $t\leq T$ per \Cref{lemma:dynamics-after-weak-modular}, we have
\begin{align*}
    m_t \geq m_{\tau_r^+} + \sum_{j=\tau_r^+}^{t-1} h(m_j) (1-m_j^2) - \frac{\epsilon}{6}
\end{align*}
Now, consider $t \leq \tau_{1-\epsilon/6}^{+} \wedge T$, so that $h(m_j) (1-m_j^2) > S_k \frac{\epsilon}{6}$ for all $j \leq \tau_{1-\epsilon/6}^{+}$. Hence,
\begin{align*}
    m_t \geq r + \frac{\eta (t-\tau_r^+) S_k \epsilon}{6} - \frac{\epsilon}{6} > \frac{\eta (t-\tau_r^+) S_k \epsilon}{6} 
\end{align*}
Hence, notice that the RHS of the inequality is greater than $1$ at time $t = \tau_r^+ + \lceil\frac{6}{\eta S_k \epsilon}\rceil \leq T_{\mathrm{weak}} + \lceil\frac{6}{\eta S_k \epsilon}\rceil$. Hence, it must be the case that  $\tau_{1-\epsilon/6}^{+} \wedge T \leq T_{\mathrm{weak}}+ \lceil\frac{6}{\eta S_k \epsilon}\rceil$. However, notice that $T = \lceil\frac{dV_k (\log dV_k)}{\delta S_k \epsilon}\rceil$ which is larger than $T_{\mathrm{weak}} + \lceil\frac{6}{\eta S_k \epsilon}\rceil$ so it must be the case that $\tau_{1-\epsilon/6}^+ \leq T$. Finally, we need to show that $m_t$ stays above $1-\epsilon$ after it crosses $1-\epsilon/6$. However, notice that for $t' \geq t \geq \tau_r^+$, we have
\begin{align*}
    m_{t'}-m_{t}&\geq \left|m_t - m_{\tau_r^+}-\eta \sum_{j=0}^{t-1} h(m_j)(1-m_j^2)\right| + \left|m_{t'} - m_{\tau_r^+}-\eta \sum_{j=0}^{t'-1} h(m_j)(1-m_j^2)\right| + \sum_{j=t}^{t'-1} h(m_j)(1-m_j^2)\\
    &\geq -\frac{\epsilon}{3}
\end{align*}
so that $m_{t}\geq 1-\frac{\epsilon}{2}$ for $t\geq \tau_{1-\epsilon/6}^{+}$. Hence, we conclude that $m_{T} \geq 1-\frac{\epsilon}{2}$. Since this result holds for any $\tau_r^+$, we can conclude the proof.
\end{proof}

\section{Additional details}

In this section we provide details for two remaining points mentioned in the main text, namely the existence of multiple global optima when Assumption~\ref{assumption:orthogonal} does not hold, and the fact that one can learn a good approximation to the teacher model once $u$ is learned.

\subsection{Multiple global optima when Assumption~\ref{assumption:orthogonal} does not hold}
\label{app:global}

The following example shows that if the direction $u$ of the perturbation lies in the span of the base model weight vectors, then there can exist multiple global optima.

% just found this in mathematica..
\begin{example}
    Let $\lambda_1, \lambda = 1$, let $w_1 = (1,0)$, $w_2 = (0,1)$, and consider the activation $\sigma(z) = z^2$. If the base model $f: \R^2\to\R$ is given by $f(x) = \sum^2_{i=1} \lambda_i \sigma(\langle w_i, x\rangle)$, then observe that the following two rank-1 perturbations of equal scale are equal.

    First, take $u = (1/\sqrt{2}, 1/\sqrt{2})$ and $u' = (1/\sqrt{3},\sqrt{6}/3)$. Then define $c = (-(1+\sqrt{2})(2+\sqrt{3}), (1+\sqrt{2})(\sqrt{2}+\sqrt{3}))$ and $c' = -c$. Then one can verify that the teacher models $\sum^2_{i=1} \lambda_i \sigma(\langle w_i + c_i u, x\rangle)$ and $\sum^2_{i=1} \lambda_i \sigma(\langle w_i + c'_i u', x\rangle)$ are functionally equivalent, even though $\{w_1 + c_1 u, w_2 + c_2 u\} \neq \{w_1 + c'_1 u', w_2 + c'_2 u'\}$, regarded as unordered pairs of vectors in $\R^2$. Furthermore, $\norm{c} = \norm{c'}$.
\end{example}

\subsection{Learning the teacher model once $u$ is learned}
\label{app:learnc}

In this section, we show that learning $u$ is sufficient to learning the teacher model by adding additional features to the model and training the second layer.

\begin{definition}[Linear Model Family From Learned Features] Let $\hat u$ be given. Then, define the model family
\begin{equation}
    \mL_\lambda = \left\{\sum_{i=1}^k \lambda_{i, 1} \sigma \left(\left\langle\frac{w_i + \frac{\xi}{\sqrt{k}} \hat u}{\sqrt{1+\xi^2/k}}, x\right\rangle\right)+ \lambda_{i, 2}\sigma\left(\left\langle\frac{w_i - \frac{\xi}{\sqrt{k}} \hat u}{\sqrt{1+\xi^2/k}}, x\right\rangle\right): \lambda \in \R^{k} \times \R^{k}\right\} \label{eq:linear-model-family}
\end{equation}
\end{definition}
Then, we will show that once we learn $\hat u$ to a sufficient accuracy, there exist a choice of $\lambda$ that allows the linear model to closely approximate the teacher model.
\begin{theorem}[Learning $u$ is sufficient to learn $f^*$] Suppose $\hat u$ is such that $1-|\langle u,\hat u\rangle| \leq \epsilon\cdot \frac{k+\xi^2}{2C_{\sigma}\lambda_{\max}^2 \xi^2 k^2}$ which is $\Theta(\epsilon/k)$ for $\xi=\Theta(1)$ and $\Theta(\epsilon/k^2)$ for $\xi=\Theta(\sqrt{k})$  Then, there exists a model $h \in \mL_\lambda$ as defined in \Cref{eq:linear-model-family} such that $\E_x(f^*(x)-h(x))^2\leq \epsilon$. In particular, second layer training on the family of neural networks defined as $\mL_\lambda$, we 
\end{theorem}
\begin{proof}
WLOG suppose $\langle u, \hat u\rangle > 0$, otherwise we flip all the signs of the $c_i$ in the later part of the proof. Consider the candidate model $h\in \mL_\lambda$ (given in \cref{eq:linear-model-family}) given by
    \begin{equation*}
        h(x)=\sum_{i=1}^k \lambda_i \sigma \left(\left\langle \frac{w_i + \xi c_i \hat u}{\sqrt{1+\xi^2/k}},x\right\rangle\right)
    \end{equation*}
    We aim to show $\E_x(f^*(x) -\hat f(x))^2 \leq \epsilon$. Notice
    \begin{equation*}
        \E_x(f^*(x)-\hat f(x))^2 \leq k\sum_{i=1}^k \lambda_i^2 \E_x \left(\sigma(\langle v_i, x\rangle)-\sigma (\langle \tilde v_i, x\rangle)\right)^2
    \end{equation*}
    where $v_i$ is as before and $\tilde v_i =\frac{w_i + \xi c_i \hat u}{\sqrt{1+\xi^2/k}}$. Then, it suffices to show that the expectation is less than $\frac{\epsilon}{\lambda_{\max}^2 k^2}$. Note
    \begin{equation*}
        \E_x(\sigma(\langle v_i, x\rangle)-\sigma(\langle v_i, x\rangle))^2 \leq C_\sigma \norm{v_i-\hat v_i}^2
    \end{equation*}
    Furthermore, we have
    \begin{equation*}
        \norm{v_i -\hat v_i}=\frac{\xi/\sqrt{k} \norm{u - \hat u}}{\sqrt{1+\xi^2/k}}
    \end{equation*}
    So that 
    \begin{equation*}
        k\sum_{i=1}^k\lambda_i^2 \E_x(\sigma(\langle v_i, x\rangle)-\sigma(\langle v_i, x\rangle))^2\leq C_\sigma \lambda_{\max}^2 k \frac{2\xi^2 (1-\langle u, \hat u\rangle)}{1+\xi^2/k}
    \end{equation*}
    Then, it suffices to get $1-\langle u,\hat u\rangle \leq \epsilon\cdot \frac{k+\xi^2}{2C_{\sigma}\lambda_{\max}^2 \xi^2 k^2}$ as desired.
\end{proof}

\begin{remark}
    The above result can be extended to the case when the $c_i$ are not necessarily quantized, by quantizing the interval $[-1,1]$ into a sufficiently granular discrete set of elements. Then, the algorithm follows similarly by adding these features into the model and training the second layer (e.g. via linear regression or SGD). 
\end{remark}

\end{appendices}

\end{document}